\documentclass{article}

\usepackage[preprint]{neurips_2026}

\usepackage[utf8]{inputenc} 
\usepackage[T1]{fontenc}    
\usepackage{hyperref}       
\usepackage{url}            
\usepackage{booktabs}       
\usepackage{amsfonts}       
\usepackage{nicefrac}       
\usepackage{microtype}      
\usepackage{xcolor}         

\usepackage{amsmath, amssymb, amsthm, mathtools, mathrsfs, bbm, bm}

\usepackage{algorithm}
\usepackage{algorithmic}
\usepackage[OT1]{fontenc}

\usepackage{graphicx}
\usepackage{caption}
\usepackage{subcaption}
\usepackage{enumitem}

\usepackage[nameinlink]{cleveref}
\crefname{equation}{}{}
\crefname{enumi}{}{}
\crefname{assumption}{Assumption}{Assumptions}
\crefname{lemma}{Lemma}{Lemmas}
\crefname{theorem}{Theorem}{Theorems}
\crefname{corollary}{Corollary}{Corollaries}
\crefname{proposition}{Proposition}{Propositions}
\crefname{claim}{Claim}{Claims}
\crefname{procedure}{Procedure}{Procedures}
\crefname{algorithm}{Algorithm}{Algorithms}
\crefname{figure}{Figure}{Figures}
\crefname{remark}{Remark}{Remarks}
\crefname{section}{Section}{Sections}
\crefname{definition}{Definition}{Definitions}
\crefname{example}{Example}{Examples}
\crefname{table}{Table}{Tables}
\crefname{problem}{Problem}{Problems}
\crefname{appendix}{Appendix}{Appendices}

\theoremstyle{definition}
\newtheorem{theorem}{Theorem}
\newtheorem{lemma}{Lemma}
\newtheorem{corollary}{Corollary}
\newtheorem{proposition}{Proposition}
\newtheorem{definition}{Definition}
\newtheorem{remark}{Remark}
\newtheorem{assumption}{Assumption}

\newenvironment{repeatassumption}[1]{%
  \renewcommand{\theassumption}{\ref{#1}}%
  \addtocounter{assumption}{-1}%
  \begin{assumption}%
}{%
  \end{assumption}%
}

\newcommand{\bb}{\mathbb}

\newcommand{\R}{\bb R}
\newcommand{\E}{\bb E}

\newcommand{\tr}{\text{tr}}
\newcommand{\loc}{\mathrm{loc}}
\newcommand{\prior}{\text{prior}}

\newcommand{\KL}{\text{KL}}
\newcommand{\ub}{\text{ub}}
\newcommand{\lb}{\text{lb}}
\newcommand{\op}{\text{op}}
\newcommand{\Var}{\text{Var}}
\newcommand{\Cov}{\text{Cov}}
\newcommand{\sep}{\text{sep}}

\newcommand{\Hess}{\text{hess}}

\DeclareMathOperator\supp{supp}

\title{Which Pairs to Compare for LLM Post-Training?}

\author{%
\hspace*{-0.5cm}%
\begin{minipage}[t]{0.33\textwidth}
\centering\normalfont
\textbf{Jiangze Han}\\
CUHK Business School\\
\texttt{jiangzehan@cuhk.edu.hk}
\end{minipage}
\hfill
\begin{minipage}[t]{0.33\textwidth}
\centering\normalfont
\textbf{Vineet Goyal}\footnotemark[1]\\
Columbia University\\
\texttt{vgoyal@ieor.columbia.edu}
\end{minipage}
\hfill
\begin{minipage}[t]{0.33\textwidth}
\centering\normalfont
\textbf{Will Ma}\footnotemark[1]\\
Columbia University\\
\texttt{wm2428@gsb.columbia.edu}
\end{minipage}
}

\begin{document}

\maketitle

\begin{abstract}
Preference-based post-training has become a central paradigm for aligning language models. 
A common data-collection strategy is to generate a small set of completions for each prompt and label the resulting comparison pairs. 
However, human preference labels are often much more expensive than generating additional completions, suggesting a different use of the same labeling budget: generate a larger pool of completions, but label only the most informative comparison pairs. This paper studies which pairs should be compared in preference-based post-training. We formulate comparison curation as a sampling-design problem and evaluate designs by the quality of the final policy under the preference-based post-training objective. We instantiate this framework for Direct Preference Optimization (DPO), analyzing how the choice of labeled pairs propagates through DPO training to downstream policy performance. Our main results provide matching upper and lower bounds on the post-training optimality gap of the DPO-trained policy. 
The bounds show that comparison selection affects downstream performance through a single design-dependent information matrix, which links label allocation to parameter estimation error and policy suboptimality. 
This yields an explicit optimization criterion for budgeted comparison curation and motivates practical sampling designs for selecting informative pairs from large generated completion pools. 
Experiments on synthetic settings and language-model post-training benchmarks show that the proposed designs consistently improve sample efficiency over common comparison-selection heuristics.

\end{abstract}

\section{Introduction}\label{sec:intro}
Preference-based post-training has become a central paradigm for aligning large language models with human intent. Early Reinforcement Learning for Human Feedback (RLHF) pipelines follow a two-stage procedure: starting from a reference policy \(\pi_0\), typically a supervised fine-tuned model, they collect human preference labels over pairs of model-generated completions, fit a reward model from these comparisons, and then optimize the policy against this learned reward under a KL regularization constraint, often using policy-gradient methods such as PPO \citep{christiano2017deep,ouyang2022training}. This framework has been highly influential, but it also introduces a complex intermediate reward-modeling stage and a computationally expensive reinforcement-learning step. More recent methods, most notably Direct Preference Optimization (DPO), bypass explicit reward modeling and instead train the policy directly from pairwise preference data through a logistic objective, while retaining an implicit-reward interpretation under the same KL-regularized RLHF framework \citep{rafailov2023direct}. Subsequent variants of preference optimization further simplify or modify this post-training pipeline, but they continue to rely on pairwise preference comparisons as the fundamental supervision signal.

Across these post-training methods, the collection of human preference labels remains a key bottleneck. A common data-collection strategy is mechanical: for each prompt, generate a small number of completions from the reference policy and label either one pair or all within-prompt pairs. This strategy treats the generated comparison set as fixed, even though generation is relatively cheap while human labeling is expensive. In practice, preference datasets are also often assembled in batches before annotation, since coordinating human labeling is costly and logistically easier when comparisons are selected in advance. This motivates an offline curation question: rather than labeling all available comparisons or selecting them uniformly, can we generate a larger candidate pool first and then choose which pairs are most worth sending to annotators?

This paper studies this comparison-selection problem: which pairs should be compared for preference-based post-training? For each prompt, we first generate a larger set of completions from the initial reference policy, and then choose an informative portfolio of within-prompt comparison pairs, optimized jointly across all prompts and generated completions. The objective is to allocate the labeling budget to comparisons that are most useful for improving the final post-trained policy under the KL-regularized RLHF objective.

Formally, we model comparison curation as a sampling-design problem. After completions are generated for each prompt, a design specifies a sampling distribution over all within-prompt completion pairs. Given a labeling budget \(n\), we sample \(n\) pairs from this design and query their preference labels. In this paper, we instantiate the post-training step with DPO: the sampled labeled comparisons are used to train a DPO policy \(\hat\pi_n\). 
The design goal is to produce a policy with high value under the KL-regularized RLHF objective \(J\). Thus, the central question is not merely how to estimate preferences accurately in isolation, but how to allocate comparison labels so as to improve the final policy under the RLHF objective.

\textbf{Contributions.}
We make three main contributions.
First, we show that the effect of comparison curation on RLHF is captured by a single design-dependent information matrix, which measures how well the labeled comparisons identify the parameter directions that matter for downstream RLHF performance.
Second, we prove matching finite-sample upper bounds and information-theoretic lower bounds for the RLHF optimality gap, showing that a trace-form surrogate characterizes both the achievable performance of RLHF and the unavoidable statistical difficulty of the problem.
These bounds yield an explicit optimization criterion for selecting informative comparison pairs, along with a practical sampling policy.
Third, we validate the design criterion empirically on synthetic tabular and contextual experiments, as well as IMDb and Anthropic-HH DPO experiments, showing that the proposed designs consistently improve sample efficiency over other approaches in the literature. 

\subsection{Literature Review}\label{sec:lit}
To our knowledge, this is the first work to apply offline experimental design to comparison curation for RLHF post-training, with a criterion directly tied to downstream policy performance.

\textbf{RLHF} uses preference data to guide policy optimization, often by learning a reward model from pairwise comparisons and then optimizing a KL-regularized objective.
Early work by \citet{christiano2017deep} formalizes learning from human preferences using trajectory comparisons, while large-scale alignment pipelines such as InstructGPT \citep{ouyang2022training} popularized preference-based RL for language models.
DPO \citep{rafailov2023direct} provides a widely used alternative that directly fits a classification-style objective on pairwise preference data, avoiding explicit reward modeling and on-policy RL while retaining an implicit-reward interpretation.
Our work takes the KL-regularized objective as given and studies the orthogonal problem of curating comparison data under a fixed labeling budget. A standard statistical model for preference labels is the \textbf{Bradley--Terry model} \citep{bradley1952rank}, where the probability that one item is preferred to another is a logistic function of a latent score difference.
This model underlies much of ranking and preference learning, and it naturally appears in RLHF/DPO when pairwise labels are viewed as noisy comparisons induced by an explicit or implicit reward signal. 

A related literature studies \textbf{adaptive preference-query selection} to reduce RLHF/DPO data-collection costs.
For example, \citet{das2025active} propose active preference optimization for sample-efficient RLHF; \citet{ji2024reinforcement} develop query-efficient RLHF methods inspired by active learning; and \citet{muldrew2024active} study acquisition strategies for active preference learning in LLMs under DPO-style objectives. \citet{xie2025exploratory} propose an exploration-augmented online DPO method for sample-efficient preference-data collection in RLHF. \citet{lin2025activedpo} propose an online DPO method which selects preference queries using an estimated reward model induced by the current LLM. \citet{kveton2025active} study active learning for DPO and also consider an offline subset-selection setting, but their offline setting selects an informative subset from already labeled preference data for computational efficiency. Closest to our work, \citet{feng2025pilaf} propose PILAF, which changes the response-generation distribution in iterative and online RLHF/DPO so that the preference-loss gradient is aligned with the oracle RLHF objective gradient. 
In contrast, our focus is offline first-round comparison curation: when completions are generated from the reference/SFT model before any preference feedback is observed, we optimize which within-prompt pairs from the fixed generated pool should be labeled, with finite-sample guarantees for the final RLHF policy. This offline formulation is motivated by practical annotation workflows, where comparison datasets are often constructed in advance and then sent for batch human labeling.

Recent work studies how \textbf{preference-data properties} affect post-training behavior.
\citet{kim2025understanding} analyze how the data-generating distribution influences DPO optimization and gradient dynamics, including likelihood displacement and iterative DPO.
\citet{pan2025matters} emphasize the importance of the quality and coverage of preferred responses, while \citet{chowdhury2024provably} study DPO under random label flips and propose robust/debiased objectives with finite-sample guarantees.
Our work is complementary: rather than focusing on response quality, training dynamics, or label noise, we study pairwise comparison curation and show how budget allocation across candidate pairs controls downstream performance through an information-design criterion.

Methodologically, our approach builds on offline \textbf{optimal experimental design}, which studies how to allocate a limited measurement budget to optimize a statistical information criterion. In the paired-comparison setting, this amounts to deciding which pairs should be compared, and with what frequencies. Classical work has studied locally optimal designs for the Bradley--Terry model, showing that the optimal comparison structure can depend strongly on the unknown utility parameters \citep{grasshoff2008optimal}. Complementary minimax analyses relate the difficulty of estimating item utilities to the topology of the comparison graph through its Laplacian spectrum \citep{shah2016estimation}. Other lines of work develop adaptive or Bayesian designs for tournaments and paired comparisons \citep{glickman2005adaptive}, as well as fixed-budget allocation rules for expert comparisons in machine learning \citep{guo2018experimental}. More recently, optimal design ideas have been used for preference-data collection in learning systems, but the design objectives are primarily tied to learning preferences or reward models. For example, \citet{mukherjee2024optimal} use optimal design to learn item rankings, with an objective based on ranking accuracy. \citet{scheid2024optimal} study optimal design for learning a contextual linear reward model, with an objective formulated in terms of bandit regret. Our setting shares the view of pair selection as a budget-allocation problem, but differs in the downstream objective: we connect comparison allocation directly to the suboptimality of the final post-trained policy under a KL-regularized RLHF objective, rather than only to parameter-estimation accuracy or reward-model learning.

\section{Problem definition}\label{s:problem-def}
\paragraph{Reinforcement Learning for Human Feedback (RLHF)}
Let \(x\) denote a prompt, and let \(\mathcal A(x)\) denote the set of candidate completions available for prompt \(x\).
We write \(y\in\mathcal A(x)\) for a completion. A language model induces a conditional distribution over completions given each prompt. Following the RLHF literature, we refer to this conditional distribution as a policy and denote it by \(\pi(\cdot\mid x)\). Thus, \(\pi(y\mid x)\) is the probability that the model generates completion \(y\) given prompt \(x\). 

RLHF aims to align a reference policy \(\pi_0(\cdot\mid x)\), often obtained through supervised fine-tuning (SFT), with human preferences. The standard RLHF formulation assumes that each prompt--completion pair \((x,y)\) has an underlying human-preference reward \(r^\star(x,y)\), which is fixed but unknown to the learner, with larger values corresponding to completions that are more preferred by humans. The goal is then to learn a policy that assigns higher probability to high-reward completions while remaining close to the reference policy. This trade-off is captured by the KL-regularized RLHF objective
\begin{equation}\label{eq:KL-RLHF-obj}
J(\pi)
\doteq
\mathbb E_x\!\left[
\mathbb E_{y\sim \pi(\cdot\mid x)}[r^\star(x,y)]
-\beta\,\KL\!\big(\pi(\cdot\mid x)\,\|\,\pi_0(\cdot\mid x)\big)
\right].
\end{equation}
The first term measures the expected human-preference reward of the post-trained policy, while the KL term measures its deviation from the reference policy.
The parameter \(\beta>0\) controls the strength of this regularization.

Let \(\pi^\star\in\arg\max_{\pi}J(\pi)\) denote the RLHF-optimal policy. For each prompt, the maximizer balances reward improvement against deviation from the reference policy. This optimizer is jointly determined by the reference policy \(\pi_0\) and the reward function
\(r^\star\). In particular, it admits the closed form \citep{rafailov2023direct}
\begin{equation}\label{eq:gibbs-opt-form}
\pi^\star(y\mid x)
=
\frac{\pi_0(y\mid x)\exp(r^\star(x,y)/\beta)}
{Z(x)},
\qquad
Z(x)
=
\sum_{y'\in\mathcal A(x)}
\pi_0(y'\mid x)\exp(r^\star(x,y')/\beta).
\end{equation}
Thus, high-reward completions receive larger probability, but only relative to their probability under the reference policy.
The normalizing constant \(Z(x)\) ensures that \(\pi^\star(\cdot\mid x)\) is a valid distribution and depends only on the prompt.

In practice, \(r^\star\) is unknown, and information about \(r^\star\) is collected through pairwise preference labels.

\textbf{Preference data generation.} The standard RLHF data-collection pipeline proceeds as follows.
Given a prompt dataset \(\{x_i\}_{i=1}^m\), one first generates a finite candidate pool for each prompt:
\[
\mathcal Y_{x_i}=\{y_{i,1},\dots,y_{i,d}\},
\qquad
y_{i,k}\stackrel{\mathrm{i.i.d.}}{\sim}\pi_0(\cdot\mid x_i),
\quad k=1,\dots,d.
\]
Human experts then provide within-prompt pairwise preferences between two candidates,
yielding labeled comparisons of the form $(x_i,y_i^+,y_i^-,a_i)$, and $ a_i\in\{0,1\} $, where $a_i=1$ indicates that $y_i^+$ is preferred to $y_i^-$, and $a_i=0$ indicates otherwise.

In RLHF, the pairwise human preferences data generation is commonly modeled by a Bradley--Terry model: for a comparison pair $e=(x,y^+,y^-)$, the preference label satisfies
\[
a\mid e \sim \mathrm{Bernoulli}\!\left(\sigma( r^\star(x,y^+)-r^\star(x,y^-))\right),
\]
where $\sigma(u)=(1+e^{-u})^{-1}$ is the logistic function.

\paragraph{Direct Preference Optimization.} In the classical RLHF pipeline \citep{christiano2017deep,stiennon2020learning,ouyang2022training}, the labeled comparisons are first used to fit an explicit reward model, typically by maximizing the Bradley–Terry likelihood induced by the observed preferences.
The learned reward model is then used as a surrogate for the unknown human reward \(r^\star\) in the KL-regularized RLHF objective \cref{eq:KL-RLHF-obj}, which is optimized to obtain the final policy. While conceptually natural, this two-stage procedure separates reward estimation from policy optimization and requires an additional reinforcement-learning step, such as PPO, to obtain the final policy \citep{christiano2017deep,stiennon2020learning,ouyang2022training}.
DPO was introduced to simplify this pipeline by directly training the policy from pairwise preference data, using the closed-form structure \cref{eq:gibbs-opt-form} of the KL-regularized RLHF optimizer \citep{rafailov2023direct}.

The closed-form optimizer in \cref{eq:gibbs-opt-form} motivates the central idea of DPO: instead of explicitly learning a reward model and then optimizing a policy against it, one can express the reward differences $r^\star(x,y^+)-r^\star(x,y^-)$ needed for preference learning directly through the policy's log-ratio relative to the reference model (see \cref{eq:reward-diff}).
Indeed, if a policy \(\pi\) is optimal for the KL-regularized RLHF objective under some reward function \(r\), then rearranging \cref{eq:gibbs-opt-form} gives
\[
r(x,y)=
\beta\log\frac{\pi(y\mid x)}{\pi_0(y\mid x)}
+
\beta\log Z(x),
\]
where \(Z(x)\) depends only on the prompt. Therefore, although the absolute reward is identifiable only up to a prompt-dependent additive term, reward differences between completions for the same prompt are identifiable. Indeed, the prompt-dependent term cancels when we compare two completions \(y\) and \(y'\) under the same prompt \(x\), so these differences can be expressed through the corresponding policy ratios:
\begin{equation}\label{eq:reward-diff}
    r(x,y^+)-r(x,y^-)=
\beta\log
\frac{\pi(y^+\mid x)\pi_0(y^-\mid x)}
{\pi(y^-\mid x)\pi_0(y^+\mid x)}.
\end{equation}
DPO uses this observation to avoid fitting a separate reward model. For a parameterized policy $\pi_\theta$, DPO interprets its log-ratio against the reference policy as an implicit reward,
\[
r_\theta(x,y)
\doteq
\beta\log\frac{\pi_\theta(y\mid x)}{\pi_0(y\mid x)}.
\]

For a comparison pair \(e=(x,y^+,y^-)\), the corresponding implicit reward difference is 
\begin{equation}\label{eq:DPO-logit}
u_\theta(e)
\doteq
r_\theta(x,y^+)-r_\theta(x,y^-)=
\beta\log
\frac{\pi_\theta(y^+\mid x)\pi_0(y^-\mid x)}
{\pi_\theta(y^-\mid x)\pi_0(y^+\mid x)} .
\end{equation}
Under the Bradley–Terry model, a larger value of \(u_\theta(e)\) means that the policy-implied reward ranks \(y^+\) above \(y^-\) more strongly. Thus, DPO estimates \(\theta\) by maximum likelihood: it chooses \(\theta\) to maximize the likelihood of the observed preference labels under the Bradley--Terry logistic model parameterized by \(u_\theta(e)\).

Given a labeled comparison dataset
\(\mathcal D_n=\{(e_i,a_i)\}_{i=1}^n\), where \(a_i=1\) indicates that the first completion \(y^+_i\) in \(e_i=(x_i,y^+_i,y^-_i)\) is preferred and \(a_i=0\) otherwise, define
\[
\ell(a_i,u_\theta(e_i))\doteq
-a_i\log\sigma(u_\theta(e_i))
-(1-a_i)\log(1-\sigma(u_\theta(e_i))).
\]
The empirical DPO risk is
\begin{equation}\label{eq:dpo-LN}
L_n(\theta)
\doteq
\frac1n\sum_{i=1}^n \ell(a_i,u_\theta(e_i)).
\end{equation}
For a policy class \(\{\pi_\theta:\theta\in\Theta\}\), the DPO estimator $\hat{\pi}_n$ is
\[
\hat\theta_n\in\arg\min_{\theta\in\Theta}L_n(\theta),
\qquad
\hat\pi_n\doteq \pi_{\hat\theta_n}.
\]

\textbf{Experiment design problem.} We study comparison-data curation for preference-based post-training. Given a limited labeling budget, the learner must decide which comparison pairs \(e=(x,y^+,y^-)\) to label so as to most effectively improve the downstream RLHF performance of the final post-trained policy. We develop our theoretical analysis for the DPO policy.  

A common way to construct a human preference dataset is to generate a small number $d$ of completions for each prompt and then label either a single pair (e.g., \(d=2\)) or all \(\binom d2\) within-prompt pairs. However, in practice, generating additional completions is relatively cheap, whereas eliciting human preference labels is costly. This motivates us to consider an alternative pipeline: generate a larger candidate pool for each prompt, and then select a budgeted subset of informative comparison pairs to label.

Formally, for each prompt \(x\), the candidate completions in \(\mathcal Y_x\) induce a complete comparison graph: each completion \(y\) is a vertex, and each unordered edge $\{y^+,y^-\}$  represents a possible pairwise comparison.
After the candidate pools \(\{\mathcal Y_x\}\) are generated, define the admissible unordered edge set
\[
\mathcal E_x
\doteq
\big\{\{y^+,y^-\}\in\mathcal Y_x\times\mathcal Y_x:\ y^+\neq y^-\big\},
\qquad
\mathcal E
\doteq
\bigcup_x \{x\}\times \mathcal E_x.
\]
A sampling design is a distribution \(D\in\Delta(\mathcal E)\) over admissible within-prompt edges. For mathematical convenience, we analyze randomized designs: given a budget \(n\), the pairwise comparisons \(e_1,\dots,e_n\) are drawn i.i.d.\ from \(D\), and each selected edge is labeled according to the Bradley--Terry model above. This formulation casts comparison curation as an optimization problem over probability distributions, rather than as direct combinatorial subset selection. After optimizing the design distribution, we construct the DPO training dataset by sampling \(n\) comparison pairs from it.

For a given sampling design \(D\), let \(\hat\pi_n(D)\) denote the DPO policy obtained by first sampling
\(n\) comparison edges \(e_i\) from \(D\), querying their preference labels \(a_i\), and then minimizing the empirical DPO risk \cref{eq:dpo-LN} on the resulting labeled dataset $\mathcal D_n=\{(e_i,a_i)\}_{i=1}^n$.
The design problem is
\[
\max_{D\in\Delta(\mathcal{E})}
\ \E\!\left[J(\hat\pi_n(D))\right],
\]
where the expectation is over the sampled comparison edges and labels. The rest of the paper develops upper and lower bounds for this objective and derives computable surrogates for selecting informative comparison portfolios.

\textbf{Parametric policy class.} For the theoretical analysis, we work with a softmax policy class.
Each prompt--completion pair \((x,y)\) is represented by an embedding
\(\phi(x,y)\in\mathbb R^p\), and any policy $\pi_\theta$ with $\theta\in\Theta$ has the form
\begin{equation}\label{eq:policy-softmax}
\pi_\theta(y\mid x)
=
\frac{\exp(f_\theta(\phi(x,y)))}
{\sum_{y'\in\mathcal A(x)}\exp(f_\theta(\phi(x,y')))},
\qquad y\in\mathcal A(x),
\end{equation}
where \(f_\theta:\mathbb R^p\to\mathbb R\) is smooth in \(\theta\).
This formulation includes linear contextual policies and tabular softmax policies as special cases.
In the linear contextual case, \(f_\theta(\phi)=\langle\theta,\phi\rangle\), with \(\theta\in\mathbb R^p\).
In the tabular case, \(\phi(x,y)=\mathbbm{1}_{(x,y)}\in\mathbb R^p\) is a one-hot vector over all prompt--completion pairs, and
\(f_\theta(\phi(x,y))=\theta_{(x,y)}\), where \(\theta\in\mathbb R^p\). For detailed explanations, see \cref{rem:policy-special-cases}.

\textbf{Notation.} For a symmetric positive semidefinite matrix \(A\), \(A^\dagger\) denotes the Moore--Penrose pseudoinverse, obtained by inverting the positive eigenvalues of \(A\) and leaving the zero eigenvalues unchanged.
We write \(A^{\dagger/2}\) for the symmetric square root of \(A^\dagger\).
When \(A\) is positive definite on a subspace \(H\), \(A^\dagger\) acts as the usual inverse on \(H\).

\subsection{Regularity assumptions}
\label{subsec:assumptions}

We collect the main assumptions used in the analysis. Detailed interpretations, examples, and primitive sufficient conditions are presented in Appendix~\ref{app:assumptions}.

\textbf{Identifiability and realizability.} Recall that each comparison edge \(e=(x,y^+,y^-)\) enters the DPO loss through the scalar logit difference
\(u_\theta(e)\). Define the pairwise sensitivity vector
\[
g(e;\theta)\doteq \nabla_\theta u_\theta(e).
\]
Because softmax probabilities depend only on within-prompt logit differences, directions that shift all logits for a prompt by the same amount are unidentifiable. We restrict attention to the identifiable tangent space
\[
H \doteq
\operatorname{span}\bigl\{g(e;\theta^\star): e=(x,y^+,y^-)\in\mathcal E\bigr\}
\subseteq \mathbb R^p .
\]
\begin{assumption}[Identifiability and realizability]\label{ass:identifiable-space}
The parameter set \(\Theta\) is convex, compact, and satisfies
\[
\Theta\subseteq \theta^\star+H.
\]
\end{assumption}
This assumption generalizes the standard normalization used when estimating a Bradley--Terry model in a tabular setting. In the tabular case, the assumption amounts to requiring that, within each prompt, the rewards of all completions sum to zero; see \cref{rem:H-tabular}. This normalization is commonly used to make the tabular Bradley--Terry rewards uniquely identifiable \citep{shah2016estimation,chowdhury2024provably,feng2025pilaf}.

\noindent\textbf{Boundedness and smoothness.} We impose standard boundedness and smoothness conditions on the feature map and the score model. Similar regularity conditions are commonly used in the analysis of DPO; see, e.g., \citet{chowdhury2024provably}.
\vspace{2mm}
\begin{assumption}[Boundedness and smoothness]\label{ass:model-bounded}
There exist constants \(R_\phi,\alpha_0,\alpha_1,\alpha_2,\alpha_3<\infty\) such that for all
\(\theta\in\Theta\) and all admissible \((x,y)\),
\[
\|\phi(x,y)\|_2\le R_\phi,\qquad
|f_\theta(\phi(x,y))|\le \alpha_0,
\]
\[
\|\nabla_\theta f_\theta(\phi(x,y))\|_2\le \alpha_1,\qquad
\|\nabla_\theta^2 f_\theta(\phi(x,y))\|_{\op}\le \alpha_2,\qquad
\|\nabla_\theta^3 f_\theta(\phi(x,y))\|_{\op}\le \alpha_3.
\]
Moreover, for each fixed \((x,y)\), the map
\(\theta\mapsto f_\theta(\phi(x,y))\) is three times continuously differentiable on \(\Theta\).
\end{assumption}

\noindent\textbf{Feature separation.} To obtain meaningful learning guarantees from pairwise comparisons, namely that the DPO solution achieves a vanishing RLHF optimality gap as the comparison budget $n$ grows, we need the candidate completion set
$\mathcal Y_x$ to be sufficiently informative. Intuitively, if for a given prompt $x$ the candidate pool
contains only near-duplicate completions (or completions that are indistinguishable under the model features),
then comparing them provides little information about how the policy should change, and certain parameter
directions cannot be learned no matter how many comparisons we collect. To rule out such degenerate cases,
we impose a mild diversity condition on the candidate set: for each prompt, $\mathcal Y_x$ should contain
at least two completions that are sufficiently different in their model-induced features (in every identifiable
direction). This is formalized by the following feature-separation assumption.

\vspace{2mm}
\begin{assumption}[Feature separation on \(H\)]\label{ass:feature-sep}
There exists \(\Delta_g>0\) such that for every \(\theta\in\Theta\), every prompt \(x\), and every unit vector
\(v\in H\), there exist two candidates \(y_1,y_2\in\mathcal Y_x\) satisfying
\[
\Big|v^\top\big(\nabla_\theta f_\theta(\phi(x,y_1))
-\nabla_\theta f_\theta(\phi(x,y_2))\big)\Big|
\ge
\Delta_g.
\]
\end{assumption}

Consider the tabular setting with \(d\) items and reward vector \(r=(r_1,\ldots,r_d)\). The feature vectors are one-hot vectors, and the identifiable subspace is
\[
H=\left\{v\in\mathbb R^d:\sum_{i=1}^d v_i=0\right\}.
\]
In this case, \cref{ass:feature-sep} reduces to the requirement that every unit vector \(v\in H\) separates some pair of items: there exists \(\Delta_g>0\) such that for every $v\in H$
\[
\max_{i,j\in[d]} |v_i-v_j|\ge \Delta_g .
\]
This holds automatically, since a unit zero-mean vector cannot have all coordinates equal.

\noindent\textbf{Design coverage.} Even with a diverse candidate pool, a meaningful guarantee further requires that the sampling design $D$ does not
ignore informative pairs. We therefore impose the following coverage assumption, which is standard in the estimation of Bradley--Terry models; see, e.g., \citet{shah2016estimation,chowdhury2024provably}.

\begin{definition}[Design covariance matrix]\label{def:design-matrix}
    A sampling design \(D\) over comparison edges induces the design covariance matrix
\[
\Sigma_D(\theta)\doteq
\E_{e\sim D}\!\left[g(e;\theta)g(e;\theta)^\top\right].
\]
\end{definition}
Our guarantees require that the sampling design \(D\) places sufficient mass on informative comparisons.

\setcounter{assumption}{3}
\renewcommand{\theassumption}{\arabic{assumption}(a)}

\begin{assumption}[Design coverage at the truth]\label{ass:design-coverage-a}
There exists \(\mu_\star>0\) such that
\[
v^\top \Sigma_D(\theta^\star)v
\ge
\mu_\star\|v\|_2^2,
\qquad \forall v\in H.
\]
\end{assumption}

In some arguments, we require the above assumption to hold uniformly over a neighborhood of $\theta^\star$. Accordingly, we introduce the following localized version. Whenever invoked, the region $\mathcal R\subseteq\Theta$ will be specified, typically chosen as a neighborhood of $\theta^\star$.

\addtocounter{assumption}{-1}
\renewcommand{\theassumption}{\arabic{assumption}(b)}

\begin{assumption}[Uniform design coverage]\label{ass:design-coverage-b}
For a specified region \(\mathcal R\subseteq\Theta\), there exists \(\mu_{\mathcal R}>0\) such that
\[
v^\top \Sigma_D(\theta)v
\ge
\mu_{\mathcal R}\|v\|_2^2,
\qquad
\forall v\in H,\ \forall \theta\in\mathcal R.
\]
\end{assumption}

\renewcommand{\theassumption}{\arabic{assumption}}

\noindent\textbf{Optimization landscape.}
Let $L(\theta)\doteq\E_{(e,a)}[\ell(a,u_\theta(e))]$ denote the population DPO risk under the sampling design $D$ and label (BT) model.
For general nonlinear score models, \(L(\theta)\) may be nonconvex and may have multiple minimizers, so we assume the target parameter is unambiguous (\cref{rem:unique-minimizer}).

\begin{assumption}[Unique population minimizer]\label{ass:unique-minimizer-pop-risk}
The population DPO risk \(L(\theta)\) has a unique global minimizer over \(\Theta\).
\end{assumption}
Such uniqueness assumption is standard in the analysis of extremum estimators, where consistency typically requires the limiting population risk to have a unique optimizer \citep[Section~2.1]{newey1994large}\citep[Theorem~5.7]{van2000asymptotic}.

\noindent\textbf{Prior regularity for the lower bound.} Our final assumption introduces a Bayesian formulation for the optimal parameter $\theta^\star$.
When prior information about $\theta^\star$ is available, our analysis yields an information-theoretic
(Bayesian) lower bound on the RLHF optimality gap under an arbitrary sampling design $D$. Detailed explanation is provided in \cref{rem:van-tree}.
\begin{assumption}[Prior regularity]\label{ass:prior-vantrees}
Suppose \(\theta^\star\sim\rho\), where \(\rho\) is supported on \(\Theta\). The density \(\rho\) satisfies: (i) \(\rho\in C^1(\Theta)\) and \(\rho(\theta)>0\) for all \(\theta\in\mathrm{int}(\Theta)\); (ii) \(\rho(\theta)=0\) for all \(\theta\in\partial\Theta\); (iii) the prior Fisher information is finite: $ \int_\Theta \|\nabla\log\rho(\theta)\|_2^2\,\rho(\theta)\,d\theta<\infty.$
\end{assumption}
\cref{ass:prior-vantrees} is a standard regularity condition on the prior density. It ensures that the
quantity $\int_\Theta \|\nabla\log\rho(\theta)\|_2^2\,\rho(\theta)\,d\theta$ is finite and that the
integration-by-parts steps required by the Bayesian information inequality (Van Trees) in our lower-bound
argument are valid.

\section{Main results}\label{s:main-results}
Following the literature, we refer to the gradient of the log-likelihood as the \emph{score function}. The second moment of the score function is the \emph{Fisher information matrix}. In our contextual policy model,
the Fisher information are defined as follows. Detailed interpretations are presented in Appendix~\ref{app:info-theory}.

\begin{definition}[Fisher information matrix for the policy family]\label{def:Fisher}
For each parameter $\theta\in\Theta$, define the score vector
\[
s_\theta(x,y)\doteq \nabla_\theta \log \pi_\theta(y\mid x).
\]
The Fisher information matrix is
\begin{equation}\label{eq:def-I-theta}
I(\theta) \;  \doteq
\mathbb E_{x}\,\mathbb E_{y\sim \pi_\theta(\cdot\mid x)}
\!\left[
s_\theta(x,y)s_\theta(x,y)^\top
\right].
\end{equation}
\end{definition}

\vspace{2mm}
\begin{theorem}[Informal] 
\label{thm:informal}
Consider a sampling design \(D\) over within-prompt pairs, and \(n\) i.i.d.\ comparisons drawn from \(D\). Let $I(\theta)$ and $\Sigma_D(\theta)$ be as defined in \cref{def:Fisher,def:design-matrix}. Under  \cref{ass:identifiable-space,ass:model-bounded,ass:feature-sep}, we have the following.

\begin{enumerate}
\item[\textbf{(i)}] \textbf{Upper bound.} Let
\(\hat\theta_n\in\arg\min_{\theta\in\Theta}L_n(\theta)\) with \(\hat\pi_n=\pi_{\hat\theta_n}\).
If \cref{ass:design-coverage-a,ass:unique-minimizer-pop-risk} hold in addition, then for any \(\delta\in(0,1)\), there exists \(n_0(\delta)\) such that for all
\(n\ge n_0(\delta)\), with probability at least \(1-\delta\),
\[
J(\pi^\star)-J(\hat\pi_n)
\le
\frac{C_{\mathrm{ub}}(\delta)}{n}\,
\tr\!\Big(I(\theta^\star)\Sigma_D^\dagger(\theta^\star)\Big),
\]
where \(C_{\mathrm{ub}}(\delta)\) depends only on the fixed constants in the assumptions and \(\delta\).
Moreover, $ n_0(\delta)
=
\widetilde{\mathcal O}\!\bigl(
C_{\mathrm{loc}}p
+
C_{\mathrm{Hess}}d_H
+
(C_{\mathrm{loc}}+C_{\mathrm{Hess}})\log(1/\delta)
\bigr),$ where \(p\) is the ambient dimension, \(d_H=\dim(H)\), \(C_{\mathrm{loc}}\) is a localization complexity parameter, and \(C_{\mathrm{Hess}}\) is a local curvature/concentration parameter.

\item[\textbf{(ii)}] \textbf{Lower bound.} Let $ n_{\mathrm{lb}}
:=
\left\lceil \frac{c_\text{reg}}{\mu_{\mathcal{R}}}
\Big(1+\int_\Theta \|\nabla \log \rho(\theta)\|_2^2\,\rho(\theta)\,d\theta\Big)
\right\rceil ,$ where $c_\text{reg}$ is a constant depending only on the fixed constants in the assumptions. If \cref{ass:design-coverage-b,ass:prior-vantrees}
hold in addition with $\mathcal R=\supp(\rho)$,
then, for any induced policy estimator
$\tilde\pi_n=\pi_{\tilde\theta_n}$ and all $n\ge n_{\mathrm{lb}}$,
\[
\E_{\theta^\star\sim\rho}
\E_{\mathcal D_n\mid\theta^\star,D}
\left[
J(\pi_{\theta^\star})-J(\tilde\pi_n)
\right]
\ge
\frac{\beta}{2n}
\E_{\theta^\star\sim\rho}
\left[
\tr\!\left(
I(\theta^\star)\Sigma_D^\dagger(\theta^\star)
\right)
\right].
\]
Here $n_{\mathrm{lb}}$ depends only on the assumption
constants and the prior $\rho$.
\end{enumerate}
\end{theorem}
We present the proof in Appendix~\ref{app:proof-informal}. Here we present a high level sketch for the proof of \cref{thm:informal}. The main idea is to bound the downstream policy suboptimality to a weighted parameter-estimation and policy estimation error, and then control these errors from above for DPO and from below for any estimator.

\textbf{Step 1: Sandwich bound for RLHF gap.} We first show that the RLHF optimality gap can be (i) upper bounded by the parameter estimation error and (ii) lower bounded by the policy estimation error. Specifically, the parameter estimation error is measured by a quadratic form $(\theta-\theta^\star)^\top I(\theta^\star)(\theta-\theta^\star)$; and the policy estimation error is measured by the squared Hellinger distance $H(\pi_{\tilde\theta_n}, \pi^\star)^2$ between estimated policy $\pi_{\tilde\theta_n}$ and the optimal policy $\pi^\star$. That is, under the regularity assumptions, for all \(\theta\in\Theta\),
\[
\beta\E_x\bigg[\|\sqrt{\pi_\theta(\cdot|x)}-\sqrt{\pi^\star(\cdot|x)}\|^2_2\bigg]
\le
J(\pi^\star)-J(\pi_\theta)
\le
c_+\,(\theta-\theta^\star)^\top I(\theta^\star)(\theta-\theta^\star),
\]
where \(I(\theta^\star)\) is the curvature/Fisher matrix of the policy around the RLHF optimum. The sandwich bound has two complementary implications: through the upper bound, a small parameter-estimation error guarantees a small optimality gap; through the lower bound, a large policy-estimation error necessarily implies a large optimality gap. We present a detailed proof in Appendix~\ref{ss:step1}.

\textbf{Step 2: Upper bound for DPO.} Let \(\hat\theta_n\) be the empirical DPO minimizer and \(\Delta=\hat\theta_n-\theta^\star\). We show that 
\[
\mathbb E\!\left[\Delta^\top I(\theta^\star)\Delta\right]
\lesssim
\frac{1}{n}\,
\tr\!\big(I(\theta^\star)\Sigma^\dagger_D(\theta^\star)\big).
\]
Step~1 then gives the stated RLHF-gap upper bound. See Appendix~\ref{s:upper-bound} for the full proof.

\textbf{Step 3: Lower bound for any estimator.} For the converse, we apply the functional Van Trees inequality to show that any estimator \(\pi_{\tilde\theta_n}\) must incur Bayes risk at least of order $\frac{1}{n}\,
\mathbb E_{\theta^\star\sim\rho}
\!\left[
\tr\!\big(I(\theta^\star)\Sigma^\dagger_D(\theta^\star)\big)
\right]. $ Combining this with the Squared Hellinger lower bound from Step~1 yields
\[
\inf_{\tilde\pi_n}
\mathbb E_{\theta^\star\sim\rho}\mathbb E_{\mathcal D_n\mid\theta^\star}
\!\left[
J(\pi_{\theta^\star})-J(\tilde\pi_n)
\right]
\ge
\frac{C_{\lb}}{n}\,
\mathbb E_{\theta^\star\sim\rho}
\!\left[
\tr\!\big(I(\theta^\star)\Sigma^\dagger_D(\theta^\star)\big)
\right].
\]
Thus the same trace functional controls both the achievable DPO guarantee and the information-theoretic limit, identifying it as the natural design criterion for comparison curation. We present a detailed proof in Appendix~\ref{s:lower-bound}.

\subsection{Sampling policy design}

From \cref{thm:informal}, both the upper and lower bounds on the RLHF optimality gap are governed (up to constants) by the same trace criterion: $\tr\!\Big(I(\theta^\star)\Sigma^\dagger_D(\theta^\star)\Big).$ Motivated by this characterization, we define the oracle sampling design as the solution to the following trace minimization problem: 
\begin{equation}\label{eq:oracle-design}
D^\star(\theta^\star) \in \arg\min_{D\in\Delta(\mathcal E)}
\ \tr\!\big(I(\theta^\star)\,\Sigma^\dagger_D(\theta^\star)\big),
\end{equation}
where \(\mathcal E\) is the admissible within-prompt edge set. We next show that under certain conditions, we can replace the dependence on $\theta^*$ by $\theta_0$. Therefore, we can have an implementable policy for such settings.

\begin{theorem}\label{thm:theta0}
Suppose
\cref{ass:identifiable-space,ass:model-bounded,ass:feature-sep}
hold.  Let $\theta_0\in\Theta$ be a fixed reference parameter, and $r_0 \doteq \|\theta^\star-\theta_0\|_2$. Consider the design
\[D_{\theta_0}\in\arg\min_{D\in\Delta(\mathcal E)}
\ \tr\!\Big(I(\theta_0)\Sigma^\dagger_D(\theta_0)\Big).\]
There exists a constant $r_{\rm ref}>0$, depending only
on the model primitive constants, such that whenever
$r_0<r_{\rm ref}$, 
\[
\tr\!\left(I(\theta^\star)\Sigma_{D_{\theta_0}}^\dagger(\theta^\star)\right)
\le
C_{\rm ref}(r_0)^2\,\inf_{D\in\Delta(\mathcal{E})}
\tr\!\left(I(\theta^\star)\Sigma_D^\dagger(\theta^\star)\right),
\]
where $C_{\rm ref}(r_0)\ge1$ depends only on the model
primitive constants and $r_0$, and
$C_{\rm ref}(r_0)\to1$ as $r_0\to0$.
\end{theorem}

We present a detailed proof in Appendix~\ref{app:proof-theta0}.


\section{Numerical experiments}

Our experiments proceed from controlled model-based settings to realistic language-model post-training benchmarks.
The synthetic experiments isolate the theoretical mechanism in realizable tabular and contextual models.
The IMDb experiment introduces real LLM fine-tuning while retaining a scalar proxy reward, allowing us to evaluate the reward--KL trade-off under model mismatch.
The Anthropic-HH experiment further evaluates preference quality without relying on an explicit reward function, using GPT-4.1 as an automatic judge.

\textbf{Synthetic setting.} The synthetic experiments test the method in settings that exactly follow the theoretical model.
Here the ground-truth reward is known to the experimenter, allowing us to directly evaluate whether the proposed design learns the unknown reward and the corresponding RLHF-optimal policy more efficiently. In the tabular setting, each item corresponds to a candidate completion, so a policy is simply a probability distribution over the \(d\) items.
We generate a non-uniform reference policy \(\pi_0\in\Delta^d\), mimicking the highly non-uniform output distribution of a pre-trained or supervised-fine-tuned language model.
The true reward vector \(r^\star\in\mathbb R^d\) is generated independently with small variance, creating a low-signal regime in which item rewards are relatively close. We compare four pairwise-comparison designs under a fixed sample budget \(n\): the oracle design \(D^\star(\theta^\star)\), the plug-in design \(D^\star(\theta_0)\), uniform sampling over all item pairs, and a heuristic that samples pairs by drawing two items without replacement according to \(\pi_0\).
For each design, we collect \(n\) noisy Bradley--Terry comparisons between item pairs and solve the empirical DPO objective in the tabular policy class.
This produces an estimated tabular policy \(\hat\pi_n\in\Delta^d\), which we evaluate by the RLHF optimality gap \(J(\pi^\star)-J(\hat\pi_n)\). The results in \cref{fig:tabular} show that \(D^\star(\theta^\star)\) and \(D^\star(\theta_0)\) achieve nearly identical and consistently small gaps, substantially outperforming both baselines. Uniform sampling improves as the budget grows but remains less sample-efficient, while the \(\pi_0\)-based heuristic performs poorly and shows little improvement.
These results suggest that, when the reference policy is highly non-uniform and reward differences are weak, concentrating comparisons on high-\(\pi_0\) items is ineffective; instead, information-guided comparison design yields much better downstream policy quality. 

Next, we study offline comparison curation in a \textbf{linear contextual setting}.
Each prompt \(x\in\mathbb R^p\) has a finite candidate set \(\mathcal A(x)=\{y_1,\ldots,y_d\}\), and each candidate has an action embedding \(a\in\mathbb R^p\).
We use the softmax policy class
\[
\pi_\theta(y\mid x)\propto \exp\!\big(\theta^\top\phi(x,y)\big),
\qquad
\phi(x,y)=[x;\,a;\,x\odot a]\in\mathbb R^{3p},
\]
where \(\odot\) denotes elementwise multiplication.
We generate a concentrated reference policy \(\pi_0=\pi_{\theta_0}\) using a large \(\|\theta_0\|\), and define a latent linear reward
\(r^\star(x,y)=\eta^\top\phi(x,y)+C(x)\), where the prompt-only term \(C(x)\) cancels in within-prompt comparisons.
Preference labels are generated from the Bradley--Terry model induced by \(\theta^\star\). For each labeling budget \(n\), we sample \(n\) comparison pairs using one of four rules:
the oracle design \(D^\star(\theta^\star)\), the plug-in design \(D^\star(\theta_0)\), uniform sampling over candidate pairs, and a \(\pi_0\)-weighted heuristic that samples candidates according to the reference policy.
We fit the DPO estimator and evaluate the held-out RLHF optimality gap \(J(\pi^\star)-J(\hat\pi_n)\), reporting mean and variability over Monte Carlo repetitions. \cref{fig:linear_contextual} shows that the plug-in design \(D^\star(\theta_0)\) closely tracks the oracle design and achieves small RLHF gaps across budgets.
Uniform sampling is less sample-efficient at small budgets, while the \(\pi_0\)-weighted heuristic performs poorly and has larger variance.
This indicates that high-probability candidates under the reference policy are not necessarily the most informative comparisons; explicitly optimizing the comparison design better targets the directions that matter for downstream RLHF performance. We further compare our design against Dope \citep{mukherjee2024optimal} and one-batch ActiveDPO \citep{lin2025activedpo} in the contextual-linear setting. We set both the contextual and action dimensions to 4, resulting in 12-dimensional features, and generate 16 candidate completions for each of 200 training prompts. Our design achieves the smallest mean RLHF optimality gap at every tested budget, reducing the gap by 4--44\% relative to Dope and by 58--61\% relative to one-batch ActiveDPO. Error bars report standard errors across five independently generated problem instances, with 20 Monte Carlo repetitions.

\begin{figure}[t]
    \centering
    \begin{subfigure}[t]{0.33\linewidth}
        \centering
        \includegraphics[width=\linewidth]{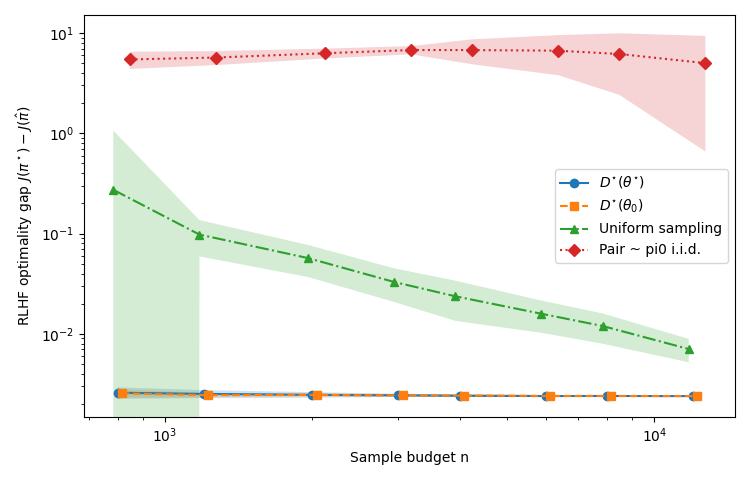}
        \caption{Tabular setting}
        \label{fig:tabular}
    \end{subfigure}\hfill
    \begin{subfigure}[t]{0.33\linewidth}
        \centering
        \includegraphics[width=\linewidth]{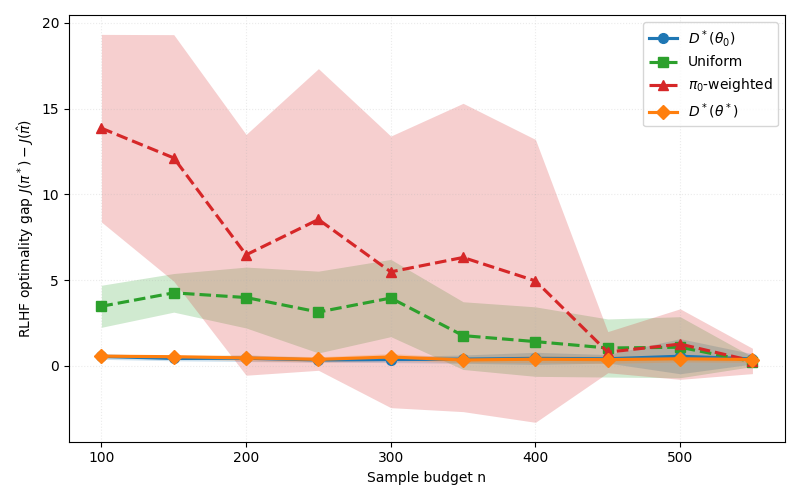}
        \caption{Linear contextual setting}
        \label{fig:linear_contextual}
    \end{subfigure}\hfill
    \begin{subfigure}[t]{0.33\linewidth}
        \centering
        \includegraphics[width=\linewidth]{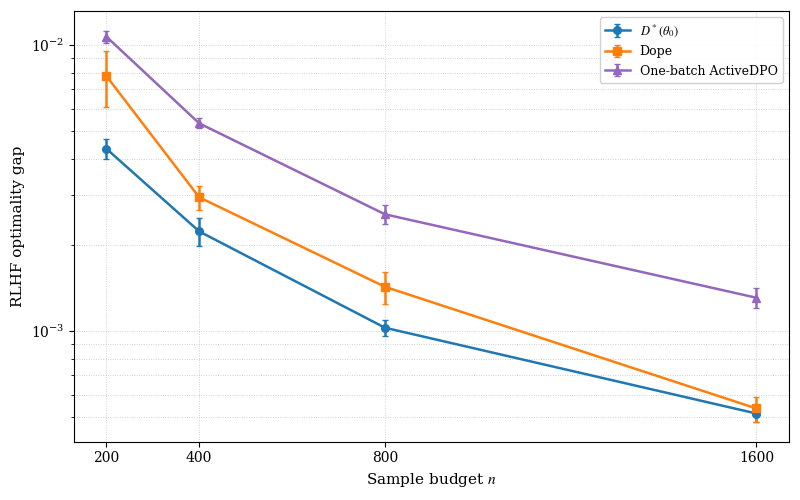}
        \caption{Dope and ActiveDPO}
        \label{fig:dope-activedpo}
    \end{subfigure}
    \caption{Synthetic experiments}
    \label{fig:synthetic}
\end{figure}

\textbf{IMDb experiment.} The IMDb experiment provides a more realistic language-model post-training testbed while retaining a well-defined proxy reward given by a sentiment classifier. This setting allows us to examine whether the design remains effective when the theoretical model is only an approximation of the actual LLM training process, and whether it improves the reward--KL trade-off relative to baseline comparison-selection rules. In IMDb experiment, we follow the DPO pipeline in \citet{rafailov2023direct}. We fine-tune GPT-2-large on the IMDb training split using SFT, and then use the resulting SFT model as the reference policy for DPO. We compare preference datasets constructed by our \(D^\ast\)-based method with benchmark selection rules in two curation tasks. These two tasks capture two natural decisions in preference-data collection: which prompts should be annotated, and, given a prompt, which candidate responses should be compared. In the prompt selection task, we generate two responses for each prompt in a candidate pool of \(1{,}000\) prompts and select \(n=175\) prompt-level comparisons for DPO training. The small budget \(n=175\) places the experiment in a low-annotation regime, where the value of selecting informative comparisons is most pronounced. The benchmark selects the first \(175\) comparisons, whereas \(D^\ast\) computes design weights over all candidate comparisons and samples without replacement. In the response selection task, we consider a candidate pool of $175$ prompts. We generate \(d=8\) candidate responses for each prompt and select one within-prompt response pair. The benchmark compares two arbitrary responses, whereas \(D^\ast\) samples a pair according to the normalized within-prompt design weights. We evaluate the trained policies using the reward--KL frontier as in \citet{rafailov2023direct}. Each point corresponds to the final checkpoint of a DPO run under a specific value of \(\beta\), and error bars report Monte Carlo standard errors. As shown in \Cref{fig:prompt-selection-frontier,fig:response-selection-frontier}, \(D^\ast\)-curated data consistently improves the reward--KL tradeoff relative to the benchmark in both curation tasks.

\begin{figure}[t]
    \centering

    \begin{subfigure}[t]{0.48\linewidth}
        \centering
        \includegraphics[
            height=2.8cm,
            keepaspectratio
        ]{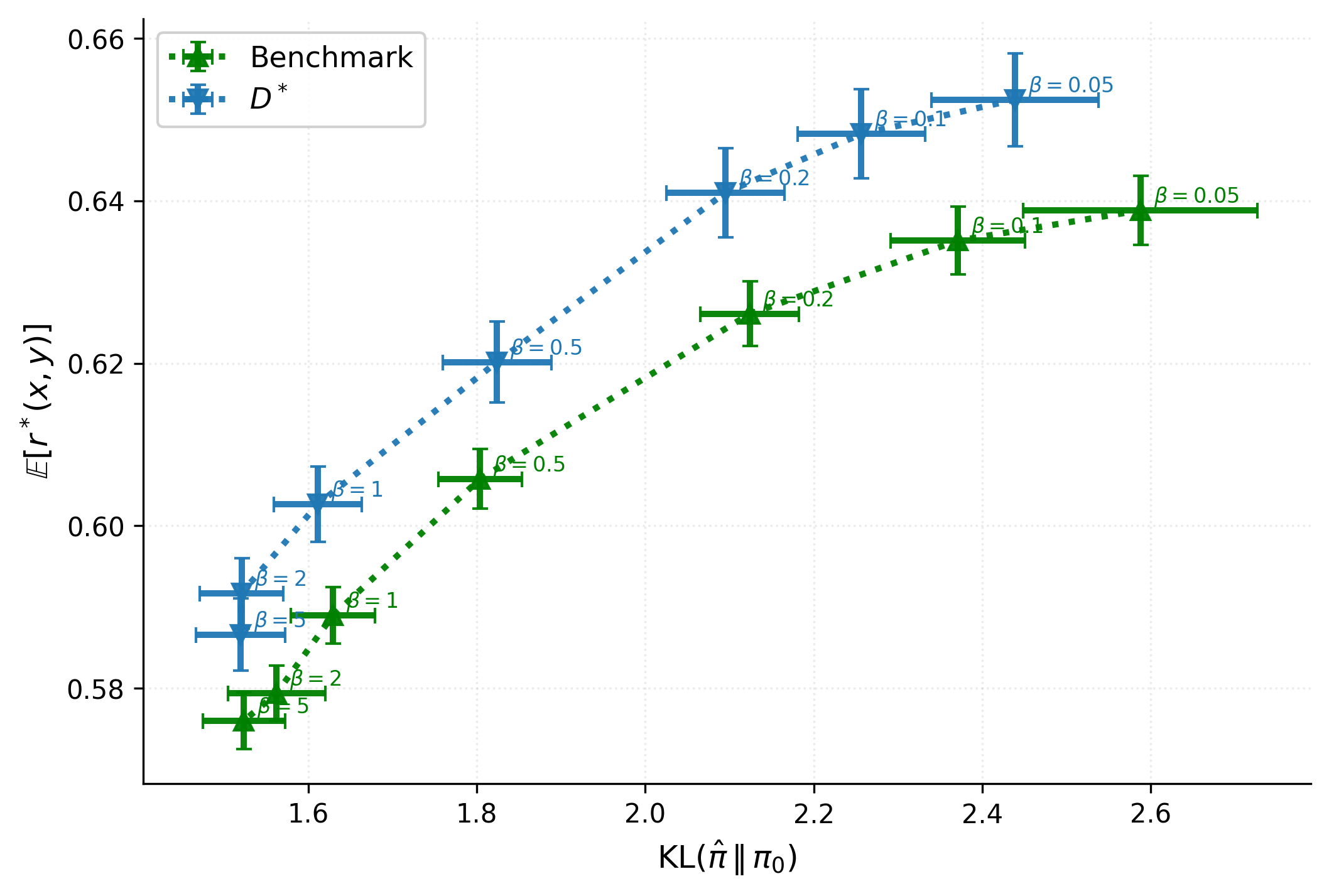}
        \caption{Response selection task}
        \label{fig:response-selection-frontier}
    \end{subfigure}
    \hfill
    \begin{subfigure}[t]{0.48\linewidth}
        \centering
        \includegraphics[
            width=0.92\linewidth,
            height=2.8cm,
            keepaspectratio
        ]{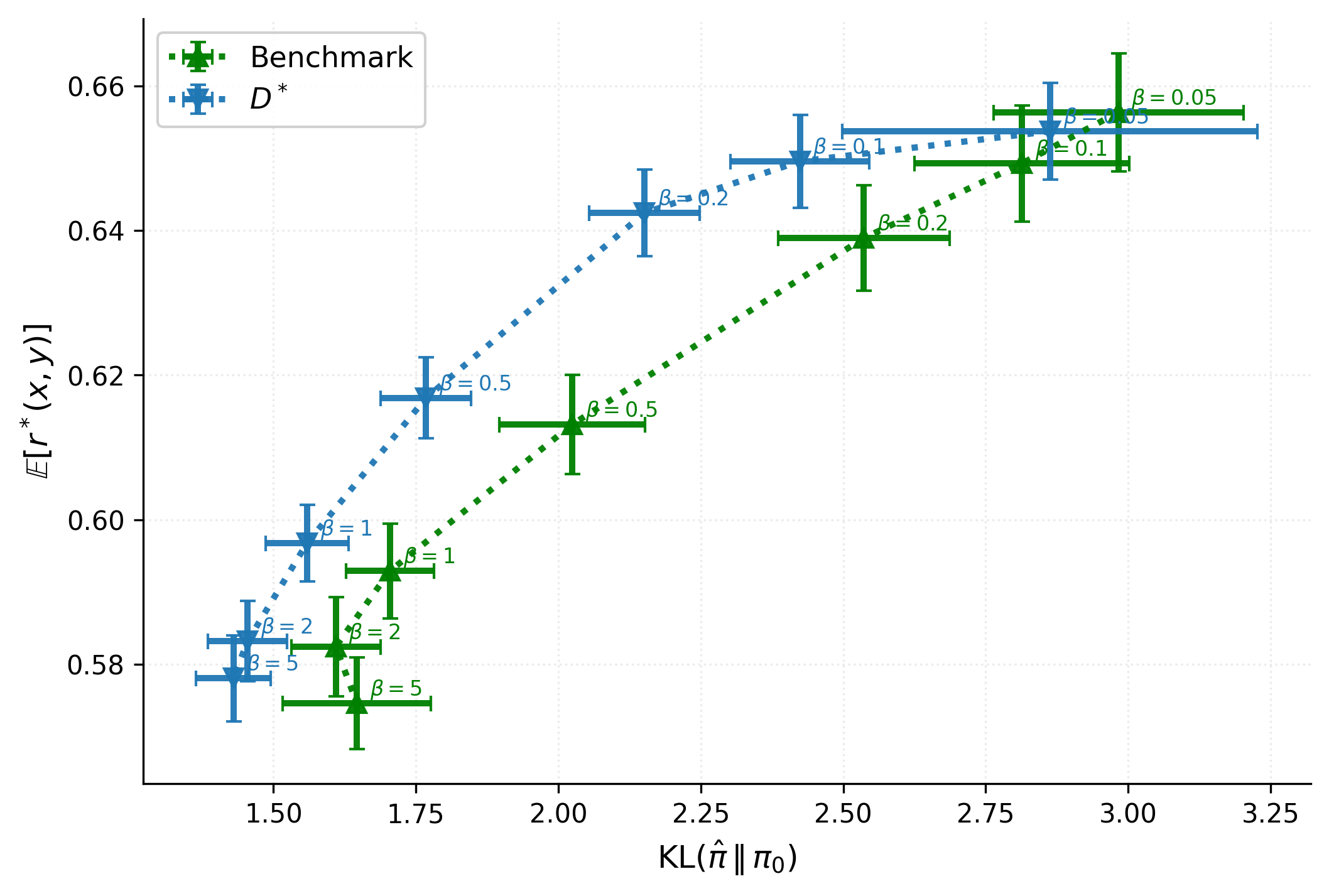}
        \caption{Prompt selection task}
        \label{fig:prompt-selection-frontier}
    \end{subfigure}

    \caption{IMDb experiments with GPT2-Large}
    \label{fig:imdb}
\end{figure}

\textbf{Anthropic HH experiment.} The Anthropic-HH experiment further removes the availability of a clear scalar reward function.
Instead of evaluating against an explicit reward model, we use GPT-4.1 as an automatic judge to compare model outputs and report pairwise win rates.
This setting tests whether, under the same labeling budget, the proposed comparison design leads to responses that are more often preferred in a direct preference evaluation. In this Anthropic HH experiment, we follow the DPO pipeline of \citet{rafailov2023direct}. We use the default Anthropic-HH train/test splits and first train a Pythia-2.8B SFT model on the chosen responses in the training split. Starting from this SFT model, we train DPO models using preference datasets constructed either by our \(D^\ast\)-based method or by a benchmark rule. The candidate pool consists of \(160{,}800\) preference pairs from the HH training split. For each budget \(n\), the benchmark uses the \(n\) arbitrary preference pairs, while our method samples \(n\) pairs without replacement from the optimized \(D^\ast\) design distribution. Details on feature construction, the trace-design optimization, and hyperparameters are provided in Appendix~\ref{app:hh-implementation}. We evaluate the trained policies on prompts from the Anthropic-HH test split. Following the evaluation setup of \citet{rafailov2023direct}, we generate responses at sampling temperatures \(0.25\), \(0.7\), and \(1.0\), and use GPT-4.1 to compare each generated response against the corresponding HH chosen response. \Cref{fig:hh} reports the win rate for two representative budgets, \(n=80{,}400\) and \(n=96{,}480\), corresponding to \(50\%\) and \(60\%\) of the candidate pool. Across these budgets and all sampling temperatures, the \(D^\ast\)-curated datasets outperform the benchmark. Similar improvements are observed across the other budgets we tested.

    \begin{figure}[t]
    \centering
    \begin{subfigure}[t]{0.48\linewidth}
        \centering
        \includegraphics[height=2.8cm,keepaspectratio]{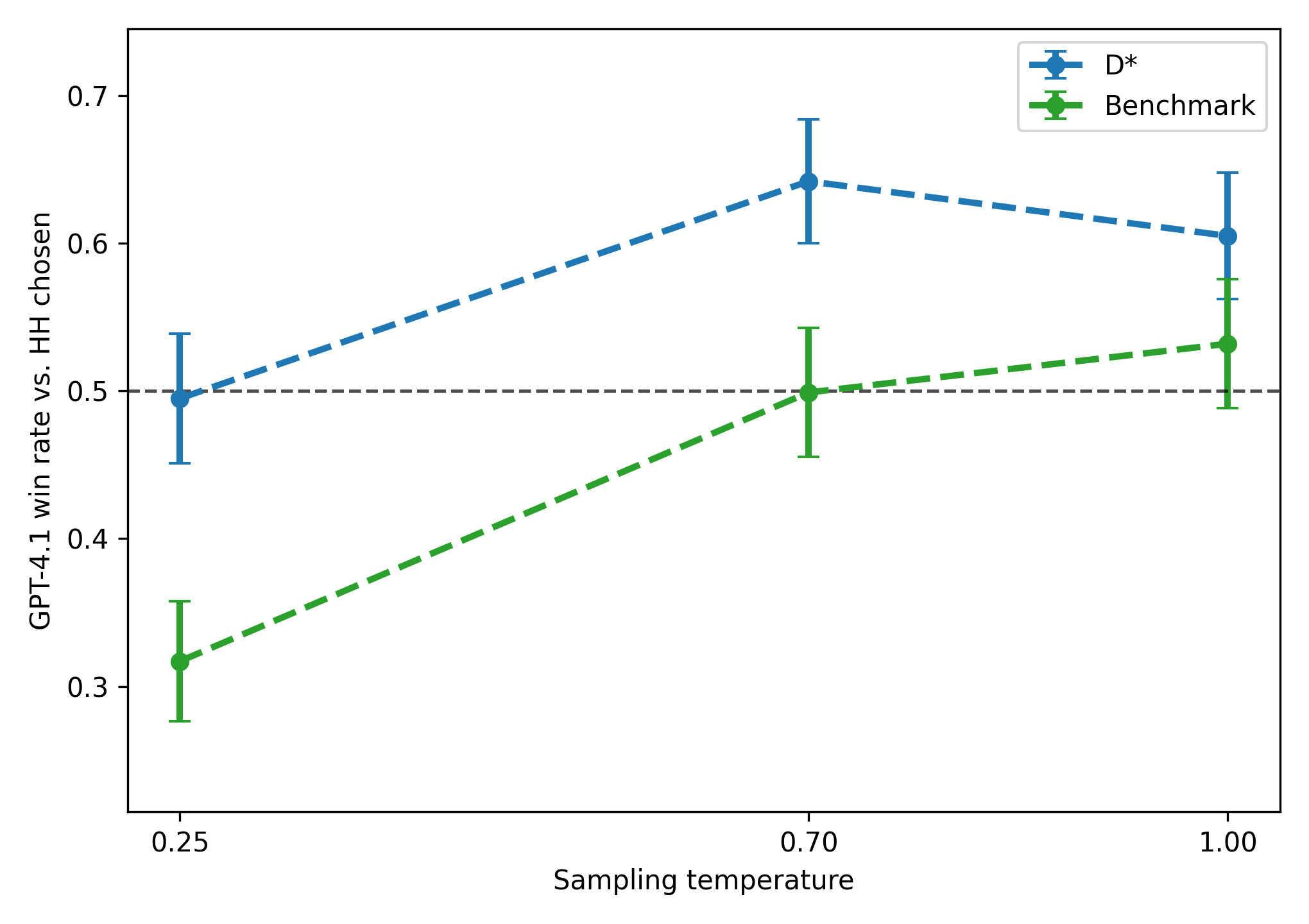}
        \caption{Sample budget $N=80,400$}
        \label{fig:hh-80400}
    \end{subfigure}\hfill
    \begin{subfigure}[t]{0.48\linewidth}
        \centering
        \includegraphics[height=2.8cm,keepaspectratio]{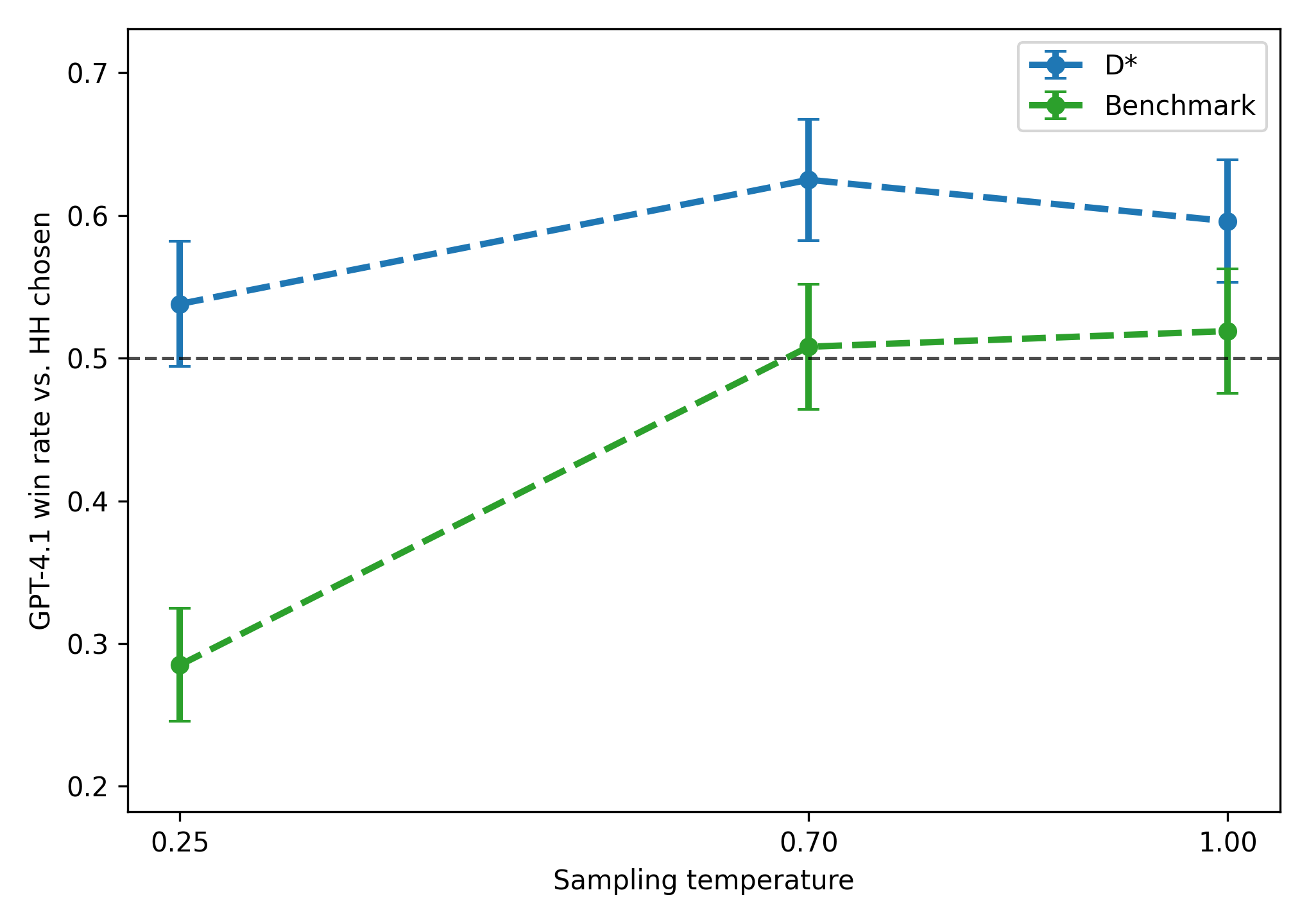}
        \caption{Sample budget $N=96,480$}
        \label{fig:hh-96480}
    \end{subfigure}
    \caption{Anthropic HH experiment}
    \label{fig:hh}
\end{figure}

\section{Discussion and Limitations}\label{s:discussion}
Our work has several limitations. The theoretical analysis relies on softmax parameterization, realizability, regularity, and coverage conditions that may hold only approximately for large language models. Moreover, our experiments do not evaluate industrial-scale language models; extending the empirical study to substantially larger models is an important direction for future work.

\newpage
\bibliographystyle{plainnat} 
\bibliography{reference}

\newpage 
\appendix

\section{Technical appendices and supplementary material}

\subsection{Motivation and interpretation for \cref{s:problem-def}}
\begin{remark}[Special cases: linear contextual and tabular policies]\label{rem:policy-special-cases}
The parametric policy class given in equation \cref{eq:policy-softmax} covers two important special cases:
\begin{enumerate}[label=(\roman*)]
    \item \textbf{Linear contextual policy.} Suppose $f_\theta$ is a linear function of the embedding vector $\phi$:
\[
f_\theta(\phi(x,y))=\langle \theta,\phi(x,y)\rangle,\qquad \theta\in\mathbb R^p.
\]
Then equation \cref{eq:policy-softmax} becomes the standard softmax policy with linear logits.
\item \textbf{Tabular policy over prompt--completion pairs.}
Consider a finite collection of prompt--completion pairs
\[
\mathcal U \doteq \{(x,y): x\in\mathcal X,\ y\in\mathcal A(x)\},\qquad |\mathcal U|=d,
\]
and let the embedding be one-hot over pairs:
\[
\phi(x,y)=\mathbbm{1}_{(x,y)}\in\mathbb R^d .
\]
If $f_\theta$ is linear in $\phi$, then
\[
f_\theta(\phi(x,y))=\langle \theta, e_{(x,y)}\rangle=\theta_{(x,y)},
\qquad \theta\in\mathbb R^d,
\]
so each prompt--completion pair $(x,y)$ has its own tabular logit parameter. Substituting into equation \cref{eq:policy-softmax} yields the tabular softmax policy: for each prompt $x$,
\[
\pi_\theta(y\mid x)
=
\frac{\exp(\theta_{(x,y)})}{\sum_{y'\in\mathcal A(x)}\exp(\theta_{(x,y')})},
\qquad y\in\mathcal A(x).
\]
\end{enumerate}
\end{remark}

\subsection{Motivation and interpretation of regularity assumptions}\label{app:assumptions}
In this subsection, we list all necessary assumptions in our analysis. 

\paragraph{Identifiability and realizability.}
Our first assumption concerns identifiability of the optimal policy parameter $\theta^\star$ from pairwise
comparison data. In our softmax parameterization \cref{eq:policy-softmax}, the map $\theta\mapsto \pi_\theta$
is not necessarily injective: different parameter vectors may induce the same policy. For example, adding a
prompt-dependent constant to all logits $\{f_\theta(\phi(x,y))\}_{y\in\mathcal A(x)}$ leaves
$\pi_\theta(\cdot\mid x)$ unchanged, so directions that only shift logits uniformly within each prompt are
intrinsically unidentifiable.

More generally, each labeled edge $e=(x,y^+,y^-)$ enters DPO only through the scalar logit
$u_\theta(e)$, i.e., the implicit-reward difference between $y^+$ and $y^-$. This motivates the
\emph{pairwise sensitivity vector}
\[
g(e;\theta)\doteq \nabla_\theta u_\theta(e),
\]
which captures how a comparison changes under infinitesimal perturbations of $\theta$.
If a direction $v$ satisfies $v^\top g(e;\theta^\star)=0$ for all queried edges $e$, then perturbing $\theta$
along $v$ does not change any logit difference that appears in the data (to first order). Since the DPO
objective depends on $\theta$ only through these logit differences, pairwise comparisons cannot identify such
directions.

To focus on directions that are learnable from comparisons, we restrict our analysis to the identifiable
subspace $H$ spanned by $\{g(e;\theta^\star)\}$ (defined formally in \cref{def:tangent-space-H}), where the
queried comparisons are informative.

\begin{remark}\label{rem:g-special-cases}
This pairwise sensitivity vector $g(e;\theta)$ has a natural form in the linear contextual setting, where
\[
f_\theta(\phi(x,y))=\langle \theta,\phi(x,y)\rangle,
\qquad \theta\in\mathbb R^p.
\]
In this case, $\nabla_\theta f_\theta(\phi(x,y))=\phi(x,y)$, and the softmax-score identity yields
\[
\nabla_\theta\log\pi_\theta(y\mid x)
=
\phi(x,y)-\E_{y'\sim\pi_\theta(\cdot\mid x)}[\phi(x,y')].
\]
Hence, for an edge $e=(x,y^+,y^-)$,
\[
g(e;\theta)
=
\beta\Big(\phi(x,y^+)-\phi(x,y^-)\Big),
\]
since the within-prompt centering terms cancel in the difference.
Thus in the linear contextual case, $g(e;\theta)$ is (up to $\beta$) the difference of the two completion
embeddings being compared. 

Particularly, in the tabular setting where $\phi(x,y)$ is a one-hot vector $\mathbbm{1}_{(x,y)}$ of the pair $(x,y)$ and $f_\theta(\phi(x,y))=\langle \theta,\mathbbm{1}_{(x,y)}\rangle$. Therefore, for an edge $e=(x,y^+,y^-)$,
\[
g(e;\theta)
=
\beta\,(\mathbbm{1}_{(x,y^+)}-\mathbbm{1}_{(x,y^-)}),
\]
i.e., $g(e;\theta)$ is (up to the factor $\beta$) the signed incidence vector of the compared within-prompt pair,
embedded in the global tabular parameterization.

\end{remark}

\begin{definition}[Identifiable tangent space $H$]\label{def:tangent-space-H}
We define the identifiable tangent space at $\theta^\star$ by
\[
H \doteq
\operatorname{span}\bigl\{g(e;\theta^\star): e=(x,y^+,y^-)\in\mathcal{E}\bigr\}
\;\subseteq\;\R^p.
\]
\end{definition}

\begin{remark}[Tabular special case: sum-to-zero reward subspace]\label{rem:H-tabular}
In the tabular setting, 
\[
f_\theta(\phi(x,y))=\theta_{(x,y)},\qquad
\pi_\theta(y\mid x)
=\frac{\exp(\beta\,\theta_{(x,y)})}{\sum_{y'\in\mathcal Y_x}\exp(\beta\,\theta_{(x,y')})}.
\]
For each fixed context $x$, adding a constant $c_x$ to all $\{\theta_{(x,y)}:y\in\mathcal Y_x\}$
does not change $\pi_\theta(\cdot\mid x)$, hence does not change $u_\theta(e)$.
Thus the identifiable directions are exactly those that are orthogonal to these
``all-ones'' shifts within each context. In this case the tangent space $H$ in \cref{def:tangent-space-H} coincides with the within-context sum-to-zero subspace
\[
H
=
\Bigl\{v:\ \sum_{y\in\mathcal Y_x} v_{(x,y)}=0
\ \ \text{for every }x\in\mathcal X\Bigr\}.
\]
\end{remark}

Throughout the analysis, we restrict attention to the identifiable subspace $H$. We also assume optimal realizability: there exists $\theta^\star\in\Theta$ such that $\pi_{\theta^\star}=\pi^\star$. Under realizability, the implicit reward $r_{\theta^\star}$ induced by the optimal policy $\pi^\star$ matches the true reward up to a prompt-only offset. That is, there exists a function $c:\mathcal X\to\mathbb R$ such that
\[
r_{\theta^\star}(x,y)=r^\star(x,y)+c(x),
\qquad \forall (x,y).
\]
Consequently, under realizability, in the BT model, using the implicit reward $r_{\theta^\star}$ to generate pairwise preference labels is equivalent to using the true reward $r^\star$: for any within-prompt comparison $e=(x,y^+,y^-)$,
\[
u_{\theta^\star}(e)
=
r_{\theta^\star}(x,y^+)-r_{\theta^\star}(x,y^-)
=
r^\star(x,y^+)-r^\star(x,y^-),
\]
since the prompt-only offset $c(x)$ cancels in the difference.

We summarize our requirements about the parameter set $\Theta$ in the following assumption.

\begin{repeatassumption}{ass:identifiable-space} [Identifiability and realizability]
    Suppose $\Theta$ is convex, compact, and satisfies \[
\Theta\;\subseteq\;\theta^\star+H.
\]
\end{repeatassumption}

\paragraph{Boundedness and smoothness.} As for the embedding vector $\phi$ and feature function $f_\theta$, we impose the following boundedness and smoothness assumption.

\begin{repeatassumption}{ass:model-bounded}[Boundedness and smoothness]
There exist constants $R_\phi,\alpha_0,\alpha_1,\alpha_2,\alpha_3<\infty$ such that for all
$\theta\in\Theta$ and all admissible $(x,y)$,
\[
\|\phi(x,y)\|_2\le R_\phi,\qquad
|f_\theta(\phi(x,y))|\le \alpha_0,
\]
\[
\|\nabla_\theta f_\theta(\phi(x,y))\|_2\le \alpha_1,\qquad
\|\nabla_\theta^2 f_\theta(\phi(x,y))\|_{\op}\le \alpha_2,\qquad
\|\nabla_\theta^3 f_\theta(\phi(x,y))\|_{\op}\le \alpha_3.
\]
Moreover, for each fixed $(x,y)$, the map $\theta\mapsto f_\theta(\phi(x,y))$ is three times
continuously differentiable on $\Theta$.
\end{repeatassumption}

\paragraph{Feature separation.} To obtain meaningful learning guarantees from pairwise comparisons, namely that the DPO solution achieves a vanishing RLHF optimality gap as the comparison budget $n$ grows, we need the candidate completion set
$\mathcal A(x)$ to be sufficiently informative. Intuitively, if for a given prompt $x$ the candidate pool
contains only near-duplicate completions (or completions that are indistinguishable under the model features),
then comparing them provides little information about how the policy should change, and certain parameter
directions cannot be learned no matter how many comparisons we collect. To rule out such degenerate cases,
we impose a mild diversity condition on the candidate set: for each prompt, $\mathcal A(x)$ should contain
at least two completions that are sufficiently different in their model-induced features (in every identifiable
direction). This is formalized by the following feature-separation assumption.

\begin{repeatassumption}{ass:feature-sep}[Feature separation on $H$]
There exists $\Delta_g>0$ such that for every $\theta\in\Theta$, every prompt $x$, and every unit vector
$v\in H$, there exist two candidates $y_1,y_2\in\mathcal A(x)$ satisfying
\[
\Big|v^\top\big(\nabla_\theta f_\theta(\phi(x,y_1))-\nabla_\theta f_\theta(\phi(x,y_2))\big)\Big|
\;\ge\;
\Delta_g.
\]
\end{repeatassumption}

\begin{remark}[Feature separation in linear contextual and tabular settings]\label{rem:feature-sep-special}
If $f_\theta(\phi)=\langle\theta,\phi\rangle$, then $\nabla_\theta f_\theta(\phi(x,y))=\phi(x,y)$ and
\cref{ass:feature-sep} reduces to a separation condition on the embeddings:
for every prompt $x$ and every unit $v\in H$, there exist $y_1,y_2\in\mathcal A(x)$ such that
\[
\big|v^\top(\phi(x,y_1)-\phi(x,y_2))\big|\ \ge\ \Delta_g.
\]
Intuitively, this means the candidate pool $\mathcal A(x)$ contains at least two completions that are well
separated in embedding space (in every identifiable direction), so within-prompt comparisons are informative.

In particular, if $\phi(x,y)=\mathbbm{1}_{(x,y)}$ and $f_\theta(\phi(x,y))=\theta_{(x,y)}$, then
$\nabla_\theta f_\theta(\phi(x,y))=\mathbbm{1}_{(x,y)}$. Hence for any two distinct candidates
$y_1\neq y_2$ in the same prompt,
\[
v^\top\big(\nabla_\theta f_\theta(\phi(x,y_1))-\nabla_\theta f_\theta(\phi(x,y_2))\big)
=
v_{(x,y_1)}-v_{(x,y_2)}.
\]
Therefore, \cref{ass:feature-sep} holds (with a constant of order one) as soon as, for each prompt,
the candidate set $\mathcal A(x)$ contains at least two distinct completions and the identifiable space $H$
includes the corresponding pair-difference directions. In this sense, feature separation is automatic in the
tabular model: distinct candidates correspond to orthogonal one-hot features, so any nontrivial comparison
produces a nonzero feature difference.
\end{remark}

\paragraph{Coverage condition.}Even with a diverse candidate pool, a meaningful guarantee further requires that the sampling design $D$ does not
ignore informative pairs. Intuitively, if $D$ concentrates on only a small subset of within-prompt pairs (or repeatedly compares near-duplicate completions), then the resulting data provide little information about some identifiable directions, and the optimality gap cannot be driven down no matter how large $n$ is.

Once each comparison edge $e$ is mapped to a pairwise sensitivity vector $g(e;\theta)$, a sampling design
$D$ over edges naturally induces a second-moment matrix that summarizes how informative the chosen comparisons
are about the parameter. In particular, the outer product $g(e;\theta)g(e;\theta)^\top$ captures the rank-one
curvature/information contribution of a single queried pair, and averaging this contribution over the edge
distribution $D$ yields the design covariance matrix. We define, at the any $\theta\in\Theta$,
\[
\Sigma_D(\theta)\doteq
\E_{e\sim D}\!\left[g(e;\theta)\,g(e;\theta)^\top\right],
\]
which will be the key object connecting pair selection to estimation error and ultimately to our
design objective.

\begin{repeatassumption}{ass:design-coverage-a}[Design coverage at the truth]
We assume that $\Sigma_D(\theta^\star)$ has a spectral gap on the identifiable subspace $H$, i.e., there exists
$\mu_\star>0$ such that
\[
v^\top \Sigma_D(\theta^\star) v \;\ge\; \mu_\star \,\|v\|_2^2,
\qquad \forall\,v\in H.
\]
\end{repeatassumption}

In some arguments, we require the above assumption to hold uniformly over a neighborhood of $\theta^\star$. Accordingly, we introduce the following localized version. Whenever invoked, the region $\mathcal R\subseteq\Theta$ will be specified (and may depend on the application), typically chosen as a neighborhood of $\theta^\star$.

\begin{repeatassumption}{ass:design-coverage-b}[Uniform design coverage]
Fix a region $\mathcal R\subseteq \Theta$. We assume that $\Sigma_D(\theta)$ has a uniform spectral gap on $H$ over $\mathcal R$, i.e., there exists
$\mu_{\mathcal R}>0$ such that
\[
v^\top \Sigma_D(\theta) v \;\ge\; \mu_{\mathcal R}\,\|v\|_2^2,
\qquad \forall\,v\in H,\ \forall\,\theta\in\mathcal R.
\]
\end{repeatassumption}

\begin{remark}\label{rem:design-coverage}
For any direction $v\in\mathbb R^p$,
\[
v^\top \Sigma_D(\theta)\, v
=
\E_{e\sim D}\!\big[(v^\top g(e;\theta))^2\big],
\]
so $\Sigma_D(\theta)$ summarizes the second-moment geometry of the pairwise sensitivity vectors under the
sampling design $D$.

In the linear contextual setting, $g(e;\theta)=\beta(\phi(x,y^+)-\phi(x,y^-))$ does not depend on $\theta$,
so $\Sigma_D(\theta)$ is independent of $\theta$ and the coverage condition reduces to a standard
full-rank/nondegeneracy condition on the second moment of feature differences under the design $D$.
Concretely, for $v\in H$,
\[
v^\top \Sigma_D v
=
\beta^2\,\E_{e\sim D}\!\Big[(v^\top(\phi(x,y^+)-\phi(x,y^-)))^2\Big].
\]
Thus coverage holds whenever, under the sampling design $D$, the pairwise feature differences
$\phi(x,y^+)-\phi(x,y^-)$ span the identifiable subspace $H$ and are not concentrated in a lower-dimensional
subspace. For example, it suffices that for every unit direction $v\in H$ there is a non-negligible probability
of drawing an edge $e$ such that $|v^\top(\phi(x,y^+)-\phi(x,y^-))|$ is bounded away from zero; in this case
the second moment above is uniformly positive and $\Sigma_D$ is positive definite on $H$.

In the tabular setting over prompt--completion pairs, $g(e;\theta)=\beta(\mathbbm{1}_{(x,y^+)}-\mathbbm{1}_{(x,y^-)})$ and
$\Sigma_D(\theta)=\beta^2 L_D$ is (up to $\beta^2$) the weighted graph Laplacian induced by the edge distribution $D$.
Thus Assumption~\ref{ass:design-coverage-a} requires that the Laplacian has a positive spectral gap on the identifiable
subspace, i.e., its second-smallest eigenvalue satisfies $\lambda_2(L_D)>0$, equivalently, the comparison graph is connected (within each connected component determined by the prompt structure). This is a common assumption needed for estimating a BT model in the tabular setting \citep{shah2016estimation}.

Assumptions~\ref{ass:design-coverage-a} and~\ref{ass:design-coverage-b} are two versions of the same requirement, but at different levels of uniformity. Assumption~\ref{ass:design-coverage-a}
only requires a spectral gap at the truth $\theta^\star$; this is sufficient for our upper bound, since
the trace criterion and estimation error control are evaluated at $\theta^\star$.
Assumption~\ref{ass:design-coverage-b} strengthens this to a uniform spectral gap over a region
$\mathcal R$ (e.g., $\mathrm{supp}(\rho)$); this is needed for the Bayesian/Van Trees lower bound, where
we average over $\theta\sim\rho$ and therefore must rule out degeneracy throughout the relevant parameter region.
\end{remark}

\paragraph{Optimization landscape.}Our next assumption concerns the optimization landscape of the DPO risk.
When $f_\theta$ is linear in the embedding (e.g., the linear contextual case), the DPO objective is convex in
$\theta$; moreover, under standard design coverage conditions the population and empirical risks are strongly convex
on the identifiable space and admit a unique minimizer. In contrast, for general nonlinear $f_\theta$ (e.g., a
neural network score model), both the population risk $L(\theta)$ and the empirical risk $L_n(\theta)$ can be
nonconvex, and multiple global minimizers may exist.

In our setting, the target parameter $\theta^\star$ is defined a priori by the RLHF objective via
realizability, i.e., $\pi_{\theta^\star}=\pi^\star$ where $\pi^\star\in\arg\max_\pi J(\pi)$.
In later analysis we show that, under realizability and well-specified preference labels, this same $\theta^\star$
is also a global minimizer of the population DPO risk. To avoid ambiguity when the population DPO risk admits
multiple minimizers, we impose the following uniqueness assumption.

\begin{repeatassumption}{ass:unique-minimizer-pop-risk}[Unique population minimizer]
The population DPO risk $L(\theta)$ has a unique global minimizer over $\Theta$.
\end{repeatassumption}

\begin{remark}\label{rem:unique-minimizer}
When $L(\theta)$ has multiple global minimizers, the empirical DPO estimator $\hat\theta_n$ may converge to
different minimizers depending on optimization and initialization, while all such limits induce the same minimum
population risk. Since our performance bounds are stated in terms of the estimation error relative to the RLHF
target $\theta^\star$, we impose \cref{ass:unique-minimizer-pop-risk} to ensure that the population
DPO minimizer coincides with the RLHF-optimal parameter and to eliminate ambiguity in the definition of the
target.
\end{remark}

\paragraph{Prior regularity.}
Our final assumption introduces a Bayesian formulation for the optimal parameter $\theta^\star$.
When prior information about $\theta^\star$ is available, our analysis yields an information-theoretic
(Bayesian) lower bound on the RLHF optimality gap under an arbitrary sampling design $D$.

\begin{repeatassumption}{ass:prior-vantrees}[Prior regularity]
Suppose the optimal parameter $\theta^\star$ follows a prior density $\rho$ supported on $\Theta$.
Moreover, $\rho$ satisfies:
\begin{enumerate}
    \item $\rho\in C^1(\Theta)$ and $\rho(\theta)>0$ for all $\theta\in \mathrm{int}(\Theta)$;
    \item $\rho(\theta)=0$ for all $\theta\in \partial\Theta$;
    \item The prior Fisher information is finite:
    \[
    \int_\Theta \|\nabla \log \rho(\theta)\|_2^2\,\rho(\theta)\,d\theta < \infty .
    \]
\end{enumerate}
\end{repeatassumption}

\begin{remark}\label{rem:van-tree}
\cref{ass:prior-vantrees} is a standard regularity condition on the prior density. It ensures that the
quantity $\int_\Theta \|\nabla\log\rho(\theta)\|_2^2\,\rho(\theta)\,d\theta$ is finite and that the
integration-by-parts steps required by the Bayesian information inequality (Van Trees) in our lower-bound
argument are valid.

Recall that under realizability, $\theta^\star$ parameterizes the RLHF-optimal policy $\pi^\star$. Moreover, by
the Gibbs closed form \cref{eq:gibbs-opt-form}, $\pi^\star$ is uniquely determined by the reference policy
$\pi_0$ and the reward function $r^\star$. Hence one may equivalently place a prior on $r^\star$ and view the
resulting distribution as an induced (push-forward) prior on $\theta^\star$; our analysis only uses the prior on
$\theta^\star$ itself.
\end{remark}

\section{Preliminary on information theory}\label{app:info-theory}

In this section, we provide more interpretations for \cref{def:Fisher}.

\begin{remark}[Fisher information under the softmax policy parameterization]
Under the softmax policy \cref{eq:policy-softmax},
the score $\nabla_\theta\log\pi_\theta(y\mid x)$ can be written explicitly in terms of
$f_\theta$ and its gradient:
\begin{equation}\label{eq:softmax-score}
\nabla_\theta\log\pi_\theta(y\mid x)
=
\nabla_\theta f_\theta(\phi(x,y))
-
\E_{y'\sim\pi_\theta(\cdot\mid x)}\!\big[\nabla_\theta f_\theta(\phi(x,y'))\big].
\end{equation}
Therefore, the Fisher information matrix admits the covariance form
\begin{equation}\label{eq:fisher-softmax-cov}
I(\theta)
=
\E_x\!\left[
\Cov_{y\sim\pi_\theta(\cdot\mid x)}
\big(\nabla_\theta f_\theta(\phi(x,y))\big)
\right].
\end{equation}

If $f_\theta(\phi)=\langle\theta,\phi\rangle$, then $\psi_\theta(x,y)=\phi(x,y)$ and $I(\theta)$ is the covariance matrix of the embedding vector $\phi(x,y)$
\[
I(\theta)
=
\E_x\!\left[\Cov_{y\sim\pi_\theta(\cdot\mid x)}\big(\phi(x,y)\big)\right].
\]
In particular, in the tabular setting where $\phi(x,y)=e_{(x,y)}$ is an one-hot vector, 
\[
I(\theta)
=
\E_x\!\left[\Cov_{y\sim\pi_\theta(\cdot\mid x)}\big(e_{(x,y)}\big)\right]
=
\E_x\!\left[\text{diag}\!\big(\pi_\theta(\cdot\mid x)\big)-\pi_\theta(\cdot\mid x)\pi_\theta(\cdot\mid x)^\top\right].
\]
\end{remark}

\begin{remark}[Interpretation of Fisher information in our setting]\label{rem:fisher-interpretation}
The Fisher information matrix $I(\theta)$ measures how sensitively the policy distribution
$\pi_\theta(\cdot\mid x)$ changes with the parameter $\theta$. The score $s_\theta(x,y)=\nabla_\theta\log\pi_\theta(y\mid x)$
is the infinitesimal change of the log-probability of outcome $y$ under a perturbation of $\theta$; taking the
second moment and averaging over $(x,y)$ therefore quantifies the amount of ``statistical signal'' carried by a
single draw $y\sim\pi_\theta(\cdot\mid x)$ about the parameter.

In particular, for any direction $v\in\mathbb R^p$,
\[
v^\top I(\theta)v
=
\mathbb E_x\,\mathbb E_{y\sim\pi_\theta(\cdot\mid x)}
\!\left[(v^\top s_\theta(x,y))^2\right]
\]
is the expected squared sensitivity of $\log\pi_\theta(y\mid x)$ along $v$.
Equivalently, under standard regularity conditions, $I(\theta)$ coincides with the local curvature of the KL
divergence: for small $\Delta$ in the identifiable directions,
\[
\mathbb E_x\,\KL\!\big(\pi_{\theta+\Delta}(\cdot\mid x)\,\|\,\pi_\theta(\cdot\mid x)\big)
=
\frac12\,\Delta^\top I(\theta)\Delta + o(\|\Delta\|_2^2).
\]
Thus, in our analysis, $I(\theta^\star)$ plays the role of a local geometry/curvature matrix around the optimal
policy: it specifies which parameter directions change the policy the most, and it is the natural weight matrix
that appears in the trace-form bounds and the resulting sampling-design objective.
\end{remark}

\begin{remark}[Equivalent curvature parametrizations]\label{rem:I-W-H}
Define
\begin{equation}\label{eq:W-main}
W(\theta^\star)
\doteq
\mathbb E_x\!\left[
\operatorname{Cov}_{y\sim\pi_{\theta^\star}(\cdot\mid x)}
\big(\nabla_\theta r_{\theta^\star}(x,y)\big)
\right].
\end{equation}
For statistical clarity, we express our main results using the policy Fisher information \(I(\theta^\star)\).
Under the softmax policy parametrization considered in this work, the implicit reward is aligned with the
policy log-likelihood gradient in the sense that
\[
\nabla_\theta r_{\theta^\star}(x,y)=\beta\,\nabla_\theta \log\pi_{\theta^\star}(y\mid x),
\qquad \text{for all \(x,y\)}.
\]
Consequently,
\[
W(\theta^\star)=\beta^2\,I(\theta^\star).
\]
The matrix \(W(\theta^\star)\) quantifies how the implicit reward changes with the policy parameter along different
directions: for any direction \(v\), the quadratic form \(v^\top W(\theta^\star)v\) equals the expected (over prompts \(x\))
variance, under \(y\sim\pi_{\theta^\star}(\cdot\mid x)\), of the directional derivative
\(\langle v,\nabla_\theta r_{\theta^\star}(x,y)\rangle\). In this sense, \(I(\theta^\star)\) (and $W(\theta^\star)$) captures the sensitivity of the
implicit reward to perturbations of \(\theta\), measured under the on-policy distribution.
\end{remark}

\section{Proofs of \cref{thm:informal}}\label{app:proof-informal}
In this section, we prove our main result \cref{thm:informal} following the road map in \cref{s:main-results}. Specifically, the statement (i) in \cref{thm:informal} is proved in \cref{thm:dpo-upper-general-final}, and the statement (ii) in \cref{thm:informal} is proved in \cref{prop:vt-lb-trace}.

\subsection{Step 1: Sandwich bound for the RLHF optimality gap}\label{ss:step1}

We first express the RLHF optimality gap exactly as a reverse KL divergence to the optimal policy.

\begin{lemma}[RLHF gap as reverse KL]\label{lem:gap-kl-contextual}
Under \cref{ass:identifiable-space}, for any policy $\pi$,
\[
J(\pi^\star)-J(\pi)
=
\beta\,\E_x\!\left[
\KL\!\big(\pi(\cdot\mid x)\,\|\,\pi^\star(\cdot\mid x)\big)
\right].
\]
\end{lemma}

\begin{proof}[Proof of \cref{lem:gap-kl-contextual}]
Fix $x$. By \eqref{eq:gibbs-opt-form},
\[
\log \pi^\star(y\mid x)
=
\log \pi_0(y\mid x)+\frac{r^\star(x,y)}{\beta}-\log Z(x).
\]
Hence
\[
\begin{aligned}
\KL\!\big(\pi(\cdot\mid x)\,\|\,\pi^\star(\cdot\mid x)\big)
&=
\sum_{y}\pi(y\mid x)\log\frac{\pi(y\mid x)}{\pi^\star(y\mid x)} \\
&=
\sum_y \pi(y\mid x)\log\frac{\pi(y\mid x)}{\pi_0(y\mid x)}
-\frac{1}{\beta}\sum_y \pi(y\mid x)r^\star(x,y)
+\log Z(x) \\
&=
\KL\!\big(\pi(\cdot\mid x)\,\|\,\pi_0(\cdot\mid x)\big)
-\frac{1}{\beta}\,\E_{y\sim\pi(\cdot\mid x)}[r^\star(x,y)]
+\log Z(x).
\end{aligned}
\]
Rearranging gives
\[
\E_{y\sim\pi(\cdot\mid x)}[r^\star(x,y)]
-\beta\,\KL\!\big(\pi(\cdot\mid x)\,\|\,\pi_0(\cdot\mid x)\big)
=
\beta\log Z(x)
-\beta\,\KL\!\big(\pi(\cdot\mid x)\,\|\,\pi^\star(\cdot\mid x)\big).
\]
Now take expectation over $x$:
\[
J(\pi)
=
\beta\,\E_x[\log Z(x)]
-\beta\,\E_x\!\left[\KL\!\big(\pi(\cdot\mid x)\,\|\,\pi^\star(\cdot\mid x)\big)\right].
\]
Setting $\pi=\pi^\star$ yields
\[
J(\pi^\star)=\beta\,\E_x[\log Z(x)].
\]
Subtracting gives
\[
J(\pi^\star)-J(\pi)
=
\beta\,\E_x\!\left[\KL\!\big(\pi(\cdot\mid x)\,\|\,\pi^\star(\cdot\mid x)\big)\right].
\]
\end{proof}

Hence, bounding the optimality gap reduces to bounding a KL divergence.  
Define
\begin{equation}\label{eq:def-F}
F(\theta)
\doteq
\beta\,\E_x\!\left[
\KL\!\big(\pi_\theta(\cdot\mid x)\,\|\,\pi_{\theta^\star}(\cdot\mid x)\big)
\right].
\end{equation}

\begin{lemma}[Smoothness and Hessian identity]\label{lem:F-hess}
Suppose \cref{ass:model-bounded} holds. Then $F\in C^2(\Theta)$, $\nabla F(\theta^\star)=0$, and $\nabla^2F(\theta^\star)=\beta I(\theta^\star).$
\end{lemma}

\begin{proof}[Proof of \cref{lem:F-hess}]
Since $\mathcal Y(x)$ is finite and $f_\theta$ has bounded first/second derivatives in $\theta$
(\cref{ass:model-bounded}), $\log\pi_\theta(y|x)$ is twice differentiable and differentiation
can be interchanged with the finite sums defining $\pi_\theta$ and with $\mathbb E_x$.
A standard score-trick argument yields $\nabla F(\theta^*)=0$ and the Fisher representation
$\nabla^2F(\theta^*)=\beta I(\theta^*)$.
\end{proof}

By \cref{lem:gap-kl-contextual,lem:F-hess}, the RLHF optimality gap can be written as \(F(\theta)\),
a KL divergence to \(\pi_{\theta^\star}\). In particular, in a neighborhood of \(\theta^\star\), changes in the gap are controlled by the quadratic form
induced by \(I(\theta^\star)\).

\begin{lemma}[Global Fisher sandwich]\label{lem:Fisher-sandwich}
Suppose \cref{ass:identifiable-space,ass:model-bounded,ass:feature-sep} hold. Then there exists a constant \(\underline\mu>0\), depending only on the fixed constants in the standing
assumptions, such that for all \(\theta\in\Theta\),
\begin{equation}\label{eq:fisher_uniform_bounds}
\underline\mu\,  \preceq   I(\theta)  \preceq 4\alpha_1^2\, .
\end{equation}
Moreover, with $m_I\doteq \frac{\underline\mu}{4\alpha_1^2}$ and $M_I \doteq \frac{4\alpha_1^2}{\underline\mu}$, we have
\begin{equation}\label{eq:global_relative_curv_fisher}
m_I\,  I(\theta^\star) 
\;\preceq\;
  I(\theta) 
\;\preceq\;
M_I\,  I(\theta^\star) ,
\qquad \forall\theta\in\Theta.
\end{equation}
\end{lemma}

\begin{proof}[Proof of \cref{lem:Fisher-sandwich}]
Fix $(x,\theta)$. Write
\[
g_\theta(x,y)\doteq \nabla_\theta f_\theta(\phi(x,y)),\qquad
\bar g_\theta(x)\doteq \E_{y'\sim\pi_\theta(\cdot\mid x)}[g_\theta(x,y')].
\]
By softmax calculus,
\[
s_\theta(y\mid x)=\nabla_\theta\log\pi_\theta(y\mid x)=g_\theta(x,y)-\bar g_\theta(x).
\]
Hence
\[
\|s_\theta(y\mid x)\|_2
\le \|g_\theta(x,y)\|_2+\|\bar g_\theta(x)\|_2
\le \alpha_1+\alpha_1=2\alpha_1,
\]
where \cref{ass:model-bounded} was used.
Therefore, for any $v\in H$ with $\|v\|_2=1$,
\[
v^\top I(\theta)v
=\E_x\E_{y\sim\pi_\theta}[ (v^\top s_\theta(y\mid x))^2 ]
\le \E_x\E_y[\|s_\theta(y\mid x)\|_2^2]
\le 4\alpha_1^2.
\]
So
\[
I(\theta) \preceq 4\alpha_1^2 I,\qquad \forall\theta\in\Theta.
\]

By \cref{lem:sep-implies-fisher}, which relies on \cref{ass:feature-sep,ass:model-bounded}, 
\[
I(\theta) \succeq \underline\mu I,\qquad \forall\theta\in\Theta.
\]
Thus, we prove equation \eqref{eq:fisher_uniform_bounds}. In particular, at $\theta^\star$,
\[
\underline\mu\,I \preceq  I(\theta^\star) \preceq 4\alpha_1^2\,I.
\]

For any $u\in H$,
\[
u^\top  I(\theta) u \ge \underline\mu\,\|u\|_2^2,
\qquad
u^\top  I(\theta^\star)u \le 4\alpha_1^2\,\|u\|_2^2.
\]
Hence
\[
u^\top  I(\theta) u
\ge \frac{\underline\mu}{4\alpha_1^2}\,
u^\top  I(\theta^\star) u
=: m_I\,u^\top  I(\theta^\star)u.
\]
Similarly,
\[
u^\top  I(\theta) u \le 4\alpha_1^2\,\|u\|_2^2,
\qquad
u^\top  I(\theta^\star) u \ge \underline\mu\,\|u\|_2^2,
\]
so
\[
u^\top  I(\theta) u
\le \frac{4\alpha_1^2}{\underline\mu}\,
u^\top  I(\theta^\star) u
=: M_I\,u^\top  I(\theta^\star) u.
\]
Therefore, for all $\theta\in\Theta$,
\[m_I\,I(\theta^\star)
\;\preceq\;
 I(\theta)
\;\preceq\;
M_I\,I(\theta^\star),\]
with
\[
m_I=\frac{\underline\mu}{4\alpha_1^2},
\qquad
M_I=\frac{4\alpha_1^2}{\underline\mu}.
\]

\end{proof}

Next, we first show that, when the feature function $f_\theta$ is linear, the KL divergence (hence the RLHF gap) is globally equivalent to a weighted quadratic form of the estimation error.

\begin{proposition}[Linear features: global quadratic sandwich for the optimality gap]\label{prop:KL-sandwich-global}
Suppose $f_\theta(\phi)=\theta^\top\phi$. Under \cref{ass:identifiable-space,ass:model-bounded,ass:feature-sep}, for all $\theta\in\Theta$, letting $\Delta=\theta-\theta^\star$, we have:
\begin{equation}\label{eq:KL-sandwich-global}
\frac{\beta m_I}{2}\,\Delta^\top I(\theta^\star)\Delta
\ \le\
\beta\,\E_x\KL\!\big(\pi_\theta(\cdot\mid x)\,\|\,\pi_{\theta^\star}(\cdot\mid x)\big)
\ \le\
\frac{\beta M_I}{2}\,\Delta^\top I(\theta^\star)\Delta.
\end{equation}
Equivalently, by \cref{lem:gap-kl-contextual},
\[
\frac{\beta m_I}{2}\,\Delta^\top I(\theta^\star)\Delta
\ \le\
J(\pi^\star)-J(\pi_\theta)
\ \le\
\frac{\beta M_I}{2}\,\Delta^\top I(\theta^\star)\Delta.
\]
\end{proposition}

\begin{proof}[Proof of \cref{prop:KL-sandwich-global}]
By \cref{lem:F-hess}, $F\in C^2(\Theta)$ and $F(\theta^\star)=0$. Fix an arbitrary parameter $\theta\in\Theta$. Define $\Delta \doteq \theta-\theta^\star$, and $\theta_t\doteq\theta^\star+t\Delta$. We verify below that $\frac{d}{dt} F(\theta_t)=\beta t\Delta^\top I(\theta_t)\Delta$. Given this equation,
\[
 F(\theta) = \beta\int_0^1 t\Delta^\top I(\theta_t)\Delta\, dt.
\]
Applying \cref{lem:Fisher-sandwich} along the segment gives
\[
m_I\,\Delta^\top I(\theta^\star)\Delta
\le
\Delta^\top I(\theta_t)\Delta
\le
M_I\,\Delta^\top I(\theta^\star)\Delta,
\]
which proves \cref{eq:KL-sandwich-global}.

It remains to verify $\frac{d}{dt} F(\theta_t)=\beta t\Delta^\top I(\theta_t)\Delta$. Define 
\[
\Psi_x(\theta) \doteq\log\big(\sum_{y\in\mathcal{A}(x)}\exp(\theta^\top\phi(x,y))\big),\qquad \Psi(\theta)\doteq\E_x[\Psi_x(\theta)].
\]
By definition of $F(\theta)$, 
\[
F(\theta)/\beta  = \Delta^\top\nabla\Psi(\theta)-\Psi(\theta)+\Psi(\theta^\star).
\]
By differentiation, 
    \begin{equation*}
    \begin{split}
        \nabla\Psi_x(\theta)
   &=\E_{Y\sim\pi_\theta(\cdot\mid x)}[\phi(x,Y)],\\
 \nabla^2\Psi_x(\theta)
   &=\Cov_{Y\sim\pi_\theta(\cdot\mid x)}(\phi(x,Y)).
    \end{split}
    \end{equation*}
Therefore, by \cref{eq:fisher-softmax-cov}, $\nabla^2\Psi(\theta)=I(\theta)$ for any $\theta\in\Theta$. This implies 
\[
\nabla F(\theta)/\beta = \nabla^2\Psi(\theta)\Delta=I(\theta)\Delta,\quad\forall \theta\in\Theta.
\]
Since $\theta_t-\theta=t\Delta$, we have $\frac{d}{dt} F(\theta_t)=\nabla F(\theta_t)^\top\Delta=\beta t\Delta^\top I(\theta_t)\Delta$.
\end{proof}

When $f_\theta$ is nonlinear, it is possible that two distant parameter $\theta_1$ and $\theta_2$ induces almost same policies $\pi_{\theta_1}$ and $\pi_{\theta_2}$. In this case, large estimation error in parameter $\theta$ does not mean the optimality gap must be large. However, when a policy $\pi_\theta$ is close to the optimal policy $\pi^\star$, then the optimality gap must be small. Therefore, in the general nonlinear case, we show that the optimality gap is lower bounded by estimation error in the policy (measured by the squared Hellinger distance), rather than parameters. Moreover, we show that the upper bound in \cref{prop:KL-sandwich-global} still holds. 

\begin{proposition}[Nonlinear features: global quadratic sandwich for the optimality gap]\label{prop:KL-sandwich-nonlinear}
    Suppose $f_\theta$ is general. Under \cref{ass:identifiable-space,ass:model-bounded,ass:feature-sep}, for all $\theta\in\Theta$, letting $\Delta=\theta-\theta^\star$, we have:
\begin{equation}\label{eq:KL-sandwich-global}
\E_x\bigg[\|\sqrt{\pi_\theta(\cdot|x)}-\sqrt{\pi_{\theta^\star}(\cdot|x)}\|^2_2\bigg]\ \le\
\,\E_x\KL\!\big(\pi_\theta(\cdot\mid x)\,\|\,\pi_{\theta^\star}(\cdot\mid x)\big)
\ \le\
\frac{ M_I}{8}\,\Delta^\top I(\theta^\star)\Delta.
\end{equation}
Equivalently, by \cref{lem:gap-kl-contextual},
\[
\beta\E_x\bigg[\|\sqrt{\pi_\theta(\cdot|x)}-\sqrt{\pi_{\theta^\star}(\cdot|x)}\|^2_2\bigg]\ \le\
J(\pi^\star)-J(\pi_\theta)
\ \le\
\frac{\beta M_I}{8}\,\Delta^\top I(\theta^\star)\Delta.
\]
\end{proposition}

\begin{proof}[Proof of \cref{prop:KL-sandwich-nonlinear}]
    By the standard inequality between KL divergence and squared Hellinger distance, we have 
    \[
   \|\sqrt{\pi_\theta(\cdot|x)}-\sqrt{\pi_{\theta^\star}(\cdot|x)}\|^2_2\ \le\ \KL\!\big(\pi_\theta(\cdot\mid x)\,\|\,\pi_{\theta^\star}(\cdot\mid x)\big),
    \]
    which implies the lower bound.

    Next, we prove the upper bound. Fix an arbitrary $\theta\in\Theta$. Define $\Delta \doteq \theta-\theta^\star$, $d_x(y)\doteq f_\theta(\phi(x,y))-f_{\theta^\star}(\phi(x,y))$, and 
    \[
    q_t(y|x)\doteq\frac{\exp(f_{\theta^\star}(\phi(x,y))+td_x(y))}{\sum_{y'\in\mathcal{A}(x)}\exp(f_{\theta^\star}(\phi(x,y'))+td_x(y'))}.
    \]
    For simplicity of notation, define $\Psi_x(t)\doteq\log\big(\sum_{y\in\mathcal{A}(x)}\exp(f_{\theta^\star}(\phi(x,y))+td_x(y))\big)$, and $K_x(t)\doteq\KL(q_t(\cdot|x)\|\pi_{\theta^\star}(\cdot|x))$. Since $K_x(0)=0$ and $K_x(1) = \KL(\pi_\theta(\cdot|x)\|\pi_{\theta^\star}(\cdot|x))$,
    \[
    F(\theta)/\beta = \E_x[K_x(1)] = \E_x\big[\int_0^1K_x'(t)\,dt\big].
    \]
  By differentiation, $K_x'(t) = t\Psi''_x(t)$, $\Psi'_x(t) = \E_{y\sim q_t}\big[d_x(y)\big]$, and $\Psi''_x(t)=\Var_{y\sim q_t(\cdot|x)}(d_x(y))\le \frac{\alpha_1^2}{\underline\mu}\Delta^\top I(\theta^\star)\Delta$. This implies 
  \[
  F(\theta)/\beta=\E_x\big[\int_0^1t\Var_{y\sim q_t(\cdot|x)}(d_x(y))\, dt\big]\le \frac{\alpha_1^2}{2\underline\mu}\Delta^\top I(\theta^\star)\Delta.
  \]
\end{proof}

\subsection{Step 2: upper bound for the DPO minimizer}\label{s:upper-bound}
Given a sampling design $D$, suppose we collect $n$ samples $e_i=(x_i,y^+_i,y^-_i) \stackrel{\text{i.i.d.}}{\sim} D$. We can approximate the design covariance matrix $\Sigma_D(\theta)$ by the following sample covariance matrix
\[
\widehat\Sigma_n(\theta)
\;\doteq\;
\frac1n\sum_{i=1}^n g(e_i;\theta)g(e_i;\theta)^\top.
\]
We start with some useful concentration result, which lower bound the empirical covariance matrix by the sample covariance matrix.

\begin{proposition}[Empirical covariance concentration]\label{prop:Bernstein-H}
Suppose \cref{ass:identifiable-space,ass:model-bounded,ass:design-coverage-a} hold. Assume $e_i=(x_i,y^+_i,y^-_i) \stackrel{\text{i.i.d.}}{\sim} D$.
Then for any $\delta\in(0,1)$, if
\[
n \;\ge\; 8\,\frac{G^2}{\lambda_{\min}^H(\Sigma_D(\theta^\star))}\,
\log\!\frac{\dim(H)}{\delta},
\]
we have with probability at least $1-\delta$:
\[
\widehat\Sigma_n(\theta^\star)  \succeq \frac12\,\Sigma_D(\theta^\star) .
\]
Consequently, on this event,
\[
\widehat\Sigma_n^\dagger(\theta^\star) 
\preceq
2\,\Sigma_D^\dagger(\theta^\star) .
\]
For any $\delta\in(0,1)$, define $n_\Sigma(\delta):= 8\,\frac{G^2}{\lambda_{\min}^H(\Sigma_D)}\,
\log\!\frac{\dim(H)}{\delta}$.
\end{proposition}

\begin{proof}[Proof of \cref{prop:Bernstein-H}]
To simplify notation, write
\[
\widehat\Sigma_n
\doteq \frac1n\sum_{i=1}^n Y_i,
\qquad
\Sigma_D\doteq \mathbb E[Y_i],
\qquad
Y_i\doteq g(e_i;\theta^\star)g(e_i;\theta^\star)^\top.
\]
By construction, each \(Y_i\succeq 0\). Moreover, since \(\|g(e;\theta^\star)\|_2\le G\),
we have
\[
0 \preceq Y_i \preceq \|g(e_i;\theta^\star)\|_2^2\, I \preceq G^2 I
\quad\text{(on }H\text{)}.
\]
Let \(\lambda_D\doteq\lambda_{\min}^H(\Sigma_D)>0\), and define $\widetilde Y_i
\doteq\Sigma_D^{\dagger/2}Y_i\Sigma_D^{\dagger/2}. $ These matrices are independent and PSD, and satisfy
\[
\mathbb E[\widetilde Y_i]=I_H,
\qquad
0\preceq\widetilde Y_i
\preceq G^2\Sigma_D^\dagger
\preceq\frac{G^2}{\lambda_D}I_H.
\]

We apply the matrix Chernoff inequality for the minimum eigenvalue
\citep[e.g., Theorem~5.1, Corollary 5.2, Remark 5.3]{tropp2012userfriendly} to the independent PSD matrices $\{\widetilde Y_i\}$: 
\[
\Pr\!\left(
\lambda_{\min}^H\!\left(
\frac1n\sum_{i=1}^n\widetilde Y_i
\right)\le\frac12
\right)
\le
\dim(H)\exp\!\left(-\frac{n\lambda_D}{8G^2}\right).
\]
Thus, if
\[
n\ge
8\,\frac{G^2}{\lambda_D}
\log\!\frac{\dim(H)}{\delta},
\]
then, with probability at least $1-\delta$,
\[
\Sigma_D^{\dagger/2}\widehat\Sigma_n
\Sigma_D^{\dagger/2}
=
\frac1n\sum_{i=1}^n\widetilde Y_i
\succeq\frac12 I_H.
\]
Multiplying both sides on the left and right by
$\Sigma_D^{1/2}$ yields $\widehat\Sigma_n\succeq\frac12\Sigma_D.$ Finally, both matrices are therefore positive definite on $H$,
so inverse monotonicity gives $ \widehat\Sigma_n^\dagger\preceq2\Sigma_D^\dagger.$
\end{proof}

\begin{corollary}\label{cor:emp-excitation}
    Under the condition of \cref{prop:Bernstein-H}, if $n\ge n_\Sigma(\delta)$,
we have with probability at least $1-\delta$:
\[
 \widehat\Sigma_n(\theta^\star)   \succeq \frac 12\lambda^H_{\min}(\Sigma_D(\theta^\star))I  ,
\]
where $\lambda^H_{\min}(\Sigma_D(\theta^\star))>0$ by \cref{ass:design-coverage-a}.
\end{corollary}

\begin{proof}[Proof of \cref{cor:emp-excitation}]
Let \(H\) be the identifiable subspace and note that, by construction,
\(\widehat\Sigma_n(\theta^\star)\) and \(\Sigma_D(\theta^\star)\) act on \(H\).
By \cref{prop:Bernstein-H}, if
\(n\ge n_\Sigma(\delta)\), where
\[
n_\Sigma(\delta)
:=
8\,\frac{G^2}{\lambda_{\min}^H(\Sigma_D(\theta^\star))}\,
\log\!\frac{\dim(H)}{\delta},
\]
then with probability at least \(1-\delta\),
\[
\widehat\Sigma_n(\theta^\star)\succeq \frac12\,\Sigma_D(\theta^\star).
\]
Since \cref{ass:design-coverage-a} ensures \(\lambda_{\min}^H(\Sigma_D(\theta^\star))>0\), we further have, when restricting to \(H\),
\[
\Sigma_D(\theta^\star)\succeq \lambda_{\min}^H(\Sigma_D(\theta^\star))\,I.
\]
Combining the two displays yields, on the same
event,
\[
\widehat\Sigma_n(\theta^\star)
\succeq
\frac12\,\lambda_{\min}^H(\Sigma_D(\theta^\star))\,I.
\]
This proves the claim.
\end{proof}

\begin{remark}[Tabular setting]\label{rem:emp-excitation-tabular}
Consider the tabular model with $d$ items, parameter $\theta\in\mathbb R^d$, and
identifiable subspace
\[
H=\Bigl\{v\in\mathbb R^d:\ \mathbf 1^\top v=0\Bigr\}.
\]
For a comparison sample $i$ with queried pair $(a_i,b_i)$, 
\[
g(e_i;\theta^\star)=\nabla_\theta u_i(\theta^\star)
=\beta\,(e_{a_i}-e_{b_i}).
\]
Therefore,
\[
\widehat\Sigma_n
=\frac1n\sum_{i=1}^n g(e_i)g(e_i)^\top
=\beta^2\frac1n\sum_{i=1}^n (e_{a_i}-e_{b_i}) (e_{a_i}-e_{b_i})^\top
=\beta^2\,\widehat L,
\]
where $\widehat L$ is the (weighted) empirical graph Laplacian of the sampled comparison graph. Hence \cref{cor:emp-excitation},
\[
\widehat\Sigma_n\succeq \lambda^H_{\min}(\Sigma_D)I ,
\]
is exactly a spectral-connectivity condition on the empirical comparison graph:
\[
\lambda_2(\widehat L)\ \ge\ \lambda^H_{\min}(\Sigma_D)/\beta^2,
\]
where $\lambda_2(\widehat L)$ is the second-smallest eigenvalue of $\widehat L$.
Intuitively, this means the sampled comparisons are sufficiently informative in all identifiable
directions (no near-disconnected bottleneck), so estimation on $H$ is well-conditioned.
\end{remark}

To prove the upper bound of optimality gap given the DPO minimizer, we need to bound the curvature of empirical DPO loss $L_n$. The curvature of $L_n$ is captured by the Hessian matrix below. 

\begin{lemma}[Hessian decomposition]\label{lem:hess-decomp}
For any $\theta\in\Theta$,
\begin{equation}\label{eq:hess-decomp}
\nabla^2 L_n(\theta)
=
A_n(\theta)+R_n(\theta),
\end{equation}
where
\[
A_n(\theta)\doteq \frac1n\sum_{i=1}^n \sigma'\!\big(u_\theta(e_i)\big)\,g(e_i;\theta)g(e_i;\theta)^\top,
\qquad
R_n(\theta)\doteq \frac1n\sum_{i=1}^n \big(\sigma(u_\theta(e_i))-a_i\big)\,\nabla_\theta^2u_\theta(e_i).
\]
\end{lemma}

\begin{proof}[Proof of \cref{lem:hess-decomp}]
Differentiate $\nabla L_n(\theta)=\frac1n\sum_i \ell'(a_i,u_i(\theta))\,g_i(\theta)$.
Using chain rule:
\[
\nabla^2 L_n(\theta)
=
\frac1n\sum_i
\ell''(a_i,u_i(\theta))\,g_i(\theta)g_i(\theta)^\top
+
\frac1n\sum_i
\ell'(a_i,u_i(\theta))\,\nabla_\theta^2 u_i(\theta).
\]
Substitute $\ell''=\sigma'$ and $\ell'=\sigma-a$.
\end{proof}

We note that the remainder term $R_n(\theta)$ does not exist when the policy is log-linear policy, and it requires more detailed analysis to bound this remainder term. 

\subsubsection{Log-linear policy}
We first consider the log-linear policy class $f_\theta(\phi(x,y))=\theta^\top\phi(x,y)$, which is easier to analyze. Note that the tabular setting is a special case of the log-linear policy when the embedding $\phi$ is a one-hot vector. The log-linear family is more tractable for three reasons:
\begin{enumerate}[label=(\roman*)]
    \item The function $g(e_i;\theta)$ is independent of $\theta$:
    \[
        g(e_i;\theta)=\beta\big(\phi(x_i,y_i^+)-\phi(x_i,y_i^-)\big).
    \]
    \item The Hessian of the DPO loss $L_n(\theta)$ has no remainder term:
    \[
        \nabla^2 L_n(\theta)=\frac{1}{n}\sum_{i=1}^n \sigma'\!\big(u_\theta(e_i)\big)\,g(e_i;\theta)g(e_i;\theta)^\top.
    \]
    \item The DPO loss $L_n(\theta)$ is strongly convex on $\Theta$.
\end{enumerate}
These three facts yield a direct lower bound on $\nabla^2 L_n(\theta)$, as we show next.

\begin{lemma}\label{lem:seg-curv-log-linear}
Suppose \cref{ass:identifiable-space,ass:model-bounded} holds. Then for any $\theta\in\Theta$,
\[
\nabla^2L_n(\theta)
\succeq
\alpha\,\widehat\Sigma_n,
\]
for some $\alpha>0$.
\end{lemma}

\begin{proof}[Proof of \cref{lem:seg-curv-log-linear}]
    By \cref{lem:global-kappa0}, $\sigma'(u_\theta(e))\geq\kappa_0>0$ for all $\theta\in\Theta$ and for all $e$. Hence, 
    \[
    \nabla^2 L_n(\theta)=\frac{1}{n}\sum_{i=1}^n \sigma'\!\big(u_\theta(e_i)\big)\,g(e_i;\theta)g(e_i;\theta)^\top\succeq\kappa_0\frac{1}{n}\sum_{i=1}^n g(e_i;\theta)g(e_i;\theta)^\top=\kappa_0\widehat\Sigma_n(\theta).
    \]
\end{proof}

Using this lower bound, we prove the following upper bound for the DPO minimizer. 

\begin{proposition}[Upper bounds for log-linear policy]\label{prop:upper-bound-loglinear}
Suppose $f_\theta(\phi)=\theta^\top\phi$ and
\cref{ass:identifiable-space,ass:model-bounded,ass:feature-sep,ass:design-coverage-a} hold.
Fix a sampling design $D$, and let $n$ i.i.d.\ comparisons
be drawn from $D$.
Let $\hat\theta_n\in\arg\min_{\theta\in\Theta}L_n(\theta)$
be the DPO estimator defined in \eqref{eq:dpo-LN},
and let $\hat\pi_n\doteq\pi_{\hat\theta_n}$.
Let $\alpha>0$ be the constant in
\cref{lem:seg-curv-log-linear}. For any $\delta\in(0,1)$ and $n\ge n_\Sigma(\delta/2)$,
with probability at least $1-\delta$,
\begin{equation}\label{eq:main-ub}
J(\pi^\star)-J(\hat\pi_n)
\le
\frac{C_{\mathrm{ub}}(\delta)}{n}
\tr\!\big(I(\theta^\star)\Sigma_D^\dagger\big),
\end{equation}
where
\[
C_{\mathrm{ub}}(\delta)
\doteq
\frac{\beta M_I\dim(H)}{2\alpha^2\delta}.
\]
\end{proposition}

\begin{proof}[Proof of \cref{prop:upper-bound-loglinear}]
Let $\Delta=\hat\theta_n-\theta^\star$.
By \cref{lem:gap-kl-contextual} and the upper side of \cref{prop:KL-sandwich-global},
\begin{equation}\label{eq:gap-upper-quad-I}
J(\pi_{\theta^\star})-J(\pi_{\hat\theta_n})
\le \frac{\beta M_I}{2}\,\Delta^\top I(\theta^\star)\Delta.
\end{equation}

\noindent\textbf{(i)} We apply the mean-value theorem and first order optimality condition. Define
\[
H_n \doteq \int_0^1 \nabla^2 L_n(\theta^\star+t\Delta)\,dt.
\]
Then
\[
\nabla L_n(\hat\theta)-\nabla L_n(\theta^\star)=H_n\Delta.
\]
Since $\hat\theta$ minimizes $L_n$ over the convex set $\Theta$,
the constrained first-order optimality condition gives
\[
\nabla L_n(\hat\theta)^\top(\theta^\star-\hat\theta)\ge0.
\]
Consequently,
\[
\Delta^\top H_n\Delta
=
\Delta^\top\nabla L_n(\hat\theta)
-\Delta^\top\nabla L_n(\theta^\star)
\le
-\Delta^\top\nabla L_n(\theta^\star).
\]

\noindent\textbf{(ii)}
By \cref{lem:seg-curv-log-linear},
\[
H_n\succeq\alpha\,\widehat\Sigma_n.
\]
Let
\[
\mathcal E_\Sigma
=
\left\{\widehat\Sigma_n\succeq\tfrac12\Sigma_D\right\}.
\]
On $\mathcal E_\Sigma$, $\widehat\Sigma_n$ is positive definite
on $H$, and
\[
\widehat\Sigma_n^\dagger\preceq2\Sigma_D^\dagger.
\]
Combining step (i) with Cauchy--Schwarz gives
\[
\begin{aligned}
\alpha\Delta^\top\widehat\Sigma_n\Delta
&\le -\Delta^\top\nabla L_n(\theta^\star)\\
&\le
\sqrt{\Delta^\top\widehat\Sigma_n\Delta}\,
\sqrt{
\nabla L_n(\theta^\star)^\top
\widehat\Sigma_n^\dagger
\nabla L_n(\theta^\star)}.
\end{aligned}
\]
Thus,
\[
\Delta^\top\widehat\Sigma_n\Delta
\le
\frac1{\alpha^2}
\nabla L_n(\theta^\star)^\top
\widehat\Sigma_n^\dagger
\nabla L_n(\theta^\star).
\]
Moreover, using $\lambda_{\max}(A)\le\tr(A)$ for $A\succeq0$,
\[
\Delta^\top I(\theta^\star)\Delta
\le
\tr\!\big(I(\theta^\star)\widehat\Sigma_n^\dagger\big)
\Delta^\top\widehat\Sigma_n\Delta.
\]
Therefore, on $\mathcal E_\Sigma$,
\[
\Delta^\top I(\theta^\star)\Delta
\le
\frac{\tr\!\big(I(\theta^\star)\widehat\Sigma_n^\dagger\big)}
{\alpha^2}
\nabla L_n(\theta^\star)^\top
\widehat\Sigma_n^\dagger
\nabla L_n(\theta^\star).
\]

\noindent\textbf{(iii)}
The score at $\theta^\star$ is
\[
\nabla L_n(\theta^\star)
=
-\frac1n\sum_{i=1}^n
\big(a_i-\sigma(u_i(\theta^\star))\big)g_i(\theta^\star).
\]
Conditional on the sampled edges, the summands are independent
and mean-zero, with
\[
\Var(a_i\mid e_i)=\sigma'(u_i(\theta^\star))\le\tfrac14.
\]
Hence,
\[
\mathbb E\!\left[
\nabla L_n(\theta^\star)\nabla L_n(\theta^\star)^\top
\,\middle|\,e_1,\ldots,e_n
\right]
\preceq\frac1{4n}\widehat\Sigma_n.
\]
Since $\widehat\Sigma_n$ is determined by the sampled edges,
for every edge realization in $\mathcal E_\Sigma$,
\[
\begin{aligned}
&\mathbb E\!\left[
\nabla L_n(\theta^\star)^\top
\widehat\Sigma_n^\dagger
\nabla L_n(\theta^\star)
\,\middle|\,e_1,\ldots,e_n
\right]\\
&\qquad\le
\frac1{4n}
\tr\!\big(\widehat\Sigma_n^\dagger\widehat\Sigma_n\big)
=
\frac{\dim(H)}{4n}.
\end{aligned}
\]

\noindent\textbf{(iv)}
Combining steps (ii)--(iii) with
\eqref{eq:gap-upper-quad-I}, for every edge realization
in $\mathcal E_\Sigma$,
\[
\begin{aligned}
&\mathbb E\!\left[
J(\pi_{\theta^\star})-J(\pi_{\hat\theta})
\,\middle|\,e_1,\ldots,e_n
\right]\\
&\qquad\le
\frac{\beta M_I\dim(H)}{8\alpha^2n}
\tr\!\big(I(\theta^\star)\widehat\Sigma_n^\dagger\big)\\
&\qquad\le
\frac{\beta M_I\dim(H)}{4\alpha^2n}
\tr\!\big(I(\theta^\star)\Sigma_D^\dagger\big).
\end{aligned}
\]
Because $\mathcal E_\Sigma$ depends only on the sampled edges,
the tower property implies
\[
\mathbb E\!\left[
J(\pi_{\theta^\star})-J(\pi_{\hat\theta})
\,\middle|\,\mathcal E_\Sigma
\right]
\le
\frac{\beta M_I\dim(H)}{4\alpha^2n}
\tr\!\big(I(\theta^\star)\Sigma_D^\dagger\big).
\]

Fix $\delta\in(0,1)$ and let
\[
C_{\mathrm{ub}}(\delta)
\doteq
\frac{\beta M_I\dim(H)}{2\alpha^2\delta}.
\]
Conditional Markov's inequality gives
\[
\Pr\!\left(
J(\pi_{\theta^\star})-J(\pi_{\hat\theta})
>
\frac{C_{\mathrm{ub}}(\delta)}{n}
\tr\!\big(I(\theta^\star)\Sigma_D^\dagger\big)
\,\middle|\,\mathcal E_\Sigma
\right)
\le\frac{\delta}{2}.
\]
Moreover, by \cref{prop:Bernstein-H},
$\Pr(\mathcal E_\Sigma^c)\le\delta/2$ whenever
\[
n\ge n_\Sigma(\delta/2)
=
8\,\frac{G^2}{\lambda_{\min}^H(\Sigma_D)}
\log\!\frac{2\dim(H)}{\delta}.
\]
Combining the two failure probabilities shows that, with
probability at least $1-\delta$,
\begin{equation}\label{eq:upper-final-global-I}
J(\pi_{\theta^\star})-J(\pi_{\hat\theta})
\le
\frac{C_{\mathrm{ub}}(\delta)}{n}
\tr\!\big(I(\theta^\star)\Sigma_D^\dagger\big).
\end{equation}
\end{proof}

\subsubsection{General policy}
When the policy is not log-linear, the analysis becomes more involved because the three simplifying features of the log-linear case no longer hold. In particular, when $f_\theta(\phi(x,y))$ is not linear in $\phi(x,y)$,
\begin{enumerate}[label=(\roman*)]
    \item the score function $g(e_i;\theta)$ depends on $\theta$:
    \[
        g(e_i;\theta)=\beta\big(\nabla_\theta f_\theta(\phi(x_i,y_i^+))-\nabla_\theta f_\theta(\phi(x_i,y_i^-))\big);
    \]
    \item the Hessian of the DPO loss $L_n(\theta)$ includes an additional remainder term:
    \[
        \nabla^2 L_n(\theta)
        =\frac{1}{n}\sum_{i=1}^n \sigma'\!\big(u_\theta(e_i)\big)\,g(e_i;\theta)g(e_i;\theta)^\top
        + \frac1n\sum_{i=1}^n \big(\sigma(u_\theta(e_i))-a_i\big)\,\nabla_\theta^2u_\theta(e_i);
    \]
    \item the DPO loss $L_n(\theta)$ is non-convex in general.
\end{enumerate}
As a result, we cannot expect the lower bound in \cref{lem:seg-curv-log-linear} applies universally in $\Theta$, and a more careful analysis is required. Our analysis resolves these issues in three layers.

\paragraph{(I) local geometry of population loss}
We first prove that with general policy, the desired lower bound in \cref{lem:seg-curv-log-linear} holds locally for the population DPO risk $L(\theta)$. Given the sampling policy $D$ and a policy $\pi_\theta$, the population DPO risk $L(\theta)$ is defind as 
\[
L(\theta)
\;\doteq\;
\E_{(e,a)}\!\left[\ell\!\left(a,u_\theta(e)\right)\right].
\]

\begin{proposition}[Local identifiability on \(H\)]\label{prop:local-ident-H}
Suppose \cref{ass:identifiable-space,ass:model-bounded,ass:design-coverage-a} hold.
Then there exist constants \(r_{\mathrm{loc}}>0\) and \(c_{\mathrm{loc}}>0\),
depending only on the fixed constants in the standing assumptions, such that
for all \(\theta\in\Theta\) with \(\|\theta-\theta^\star\|_2\le r_{\mathrm{loc}}\),
\[
\nabla^2L(\theta)\succeq c_{\mathrm{loc}}\,\Sigma_D(\theta^\star)\succeq c_{\mathrm{loc}}\mu_\star I.
\]
\end{proposition}

\begin{proof}[Proof of \cref{prop:local-ident-H}]
First, we prove a lower bound for $A(\theta)$. Using \(\nabla^2L(\theta)=A(\theta)+R(\theta)\) and global
\(\sigma'(u_\theta(e))\ge \kappa_0\) by \cref{lem:global-kappa0},
\[
 A(\theta) 
=
 \,\E\!\left[\sigma'(u_\theta)\,g(\theta)g(\theta)^\top\right] 
\succeq
\kappa_0\, \,\E[g(\theta)g(\theta)^\top]\, .
\]
By the Gram perturbation bound with Lipschitz constant \(H_u\) and \(G_u\) (\cref{lem:gram-perturb,lem:u-properties}),
\[
\left\|\E[g(\theta)g(\theta)^\top]-\Sigma_D(\theta^\star)\right\|_{\op}
\le
2G_u\,H_u\|\theta-\theta^\star\|_2+H_u^2\|\theta-\theta^\star\|_2^2.
\]
Together with \( \Sigma_D(\theta^\star)  \succeq \mu_\star  \) by \cref{ass:design-coverage-a},
\[
 \,\E[g(\theta)g(\theta)^\top]\, 
\succeq
\bigl(1-\eta(\theta)\bigr)\, \Sigma_D(\theta^\star),
\]
where
\[
\eta(\theta)\doteq
\frac{2G_u\,H_u\|\theta-\theta^\star\|_2+H_u^2\|\theta-\theta^\star\|_2^2}{\mu_\star}.
\]
Hence
\[
 A(\theta) 
\succeq
\kappa_0(1-\eta(\theta))\, \Sigma_D(\theta^\star)  .
\]

Second, we provide a bound for $R(\theta)$. Let $\delta(\theta)$ denote the norm $\|R(\theta)\|_\op$. Then,
\[
 R(\theta)  \succeq -\delta(\theta)I .
\]
Since \(\Sigma_D(\theta^\star)\succeq \mu_\star I\), $I\preceq \mu_\star^{-1}\Sigma_D(\theta^\star),$ thus
\[
 R(\theta) 
\succeq
-\frac{\delta(\theta)}{\mu_\star}\, \Sigma_D(\theta^\star) .
\]

Using the bound of $A(\theta)$ and $R(\theta)$,
\[
 \nabla^2L(\theta) 
=
 A(\theta) + R(\theta) 
\succeq
\left[\kappa_0(1-\eta(\theta))-\frac{\delta(\theta)}{\mu_\star}\right]
\Sigma_D(\theta^\star)
=
\alpha(\theta)\, \Sigma_D(\theta^\star),
\]
where 
\[
\alpha(\theta)\doteq
\kappa_0\!\left(
1-\eta(\theta)
\right)
-\frac{\delta(\theta)}{\mu_\star}.
\]

Third, we provide a bound for $\delta(\theta)$. Using \(|\sigma(a)-\sigma(b)|\le \frac14|a-b|\),
\[
\delta(\theta)
\le
\E\!\left[
|\sigma(u_\theta)-\sigma(u_{\theta^\star})|\,\|\nabla_\theta^2u_\theta\|_{\op}
\right]
\le
\frac14\,H_u\,\E|u_\theta-u_{\theta^\star}|.
\]
By Lipschitzness of \(u_\theta\) (\cref{lem:u-properties}),
\[
|u_\theta(e)-u_{\theta^\star}(e)|\le G_u\|\theta-\theta^\star\|_2,
\]
hence
\[
\delta(\theta)\le \frac14\,H_uG_u\|\theta-\theta^\star\|_2.
\]
Substituting gives the following lower bound on \(\alpha(\theta)\):
\[
\alpha(\theta)\ge
\kappa_0\!\left(
1-\eta(\theta)
\right)
-\frac{H_uG_u}{4\mu_\star}\|\theta-\theta^\star\|_2.
\]

Finally, define
\[
r_{\loc}:=\min\{\rho_g,\rho_R\},\qquad
\rho_g:=\frac{\sqrt{G_u^2+\mu_\star/2}-G_u}{H_u},\qquad
\rho_R:=\frac{\kappa_0\mu_\star}{2H_uG_u}.
\]
If \(\|\theta-\theta^\star\|_2\le r_{\loc}\), then by the definitions of
\(\rho_g,\rho_R\):
\[
\eta(\theta)\le \frac12,\qquad
\frac{\delta(\theta)}{\mu_\star}
\le
\frac{H_uG_u}{4\mu_\star}\,\|\theta-\theta^\star\|_2
\le
\frac{\kappa_0}{8}.
\]
Therefore
\[
\alpha(\theta)\ge \kappa_0(1-\tfrac12)-\tfrac{\kappa_0}{8}
=\frac{3\kappa_0}{8}\ge \frac{\kappa_0}{4}.
\]
So
\[
 \nabla^2L(\theta) 
\succeq
\frac{\kappa_0}{4} \Sigma_D(\theta^\star)  
\succeq
\frac{\kappa_0\mu_\star}{4}I .
\]
\end{proof}

Therefore, for $\theta$ close enough to $\theta^\star$ (i.e., $\|\theta-\theta^*\|\le r_{\loc}$), the desired lower bound holds.

\paragraph{(II) high-probability local geometry of empirical loss}
Next, we translate the local bound from the population DPO risk to the empirical DPO risk. On the same ball \(B_{r_{\loc}}(\theta^\star)\), we establish a high-probability uniform Hessian concentration bound
\[
\sup_{\theta\in B_{r_{\loc}}}\!
\|\nabla^2L_n(\theta)-\nabla^2L(\theta)\|_{\op}
\le \varepsilon_n^{\mathrm{Hess}},
\]
where $\varepsilon_n^{\mathrm{Hess}}$ is decreasing in $n$. 

\begin{lemma}[Uniform Hessian concentration on \(\mathcal B_{\loc}\)]
\label{lem:hessian-concentration}
Suppose \cref{ass:identifiable-space,ass:model-bounded} hold, and define
\[
\mathcal B_{\loc}
:= \{\theta\in\Theta:\ \|\theta-\theta^\star\|_2\le r_{\loc}\}.
\]
Then for any \(\delta\in(0,1)\), with probability at least \(1-\delta\),
\[
\sup_{\theta\in\mathcal B_{\loc}}
\bigl\|\nabla^2L_n(\theta)-\nabla^2L(\theta)\bigr\|_{\op}
\;\le\;
\varepsilon_n^{\mathrm{Hess}}(\delta),
\]
where \(\varepsilon_n^{\mathrm{Hess}}(\delta)\) depends only on \(n,\delta\), the radius
\(r_{\loc}\), the identifiable dimension \(d_H\), and the fixed constants in
\cref{ass:model-bounded}, and satisfies
\[
\varepsilon_n^{\mathrm{Hess}}(\delta)
=
\mathcal{O}\!\left(
\sqrt{\frac{d_H\log n+\log(1/\delta)}{n}}
+
\frac{d_H\log n+\log(1/\delta)}{n}
\right).
\]
\end{lemma}

\begin{proof}[Proof of \cref{lem:hessian-concentration}]
First, we create a covering net on \(\mathcal B_{\loc}\). Because \(\Theta\subset\theta^\star+H\) and \(\dim(H)=d_H\), the local ball
\[
\mathcal B_{\loc}
=\{\theta\in\Theta:\|\theta-\theta^\star\|_2\le r_{\loc}\}
\]
is a \(d_H\)-dimensional Euclidean ball in the affine subspace \(\theta^\star+H\).
Fix the net radius $\rho_n \;\doteq\; \frac{1}{n}$, and let \(\mathcal N_{\rho_n}\) be a \(\rho_n\)-net of \(\mathcal B_{\loc}\) in \(\|\cdot\|_2\).
Then the standard volumetric bound gives
\[
\log|\mathcal N_{\rho_n}|
\;\le\;
d_H\log\!\left(\frac{3r_{\loc}}{\rho_n}\right)
=
d_H\log\!\bigl(3r_{\loc}n\bigr).
\]

Second, we apply matrix concentration on net points. For fixed \(\theta\in\mathcal N_{\rho_n}\), define
\[
Y_i(\theta):=\nabla^2\ell_\theta(e_i)-\mathbb E_{e\sim D}[\nabla^2\ell_\theta(e)].
\]
By \cref{lem:loss-hessian-properties}, uniformly over \(\theta\in\Theta\),
\[
\|Y_i(\theta)\|_{\op}\le 2M_H,
\qquad
\Bigl\|\E\bigl[Y_i(\theta)^2\bigr]\Bigr\|_{\op}\le V_H,
\]
and each \(Y_i(\theta)\) lives in the \(d_H\)-dimensional subspace \(H\).
Applying a matrix Bernstein inequality \citep[Theorem 1.6]{tropp2012userfriendly} on \(H\) at each net point and taking a union
bound over \(\mathcal N_{\rho_n}\), we obtain that with probability at least \(1-\delta\),
\[
\max_{\theta\in\mathcal N_{\rho_n}}
\bigl\|\nabla^2L_n(\theta)-\nabla^2L(\theta)\bigr\|_{\op}
\;\le\;
C_1\sqrt{\frac{V_H\,\Lambda_H(\delta)}{n}}
+
C_2\frac{M_H\,\Lambda_H(\delta)}{n},
\]
where
\[
\Lambda_H(\delta)
=\log|\mathcal N_{\rho_n}|+\log\!\frac{2d_H}{\delta}\;\le\;d_H\log\!\bigl(3r_{\loc}n\bigr)+\log\!\frac{2d_H}{\delta}.
\]
Finally, we extend the bound to off-net points in the ball via Hessian Lipschitzness. Take any \(\theta\in\mathcal B_{\loc}\), and choose
\(\theta^\sharp\in\mathcal N_{\rho_n}\) with
\(\|\theta-\theta^\sharp\|_2\le\rho_n\).
By global Hessian Lipschitzness (which in particular holds on \(\mathcal B_{\loc}\)),
\[
\bigl\|\nabla^2L_n(\theta)-\nabla^2L_n(\theta^\sharp)\bigr\|_{\op}
\le
L_H\|\theta-\theta^\sharp\|_2
\le
L_H\rho_n
=
\frac{L_H}{n},
\]
and similarly
\[
\bigl\|\nabla^2L(\theta)-\nabla^2L(\theta^\sharp)\bigr\|_{\op}
\le
\frac{L_H}{n}.
\]
Therefore,
\[
\bigl\|\nabla^2L_n(\theta)-\nabla^2L(\theta)\bigr\|_{\op}
\le
\bigl\|\nabla^2L_n(\theta^\sharp)-\nabla^2L(\theta^\sharp)\bigr\|_{\op}
+\frac{2L_H}{n}.
\]

Taking the supremum over \(\theta\in\mathcal B_{\loc}\) and combining with the net
bound from above yields, on the same event of probability at least \(1-\delta\),
\[
\sup_{\theta\in\mathcal B_{\loc}}
\bigl\|\nabla^2L_n(\theta)-\nabla^2L(\theta)\bigr\|_{\op}
\;\le\;
C_1\sqrt{\frac{V_H\,\Lambda_H(\delta)}{n}}
+
C_2\frac{M_H\,\Lambda_H(\delta)}{n}
+
\frac{2L_H}{n}.
\]
This is exactly the claimed uniform Hessian concentration on \(\mathcal B_{\loc}\),
with an explicit choice \(\rho_n=1/n\) and
\(\Lambda_H(\delta)\lesssim d_H\log n+\log(1/\delta)\).
\end{proof}

Therefore, for \(n\) large enough so that \(\varepsilon_n^{\mathrm{Hess}}\le \tfrac18 \kappa_0\mu_\star\), this yields a uniform empirical curvature lower bound on that ball.

\begin{proposition}[Uniform empirical Hessian lower bound on the local ball]
\label{prop:uniform-emp-hess-lb-fixed}
Suppose \cref{ass:identifiable-space,ass:model-bounded,ass:design-coverage-a} hold.
Then there exist constants \(c_{\mathrm{emp}}, C_{\mathrm{Hess}}>0\), depending only on the fixed constants in the standing assumptions, such that for any \(\delta\in(0,1)\), if
\[
n \;\ge\; n_\Hess(\delta)\doteq
C_{\mathrm{Hess}}\,
\frac{
d_H\log\!\bigl(\tfrac{1}{\mu_\star}\bigr)+\log(1/\delta)
}{
\mu_\star^2
},
\]
then, with probability at least \(1-\delta\), for all \(\theta\in\mathcal B_{\loc}\),
\[
\nabla^2L_n(\theta)\succeq c_{\mathrm{emp}}\,\Sigma_D(\theta^\star)\succeq c_{\mathrm{emp}}\mu_\star I.
\]
In particular, for any \(\Delta\) such that
\(\theta^\star+t\Delta\in\mathcal B_{\loc}\) for all \(t\in[0,1]\),
\[
\bar H_n(\Delta)\succeq c_{\mathrm{emp}}\,\Sigma_D(\theta^\star),
\qquad
\bar H_n(\Delta):=\int_0^1 \nabla^2L_n(\theta^\star+t\Delta)\,dt.
\]
\end{proposition}

\begin{proof}[Proof of \cref{prop:uniform-emp-hess-lb-fixed}]
Fix $\theta\in\mathcal B_{\loc}$. Decompose
\[
\nabla^2L_n(\theta)
=
\nabla^2L(\theta)
+
\big(\nabla^2L_n(\theta)-\nabla^2L(\theta)\big).
\]
On the Hessian concentration event,
\[
\nabla^2L_n(\theta)-\nabla^2L(\theta)
\succeq
-\varepsilon_n^{\rm Hess}(\delta)\,I.
\]
By \cref{prop:local-ident-H}, for all $\theta\in\mathcal B_{\loc}$,
\[
\nabla^2L(\theta)
\succeq
\frac{\kappa_0}{4}\,\Sigma_D(\theta^\star),
\]
where $\kappa_0$ is from \cref{lem:global-kappa0}. Hence
\[
\nabla^2L_n(\theta)
\succeq
\frac{\kappa_0}{4}\,\Sigma_D(\theta^\star)
-\varepsilon_n^{\rm Hess}(\delta)\,I.
\]

By \cref{ass:design-coverage-a} we have
\[
I \preceq \frac{1}{\mu_\star}\,\Sigma_D(\theta^\star).
\]
Therefore
\[
-\varepsilon_n^{\rm Hess}(\delta)\,I
\succeq
-\frac{\varepsilon_n^{\rm Hess}(\delta)}{\mu_\star}\,\Sigma_D(\theta^\star),
\]
and thus
\[
\nabla^2L_n(\theta)
\succeq
\left(\frac{\kappa_0}{4}-\frac{\varepsilon_n^{\rm Hess}(\delta)}{\mu_\star}\right)\Sigma_D(\theta^\star).
\]
If $\varepsilon_n^{\rm Hess}(\delta)\le \kappa_0\mu_\star/8$, then
\[
\nabla^2L_n(\theta)
\succeq
\frac{\kappa_0}{8}\,\Sigma_D(\theta^\star)
\succeq
\frac{\kappa_0\mu_\star}{8}\,I.
\]
This lower bound holds simultaneously for all $\theta\in\mathcal B_{\loc}$ on the same event.

For the segment claim, if $\theta^\star+t\Delta\in\mathcal B_{\loc}$ for all $t\in[0,1]$, then pointwise in $t$,
\[
\nabla^2L_n(\theta^\star+t\Delta)
\succeq
\frac{\kappa_0}{8}\,\Sigma_D(\theta^\star).
\]
Integrating over $t\in[0,1]$ gives
\[
\bar H_n(\Delta)
:=
\int_0^1 \nabla^2L_n(\theta^\star+t\Delta)\,dt
\succeq
\frac{\kappa_0}{8}\,\Sigma_D(\theta^\star).
\]

By \cref{lem:hessian-concentration}, there exists a constant \(C_H>0\),
depending only on the fixed constants in the standing assumptions, such that
\[
\varepsilon_n^{\mathrm{Hess}}(\delta)
\le
C_H\left(
\sqrt{\frac{d_H\log n+\log(1/\delta)}{n}}
+
\frac{d_H\log n+\log(1/\delta)}{n}
\right).
\]
Thus, to guarantee \(\varepsilon_n^{\mathrm{Hess}}(\delta)\le c_{\mathrm{emp}}\mu_\star\),
it suffices to impose
\[
d_H\log n+\log(1/\delta)\lesssim \mu_\star^2 n.
\]
A standard inversion argument yields that there exists a constant
\(C_{\mathrm{Hess}}>0\), depending only on the fixed constants in the standing assumptions, such that
\[
n \;\ge\;
C_{\mathrm{Hess}}\,
\frac{
d_H\log\!\bigl(\frac{1}{\mu_\star}\bigr)+\log(1/\delta)
}{
\mu_\star^2
}
\]
implies
\[
d_H\log n+\log(1/\delta)\le c\,\mu_\star^2 n
\]
for a sufficiently small absolute constant \(c>0\). Therefore
\[
\varepsilon_n^{\mathrm{Hess}}(\delta)\le c_{\mathrm{emp}}\mu_\star.
\]
\end{proof}

Hence, if $\hat\theta_n$ exists in the local ball \(B_{r_{\loc}}(\theta^\star)\), we can prove the desired lower bound for any $\theta$ on the line segment between $\theta^\star$ and $\hat\theta_n$.

\begin{corollary}[Empirical segment curvature along the estimator path]
\label{cor:emp-segment-curvature-estimator}
Under the condition of \cref{lem:hessian-concentration}, if additionally
\[
\Pr\!\left(\|\hat\theta_n-\theta^\star\|_2\le r_{\loc}\right)\ge 1-\delta_{\loc},
\]
then with probability at least \(1-\delta-\delta_{\loc}\),
\[
\bar H_n(\hat\theta_n-\theta^\star)
\succeq
c_{\mathrm{emp}}\Sigma_D(\theta^\star) 
\succeq
c_{\mathrm{emp}}\mu_\star I.
\]
\end{corollary}

\begin{proof}[Proof of \cref{cor:emp-segment-curvature-estimator}]
Define
\[
\mathcal E_{\rm hess}
:=
\left\{
\sup_{\theta\in\mathcal B_{\text{loc}}}
\|\nabla^2L_n(\theta)-\nabla^2L(\theta)\|_{\op}
\le \varepsilon_n^{\rm Hess}(\delta)
\right\},
\qquad
\mathcal E_{\text{loc}}:=\{\|\hat\theta_n-\theta^\star\|_2\le r_{\text{loc}}\}.
\]
By \cref{lem:hessian-concentration}, \(\Pr(\mathcal E_{\rm hess})\ge 1-\delta\).
On \(\mathcal E_{\text{loc}}\), the entire segment
\[
\{\theta^\star+t\Delta_n:\ t\in[0,1]\}\subseteq \mathcal B_{\text{loc}}.
\]
Hence on \(\mathcal E_{\rm hess}\cap\mathcal E_{\text{loc}}\), for all \(t\in[0,1]\),
\[
\nabla^2L_n(\theta^\star+t\Delta_n)
\succeq
\frac{\kappa_0}{8}\,\Sigma_D(\theta^\star),
\]
where $\kappa_0$ is defined in \cref{lem:global-kappa0}. Integrating in \(t\) gives
\[
\bar H_n(\Delta_n)
\succeq
\frac{\kappa_0}{8}\,\Sigma_D(\theta^\star).
\]
Finally, union bound implies
\[
\Pr(\mathcal E_{\rm hess}\cap\mathcal E_{\text{loc}})\ge 1-\delta-\delta_{\text{loc}}.
\]
\end{proof}

Next, we show that $\hat\theta_n$ converges to $\theta^\star$ with high probability, and hence, when $n$ is large enough, $\hat\theta_n$ will be in the local ball \(B_{r_{\loc}}(\theta^\star)\).

\paragraph{(III) asymptotical consistency of the DPO minimizer}
Because the curvature result is local, we next show that the DPO minimizer $\hat{\theta}_n$ enters this local ball with high probability. However, since the empirical DPO risk $L_n(\theta)$ is not convex in general, $\hat{\theta}_n$ need not converge to $\theta^\star$. For example, if both $\theta^\star$ and some $\theta'$ minimize the population DPO risk $L(\theta)$, it is not clear that $\hat{\theta}_n$ will converge to $\theta^\star$. In the statistics literature, a standard way to address this issue is to assume the population risk has a unique minimizer.

We first show that $\theta^\star$ is indeed a minimizer of the population DPO risk. 

\begin{lemma}
\label{lem:theta-star-minimizer-pop-risk}
Suppose \cref{ass:identifiable-space} holds. Then, $\theta^\star$ is a minimizer of $L(\theta)$, i.e., $\theta^\star\in\arg\min_{\theta\in\Theta}L(\theta)$.
\end{lemma}

\begin{proof}[Proof of \cref{lem:theta-star-minimizer-pop-risk}]
For each fixed edge $e$, define
\[
p^\star(e):=\Pr(a=1\mid e)=\sigma\!\big(u_{\theta^\star}(e)\big),
\qquad
q_\theta(e):=\sigma\!\big(u_\theta(e)\big).
\]
Conditioning on $e$, the conditional population loss is
\[
\begin{aligned}
\varphi_e(\theta)
&:=\E\!\left[\ell\!\big(a,u_\theta(e)\big)\mid e\right]\\
&=-p^\star(e)\log q_\theta(e)-\bigl(1-p^\star(e)\bigr)\log\!\bigl(1-q_\theta(e)\bigr).
\end{aligned}
\]
At $\theta^\star$:
\[
\varphi_e(\theta^\star)
=
-p^\star(e)\log p^\star(e)-\bigl(1-p^\star(e)\bigr)\log\!\bigl(1-p^\star(e)\bigr).
\]
Therefore
\[
\begin{aligned}
\varphi_e(\theta)-\varphi_e(\theta^\star)
&=
p^\star(e)\log\frac{p^\star(e)}{q_\theta(e)}
+\bigl(1-p^\star(e)\bigr)\log\frac{1-p^\star(e)}{1-q_\theta(e)}\\
&=: \mathrm{KL}\!\left(\mathrm{Bern}(p^\star(e))\,\|\,\mathrm{Bern}(q_\theta(e))\right)
\;\ge\;0.
\end{aligned}
\]
Now take expectation over $e$:
\[
\begin{aligned}
L(\theta)-L(\theta^\star)
&=\E_e\!\left[\varphi_e(\theta)-\varphi_e(\theta^\star)\right]\\
&=\E_e\!\left[
\mathrm{KL}\!\left(\mathrm{Bern}(p^\star(e))\,\|\,\mathrm{Bern}(q_\theta(e))\right)
\right]
\ge 0.
\end{aligned}
\]
Hence $L(\theta)\ge L(\theta^\star)$ for all $\theta\in\Theta$, i.e.,
$\theta^\star$ is a minimizer of $L$.
\end{proof}

\begin{proposition}[Uniform law of large numbers]\label{prop:ulln}
Suppose \cref{ass:identifiable-space,ass:model-bounded} hold. Then
\[
\sup_{\theta\in\Theta}\bigl|L_n(\theta)-L(\theta)\bigr|\xrightarrow{p}0.
\]

Moreover, for any \(\delta\in(0,1)\), with probability at least \(1-\delta\),
\[
\sup_{\theta\in\Theta}\bigl|L_n(\theta)-L(\theta)\bigr|
\;\le\;
\varepsilon_n^{\mathrm{ULLN}}(\delta),
\]
where \(\varepsilon_n^{\mathrm{ULLN}}(\delta)\) depends only on \(n,\delta\), the diameter of
\(\Theta\), and the fixed constants in \cref{ass:model-bounded}, and satisfies
\[
\varepsilon_n^{\mathrm{ULLN}}(\delta)
=
\mathcal{O}\!\left(
\sqrt{\frac{p\log n+\log(1/\delta)}{n}}
\right).
\]
\end{proposition}

\begin{proof}[Proof of \cref{prop:ulln}]
Fix $\varepsilon>0$, and choose $\delta := \frac{\varepsilon}{4G_u},$ where $G_u$ is from \cref{lem:u-properties}. Let $\{\theta^1,\dots,\theta^N\}$ be a finite $\delta$-net of $\Theta$
in $\|\cdot\|_2$, with
\[
N = N\big(\Theta,\|\cdot\|_2,\delta\big),
\]
where $N(\Theta,\|\cdot\|_2,\delta)$ is the
$\delta$–covering number of $\Theta$ in $\|\cdot\|_2$.

For any $\theta\in\Theta$, pick $j(\theta)$ such that
$\|\theta-\theta^{j(\theta)}\|_2\le\delta$. By Lemma~\ref{lem:loss-lip-bdd}(ii),
\[
|L_n(\theta)-L_n(\theta^{j(\theta)})|
\le G_u\delta,\qquad
|L(\theta)-L(\theta^{j(\theta)})|
\le G_u\delta.
\]
Hence
\[
|L_n(\theta)-L(\theta)|
\le
|L_n(\theta^{j(\theta)})-L(\theta^{j(\theta)})|+2G_u\delta
\le
\max_{1\le j\le N}|L_n(\theta^j)-L(\theta^j)|
+\frac{\varepsilon}{2}.
\]
Therefore
\[
\Big\{\sup_{\theta\in\Theta}|L_n(\theta)-L(\theta)|>\varepsilon\Big\}
\subseteq
\Big\{\max_{1\le j\le N}|L_n(\theta^j)-L(\theta^j)|>\frac{\varepsilon}{2}\Big\}.
\]

Define $ B_\ell \doteq \log\bigl(1+e^{U_{\max}}\bigr)$, where $U_{\max}$ is from \cref{lem:u-properties}. For each fixed $j$, $\ell_{\theta^j}(z_i)\in[0,B_\ell]$ are i.i.d., so Hoeffding gives
\[
\Pr\!\left(|L_n(\theta^j)-L(\theta^j)|>\frac{\varepsilon}{2}\right)
\le
2\exp\!\left(-\frac{n\varepsilon^2}{2B_\ell^2}\right).
\]
By a union bound over $j=1,\dots,N$,
\begin{equation}\label{eq:ulln-hp-bound-net}
\Pr\!\left(\sup_{\theta\in\Theta}|L_n(\theta)-L(\theta)|>\varepsilon\right)
\le
2N\exp\!\left(-\frac{n\varepsilon^2}{2B_\ell^2}\right).
\end{equation}

It remains to bound $N$ explicitly. Since $\Theta\subset\R^p$ is compact,
its Euclidean diameter
\[
\operatorname{diam}(\Theta)
:=
\sup_{\theta,\theta'\in\Theta}\|\theta-\theta'\|_2
\]
is finite. A standard covering-number bound for subsets of $\R^p$ yields
\[
N\big(\Theta,\|\cdot\|_2,\delta\big)
\le
\left(\frac{3\,\operatorname{diam}(\Theta)}{\delta}\right)^p
=
\left(\frac{12\,\operatorname{diam}(\Theta)\,G_u}{\varepsilon}\right)^p.
\]
Plugging this into \eqref{eq:ulln-hp-bound-net} gives
\[
\Pr\!\left(\sup_{\theta\in\Theta}|L_n(\theta)-L(\theta)|>\varepsilon\right)
\le
2\left(\frac{12\,\operatorname{diam}(\Theta)\,G_u}{\varepsilon}\right)^p
\exp\!\left(-\frac{n\varepsilon^2}{2B_\ell^2}\right).
\]

For fixed $\varepsilon>0$, the prefactor is independent of $n$, while
the exponential term decays like $\exp(-c n)$, so the RHS $\to 0$ as $n\to\infty$.
Therefore
\[
\sup_{\theta\in\Theta}|L_n(\theta)-L(\theta)|\xrightarrow{p}0.
\]

To derive an explicit bound, recall  
\[
\Pr\!\left(\sup_{\theta\in\Theta}|L_n(\theta)-L(\theta)|>\varepsilon\right)
\le
2\exp\!\left(
p\log\!\Bigl(\frac{12\,\mathrm{diam}(\Theta)\,G_u}{\varepsilon}\Bigr)
-\frac{n\varepsilon^2}{2B_\ell^2}
\right).
\]
Define \(\varepsilon_n^{\mathrm{ULLN}}(\delta)\) as any value satisfying
\[
p\log\!\Bigl(\frac{12\,\mathrm{diam}(\Theta)\,G_u}{\varepsilon_n^{\mathrm{ULLN}}(\delta)}\Bigr)
+\log\!\frac{2}{\delta}
\le
\frac{n\big(\varepsilon_n^{\mathrm{ULLN}}(\delta)\big)^2}{2B_\ell^2}.
\]
Then the RHS above is at most \(\delta\), and hence with probability at least \(1-\delta\),
\[
\sup_{\theta\in\Theta}|L_n(\theta)-L(\theta)|
\le
\varepsilon_n^{\mathrm{ULLN}}(\delta).
\]

Since \(\log(c/\varepsilon)\lesssim \log n\) when \(\varepsilon\) is chosen of order \(n^{-1/2}\),
the inequality above admits a solution of the form
\[
\varepsilon_n^{\mathrm{ULLN}}(\delta)
=
\mathcal{O}\!\left(
\sqrt{\frac{p\log n+\log(1/\delta)}{n}}
\right),
\]
where the implied constant depends only on \(B_\ell\), \(G_u\), and \(\mathrm{diam}(\Theta)\).

\end{proof}

Even though $L_n$ converges to $L$, $\hat\theta_n$ may not converge to $\theta^\star$. To have the convergence of $\hat\theta_n$ to $\theta^*$, we need to assume $\theta^\star$ is the only minimizer of the population DPO risk (\cref{ass:unique-minimizer-pop-risk}).

\begin{proposition}[Global ERM localization]\label{prop:global-erm-localization}
Suppose \cref{ass:identifiable-space,ass:model-bounded,ass:unique-minimizer-pop-risk} hold, and let
\[
R_\Theta \doteq \sup_{\theta\in\Theta}\|\theta-\theta^\star\|_2 < \infty.
\]
Fix \(r_{\loc}>0\), and define the population separation outside the ball
\[
\Delta_{\sep}(r_{\loc})
:=
\inf_{\substack{\theta\in\Theta\\ \|\theta-\theta^\star\|_2\ge r_{\loc}}}
\bigl(L(\theta)-L(\theta^\star)\bigr).
\]
Then \(\Delta_{\sep}(r_{\loc})>0\). Moreover, there exists a constant \(C_{\loc}>0\),
depending only on the fixed constants in the standing assumptions, such that the following holds:
for any \(\delta\in(0,1)\), if
\begin{equation}\label{eq:n-loc-statement-revised}
n \;\ge\; n_{\loc}(\delta)
\;\doteq\;
C_{\loc}\,
\frac{
p\log\!\Bigl(\dfrac{R_\Theta}{\Delta_{\sep}(r_{\loc})}\Bigr)
+\log\!\bigl(\dfrac{1}{\delta}\bigr)
}{
\Delta_{\sep}(r_{\loc})^2
},
\end{equation}
then every empirical risk minimizer \(\hat\theta_n\in\arg\min_{\theta\in\Theta}L_n(\theta)\) satisfies
\[
\Pr\!\bigl(\|\hat\theta_n-\theta^\star\|_2\le r_{\loc}\bigr)\ge 1-\delta.
\]
\end{proposition}

\begin{proof}[Proof of \cref{prop:global-erm-localization}]
Let
\[
R_\Theta \doteq \sup_{\theta\in\Theta}\|\theta-\theta^\star\|_2 < \infty.
\]
Fix a localization radius \(r_{\loc}>0\) and define the population separation outside the ball
\[
\Delta_{\sep}(r_{\loc})
\doteq
\inf_{\substack{\theta\in\Theta\\ \|\theta-\theta^\star\|_2\ge r_{\loc}}}
\bigl(L(\theta)-L(\theta^\star)\bigr).
\]
By continuity of \(L\) and uniqueness of its minimizer \(\theta^\star\)
(\cref{ass:unique-minimizer-pop-risk}), the infimum is attained and
\(\Delta_{\sep}(r_{\loc})>0\). Set
\[
\varepsilon_{\loc}\doteq \frac{\Delta_{\sep}(r_{\loc})}{3}.
\]

By \cref{prop:ulln}, for any \(\varepsilon>0\) and any \(\delta\in(0,1)\),
\[
\Pr\!\left(\sup_{\theta\in\Theta}\bigl|L_n(\theta)-L(\theta)\bigr|>\varepsilon\right)
\le
2\left(\frac{12\,G_u R_\Theta}{\varepsilon}\right)^p
\exp\!\left(-\frac{n\varepsilon^2}{2B_\ell^2}\right).
\]
We set \(\varepsilon=\varepsilon_{\loc}\) and require the right-hand side to be at most \(\delta\).
Equivalently, it suffices that
\[
p\log\!\left(\frac{12\,G_u R_\Theta}{\varepsilon_{\loc}}\right)
+\log\!\left(\frac{2}{\delta}\right)
\le
\frac{n\varepsilon_{\loc}^2}{2B_\ell^2}.
\]
Thus, one may take
\begin{equation}\label{eq:n-loc}
n_{\loc}(\delta)
\doteq
\frac{2B_\ell^2}{\varepsilon_{\loc}^2}
\left[
p\log\!\left(\frac{12\,G_u R_\Theta}{\varepsilon_{\loc}}\right)
+\log\!\left(\frac{2}{\delta}\right)
\right].
\end{equation}
When \(n\ge n_{\loc}(\delta)\), we have with probability at least \(1-\delta\),
\begin{equation}\label{eq:event-E-ulln}
\sup_{\theta\in\Theta}\bigl|L_n(\theta)-L(\theta)\bigr|
\le \varepsilon_{\loc}.
\end{equation}
Denote this event by \(\mathcal E_{\mathrm{ULLN}}\).

On \(\mathcal E_{\mathrm{ULLN}}\), for any \(\theta\in\Theta\),
\[
L_n(\theta)-L_n(\theta^\star)
\ge
\bigl(L(\theta)-\varepsilon_{\loc}\bigr)-\bigl(L(\theta^\star)+\varepsilon_{\loc}\bigr)
=
\bigl(L(\theta)-L(\theta^\star)\bigr)-2\varepsilon_{\loc}.
\]
If \(\|\theta-\theta^\star\|_2\ge r_{\loc}\), then by definition of \(\Delta_{\sep}(r_{\loc})\),
\[
L(\theta)-L(\theta^\star)\ge \Delta_{\sep}(r_{\loc}),
\]
and hence on \(\mathcal E_{\mathrm{ULLN}}\),
\[
L_n(\theta)-L_n(\theta^\star)
\ge
\Delta_{\sep}(r_{\loc})-2\varepsilon_{\loc}
=
\Delta_{\sep}(r_{\loc})-\frac{2}{3}\Delta_{\sep}(r_{\loc})
=
\frac{1}{3}\Delta_{\sep}(r_{\loc})
>0.
\]
Therefore, on \(\mathcal E_{\mathrm{ULLN}}\), no \(\theta\) with
\(\|\theta-\theta^\star\|_2\ge r_{\loc}\) can minimize \(L_n\) globally over \(\Theta\).
Consequently, every global empirical minimizer
\(\hat\theta_n\in\arg\min_{\theta\in\Theta}L_n(\theta)\) satisfies
\[
\|\hat\theta_n-\theta^\star\|_2 < r_{\loc}
\qquad\text{on }\mathcal E_{\mathrm{ULLN}}.
\]

Since \(\Pr(\mathcal E_{\mathrm{ULLN}})\ge 1-\delta\) whenever \(n\ge n_{\loc}(\delta)\), we obtain
\[
\Pr\!\bigl(\|\hat\theta_n-\theta^\star\|_2\le r_{\loc}\bigr)\ge 1-\delta.
\]
\end{proof}

Therefore, by \cref{prop:global-erm-localization}, when $n$ is large enough, $\hat{\theta}_n$ will be in the local ball with high probability. As a result, by \cref{cor:emp-segment-curvature-estimator}, with high probability, the curvature of $L_n$ on the line segment between $\hat{\theta}_n$ and $\theta^\star$ is lower bounded by $\Sigma_D(\theta^\star)$. This gives us the upper bound for the DPO minimizer. 

\paragraph{Upper bound for general policy}
Putting everything together, we can prove the following upper bound for the DPO minimizer in the general \(f_\theta\) case.

\begin{theorem}[DPO upper bound in the general $f_\theta$ case]\label{thm:dpo-upper-general-final}
Suppose \cref{ass:identifiable-space,ass:model-bounded,ass:feature-sep,ass:design-coverage-a,ass:unique-minimizer-pop-risk} hold,
and let $\hat\theta_n \in \arg\min_{\theta\in\Theta} L_n(\theta)$ and $
\hat\pi_n := \pi_{\hat\theta_n}.$
Fix $\delta\in(0,1)$ and define
\[
n_0(\delta)
\;\doteq\;
\max\Big\{\,n_{\loc}(\delta/3),\;n_{\Hess}(\delta/3)\Big\}.
\]
Then for all $n\ge n_0(\delta)$, with probability at least $1-\delta$,
\begin{equation}\label{eq:gap-trace-ub-hp}
J(\pi^\star)-J(\hat\pi_n)
\;\le\;
\frac{C_{\ub}(\delta)}{n}\,
\tr\!\bigl(I(\theta^\star)\Sigma_D^\dagger(\theta^*))\bigr),
\end{equation}
where
\[
C_{\ub}(\delta)= C\,\frac{\dim(H)}{\delta},
\]
for some constant \(C>0\) depending only on the fixed constants in standing assumptions.
\end{theorem}

\begin{proof}[Proof of \cref{thm:dpo-upper-general-final}]
Set $\Delta := \hat\theta_n - \theta^\star$.
By the upper side of \cref{prop:KL-sandwich-nonlinear}, there exists $c_+>0$ such that
for all $\theta\in\Theta$,
\[
J(\pi^\star)-J(\pi_\theta)
\;\le\;
c_+\,(\theta-\theta^\star)^\top I(\theta^\star)(\theta-\theta^\star).
\]
In particular,
\begin{equation}\label{eq:gap-upper-quad-W-clean}
J(\pi^\star)-J(\hat\pi_n)
\;\le\;
c_+\,\Delta^\top I(\theta^\star)\Delta.
\end{equation}

First, we apply mean-value theorem and optimality condition. Define the segment Hessian
\[
H_n \doteq \int_0^1 \nabla^2 L_n(\theta^\star+t\Delta)\,dt .
\]
By the integral form of Taylor's theorem,
\[
\nabla L_n(\hat\theta_n)-\nabla L_n(\theta^\star)=H_n\Delta.
\]

Since $\hat\theta_n$ is a (global) minimizer of $L_n$ over $\Theta\subseteq \theta^\star+H$,
hence
\[
\nabla L_n(\hat\theta_n)^\top(\theta^\star-\hat{\theta}_n)\ge 0,
\qquad\text{so}\qquad
\Delta^\top H_n\Delta\le -\Delta^\top\,\nabla L_n(\theta^\star).
\]

Second, we consider the following good events to bound the average curvature of $L_n$ on the segment $[\hat{\theta}_n,\theta^\star]$. Fix $\delta\in(0,1)$ and choose $\delta_{\loc},\delta_{\Hess},\delta_G\in(0,1)$ such that
\[
\delta_{\loc}+\delta_{\Hess}+\delta_G\le \delta.
\]

Let $\mathcal E_{\loc}$ be the localization event from \cref{prop:global-erm-localization}:
\[
\mathcal E_{\loc}:=\{\|\hat\theta_n-\theta^\star\|_2\le r_{\loc}\}.
\]
By \cref{prop:global-erm-localization}, for
\[
\varepsilon_{\loc}\doteq \frac{\Delta_{\sep}(r_{\loc})}{3},
\qquad
B_\ell \doteq \log(1+e^{u_{\max}}),
\]
and
\begin{equation}\label{eq:nloc-clean}
n_{\loc}(\delta)
\doteq
\frac{2B_\ell^2}{\varepsilon_{\loc}^2}
\left[
p\log\!\left(\frac{12\,G_u R_\Theta}{\varepsilon_{\loc}}\right)
+\log\!\left(\frac{2}{\delta}\right)
\right],
\end{equation}
we have, for all $n\ge n_{\loc}(\delta_{\loc})$,
\[
\Pr(\mathcal E_{\loc})\ge 1-\delta_{\loc}.
\]

Let $\mathcal E_{\Hess}$ be the Hessian concentration event from \cref{lem:hessian-concentration}:
\[
\mathcal E_{\Hess}
:=
\left\{
\sup_{\theta\in\mathcal B_{\loc}}
\|\nabla^2L_n(\theta)-\nabla^2L(\theta)\|_{\op}
\le \varepsilon_n^{\rm Hess}(\delta_{\Hess})
\right\},
\qquad
\mathcal B_{\loc}:=\{\theta\in\Theta:\|\theta-\theta^\star\|_2\le r_{\loc}\}.
\]
By \cref{lem:hessian-concentration}, for all $n\ge n_{\Hess}(\delta_{\Hess})$,
\[
\Pr(\mathcal E_{\Hess})\ge 1-\delta_{\Hess},
\]
where $n_{\Hess}(\delta_{\Hess})$ is chosen exactly as in \cref{lem:hessian-concentration}
so that on $\mathcal E_{\loc}\cap\mathcal E_{\Hess}$ we can invoke
\cref{cor:emp-segment-curvature-estimator} and obtain the segment curvature bound
\begin{equation}\label{eq:Hn-lower-clean}
H_n=\int_0^1 \nabla^2L_n(\theta^\star+t\Delta)\,dt
\succeq \alpha\,\Sigma_D(\theta^*)
\qquad\text{on }H,
\end{equation}
with $\alpha=\frac{\kappa_0}{8}$. Since $\Sigma_D(\theta^\star)$ is invertible on $H$ by \cref{ass:design-coverage-a}, pseudoinverse monotonicity on $H$ gives
\begin{equation}\label{eq:Hdagger-upper-clean}
H_n^\dagger \preceq \alpha^{-1}\Sigma_D(\theta^\star)^\dagger.
\end{equation}

Third, we bound the weighted error  \(\Delta^\top\Sigma_D(\theta^\star)\Delta\). On $\mathcal E_{\loc}\cap\mathcal E_{\Hess}$, using \eqref{eq:Hn-lower-clean},
\[
\alpha\,\Delta^\top\Sigma_D(\theta^\star)\Delta
\le
\Delta^\top H_n\Delta
\le
-\Delta^\top\nabla L_n(\theta^\star).
\]
Applying Cauchy--Schwarz in the \(\Sigma_D(\theta^\star)\)-geometry,
\[
|\Delta^\top\nabla L_n(\theta^\star)|
=
\big|\langle \Sigma_D(\theta^\star)^{1/2}\Delta,\;\Sigma_D(\theta^\star)^{\dagger/2}\nabla L_n(\theta^\star)\rangle\big|
\le
\|\Sigma_D(\theta^\star)^{1/2}\Delta\|_2\,
\|\Sigma_D(\theta^\star)^{\dagger/2}\nabla L_n(\theta^\star)\|_2.
\]
Hence
\[
\alpha\,\|\Sigma_D(\theta^\star)^{1/2}\Delta\|_2^2
\le
\|\Sigma_D(\theta^\star)^{1/2}\Delta\|_2\,
\|\Sigma_D(\theta^\star)^{\dagger/2}\nabla L_n(\theta^\star)\|_2.
\]
If $\|\Sigma_D(\theta^\star)^{1/2}\Delta\|_2=0$, then $\Delta^\top\Sigma_D(\theta^\star)\Delta=0$ and there is nothing to prove.
Otherwise divide both sides by $\|\Sigma_D(\theta^\star)^{1/2}\Delta\|_2$ and square:
\begin{equation}\label{eq:Delta-sigmaD-bound-clean}
\Delta^\top\Sigma_D(\theta^\star)\Delta
\le
\frac{1}{\alpha^2}\,
\|\Sigma_D(\theta^\star)^{\dagger/2}\nabla L_n(\theta^\star)\|_2^2.
\end{equation}

Next, we bound \(\nabla L_n(\theta^\star)\). By \cref{lem:score-cov},
\[
\mathbb E\!\left[
\|\Sigma_D(\theta^\star)^{\dagger/2}\nabla L_n(\theta^\star)\|_2^2
\right]
\le \frac{\dim(H)}{4n}.
\]
Let
\[
C_G \doteq \frac{\dim(H)}{4}.
\]
By Markov's inequality, for any $\delta_G\in(0,1)$, the event
\[
\mathcal E_{\rm score}
:=
\left\{
\|\Sigma_D(\theta^\star)^{\dagger/2}\nabla L_n(\theta^\star)\|_2^2
\le \frac{C_G}{n\delta_G}
\right\}
\]
satisfies
\[
\Pr(\mathcal E_{\rm score})\ge 1-\delta_G.
\]
Combining with \eqref{eq:Delta-sigmaD-bound-clean}, on
\[
\mathcal E_{\rm tot}:=\mathcal E_{\loc}\cap\mathcal E_{\Hess}\cap\mathcal E_{\rm score},
\]
we obtain
\begin{equation}\label{eq:Delta-sigmaD-highprob-clean}
\Delta^\top\Sigma_D(\theta^\star)\Delta
\le
\frac{C_G}{\alpha^2}\,\frac{1}{n\delta_G}.
\end{equation}

Next, we convert design-weighted error to RLHF-weighted error. Since $\Delta\in H$ and $I(\theta^\star)$ is supported on $H$,
\[
\Delta^\top I(\theta^\star)\Delta
\le
\lambda_{\max}\!\bigl(\Sigma_D(\theta^\star)^{\dagger/2}I(\theta^\star)\Sigma_D(\theta^\star)^{\dagger/2}\bigr)\,
\Delta^\top\Sigma_D(\theta^\star)\Delta.
\]
Using $\lambda_{\max}(A)\le \tr(A)$ for PSD $A$,
\[
\Delta^\top I(\theta^\star)\Delta
\le
\tr\!\bigl(I(\theta^\star)\Sigma_D(\theta^\star)^\dagger\bigr)\,\Delta^\top\Sigma_D(\theta^\star)\Delta.
\]
Hence, on $\mathcal E_{\rm tot}$, by \eqref{eq:Delta-sigmaD-highprob-clean},
\[
\Delta^\top I(\theta^\star)\Delta
\le
\frac{C_G}{\alpha^2}\,\frac{1}{n\delta_G}\,
\tr\!\bigl(I(\theta^\star)\Sigma_D(\theta^\star)^\dagger\bigr).
\]
Plugging into \eqref{eq:gap-upper-quad-W-clean},
\[
J(\pi^\star)-J(\hat\pi_n)
\le
\frac{c_+C_G}{\alpha^2\delta_G}\cdot\frac{1}{n}\,
\tr\!\bigl(I(\theta^\star)\Sigma_D(\theta^\star)^\dagger\bigr).
\]
Define
\[
C_{\ub}\doteq \frac{c_+C_G}{\alpha^2\delta_G}.
\]
Then on $\mathcal E_{\rm tot}$,
\begin{equation}\label{eq:upper-final-clean}
J(\pi^\star)-J(\hat\pi_n)
\le
\frac{C_{\ub}}{n}\,
\tr\!\bigl(I(\theta^\star)\Sigma_D(\theta^\star)^\dagger\bigr).
\end{equation}

Finally, we bound the probability of the good event. Set
\[
n_0(\delta)\doteq \max\{n_{\loc}(\delta_{\loc}),\,n_{\Hess}(\delta_{\Hess})\}.
\]
Then for all \(n\ge n_0(\delta)\),
\[
\Pr(\mathcal E_{\loc})\ge 1-\delta_{\loc},\qquad
\Pr(\mathcal E_{\Hess})\ge 1-\delta_{\Hess},\qquad
\Pr(\mathcal E_{\rm score})\ge 1-\delta_G.
\]
Therefore, by the union bound,
\[
\Pr(\mathcal E_{\rm tot})
\ge 1-\delta_{\loc}-\delta_{\Hess}-\delta_G
\ge 1-\delta.
\]
Since \eqref{eq:upper-final-clean} holds on \(\mathcal E_{\rm tot}\), the claimed
high-probability upper bound follows.
\end{proof}

\subsection{Step 3: lower bound via information inequality}\label{s:lower-bound}
In this section, we consider the Bayesian setting by assuming \cref{ass:prior-vantrees}. In the Bayesian setting, the critical step to prove the lower bound in \cref{thm:informal} is to prove the following lower bound for the quadratic sandwich by applying the Van Trees inequality.

\begin{lemma}\label{lem:vt-lb-trace}
    Suppose \cref{ass:identifiable-space,ass:model-bounded,ass:feature-sep,ass:design-coverage-b,ass:prior-vantrees} hold with $\mathcal{R}=\supp{(\rho)}$. The Fisher information of $\rho$ is given by 
\begin{equation}
    J(\rho)\doteq \int_\Theta \nabla\log\rho(\theta)\nabla\log\rho(\theta)^\top\,\rho(\theta)\,d\theta. 
\end{equation}
Let
\[
\bar\Sigma_\rho \doteq \E_{\theta\sim\rho}[\Sigma_D(\theta)],
\qquad
\bar I_\rho \doteq \E_{\theta\sim\rho}[I(\theta)],\qquad n_{\rm prior}
\;\doteq\;
\left\lceil
\frac{4\,\lambda_{\max}(J(\rho))}{\mu_\mathcal{R}}
\right\rceil.
\]
Then for any estimator $\tilde\theta_n$ measurable with respect to the $n$ pairwise comparisons, and for all $n\ge n_{\rm prior}$, there exists a constant $c_{\rm lb}>0$ (depending only on the constants appearing in assumptions and the prior $\rho$) such that
\begin{equation}\label{eq:vt-lb-trace-main}
\E_{\theta\sim\rho}\E_{\mathcal D_n\mid \theta}\!\Big[
(\tilde\theta_n-\theta)^\top I(\theta)(\tilde\theta_n-\theta)
\Big]
\;\ge\;
\frac{c_{\rm lb}}{n}\,
\E_{\theta\sim\rho}\!\left[\tr\!\big(I(\theta)\Sigma_D^\dagger(\theta)\big)\right].
\end{equation}
\end{lemma}

\begin{proof}[Proof of \cref{lem:vt-lb-trace}]
Fix any estimator $\tilde\theta_n$. We first compute the Fisher information of one comparison and of $n$ comparisons. For one queried edge $e=(x,y^+,y^-)$, the label satisfies
\[
a\mid (e,\theta)\sim {\rm Bernoulli}\!\big(\sigma(u_\theta(e))\big),
\qquad
u_\theta(e)=\beta\!\left[
\log\frac{\pi_\theta(y^+\mid x)}{\pi_0(y^+\mid x)}
-\log\frac{\pi_\theta(y^-\mid x)}{\pi_0(y^-\mid x)}
\right],
\]
with gradient feature
\[
g(e;\theta)=\nabla_\theta u_\theta(e).
\]
Hence the one-sample Fisher information is
\begin{equation}\label{eq:one-sample-fisher}
\mathcal I_D(\theta)
\doteq
\E_{e\sim D}\!\big[\sigma'(u_\theta(e))\,g(e;\theta)g(e;\theta)^\top\big].
\end{equation}
Since $\sigma'(u)\le 1/4$ for all $u$, we have
\begin{equation}\label{eq:fisher-upper-by-SigmaD}
\mathcal I_D(\theta)\preceq \frac14\,\Sigma_D(\theta).
\end{equation}
For $n$ i.i.d.\ queried comparisons under the same design, Fisher information adds:
\begin{equation}\label{eq:n-sample-fisher}
\mathcal I_n(\theta)=n\,\mathcal I_D(\theta)\preceq \frac n4\,\Sigma_D(\theta).
\end{equation}

Second, the Van Trees inequality gives the matrix lower bound
\begin{equation}\label{eq:vt-matrix}
\E_\rho \E_{\mathcal D_n\mid \theta}
\!\left[(\tilde\theta_n-\theta)(\tilde\theta_n-\theta)^\top\right]\succeq \bigl(\E_\rho[\mathcal I_n(\theta)] + J(\rho)\bigr)^{-1}.
\end{equation}

Let
\[
M_n \;\doteq\; \E_\rho \E_{\mathcal D_n\mid \theta}
\!\left[(\tilde\theta_n-\theta)(\tilde\theta_n-\theta)^\top\right],
\qquad
B_n \;\doteq\; \bigl(\E_\rho[\mathcal I_n(\theta)] + J(\rho)\bigr)^{-1}.
\]
Define the prior-averaged Fisher matrix
\[
\bar I_\rho \;\doteq\; \E_\rho[I(\theta)].
\]
Since \(\bar I_\rho \succeq 0\), multiplying \eqref{eq:vt-matrix} by \(\bar I_\rho\) and taking trace preserves the inequality:
\[
\tr\!\big(\bar I_\rho M_n\big)\;\ge\;\tr\!\big(\bar I_\rho B_n\big).
\]
Indeed, this follows from
\[
\tr\!\big(\bar I_\rho(M_n-B_n)\big)\ge 0
\quad\text{whenever}\quad
\bar I_\rho\succeq 0,\; M_n-B_n\succeq 0.
\]

For the left-hand side, use linearity of trace and expectation:
\[
\tr\!\big(\bar I_\rho M_n\big)
=
\E_\rho \E_{\mathcal D_n\mid \theta}
\!\left[
\tr\!\Big(\bar I_\rho(\tilde\theta_n-\theta)(\tilde\theta_n-\theta)^\top\Big)
\right].
\]
Applying the identity \(\tr(Axx^\top)=x^\top A x\) (with \(A=\bar I_\rho\), \(x=\tilde\theta_n-\theta\)) yields
\[
\tr\!\big(\bar I_\rho M_n\big)
=
\E_\rho \E_{\mathcal D_n\mid \theta}
\!\Big[
(\tilde\theta_n-\theta)^\top \bar I_\rho (\tilde\theta_n-\theta)
\Big].
\]

Substituting this and \(B_n=(\E_\rho[\mathcal I_n(\theta)] + J(\rho))^{-1}\) gives
\begin{equation}\label{eq:trace-vt}
\E_\rho\E_{\mathcal D_n\mid\theta}\!\Big[
(\tilde\theta_n-\theta)^\top \bar I_\rho (\tilde\theta_n-\theta)
\Big]
\;\ge\;
\tr\!\Big(\bar I_\rho\,(\E_\rho[\mathcal I_n(\theta)] + J(\rho))^{-1}\Big).
\end{equation}

Third, we absorb the prior information $J(\rho)$ into the sample information and lower bound \((\E_\rho[\mathcal I_n(\theta)] + J(\rho))^{-1}\) by a multiple of $\frac1n\bar\Sigma_D^{-1}$. Recall from \eqref{eq:n-sample-fisher} that
\[
\E_\rho[\mathcal I_n(\theta)]
\;\preceq\;
\frac n4\,\E_\rho[\Sigma_D(\theta)].
\]
Therefore,
\[
\E_\rho[\mathcal I_n(\theta)] + J(\rho)
\;\preceq\;
\frac n4\,\E_\rho[\Sigma_D(\theta)] + J(\rho).
\]
Since matrix inversion reverses the Loewner order on positive definite matrices,
\[
(\E_\rho[\mathcal I_n(\theta)] + J(\rho))^{-1}
\;\succeq\;
\left(\frac n4\,\E_\rho[\Sigma_D(\theta)] + J(\rho)\right)^{-1}.
\]
Thus it remains to lower bound the right-hand side by a multiple of
$\frac1n\,\E_\rho[\Sigma_D(\theta)]^\dagger$. Define
\begin{equation}\label{eq:n-prior-def}
n_{\rm prior}
\;\doteq\;
\left\lceil \frac{4\,\lambda_{\max}\big(J(\rho)\big)}{\mu_{\mathcal{R}}} \right\rceil.
\end{equation}
Then, by \cref{lem:prior-absorb}, for every $n\ge n_{\rm prior}$,
\[
J(\rho)\preceq \frac n4\bar\Sigma_\rho.
\]
Hence,
\[
\frac n4\bar\Sigma_\rho+J(\rho)
\preceq
\frac n4\bar\Sigma_\rho+\frac n4\bar\Sigma_\rho
=\frac n2\bar\Sigma_\rho.
\]
Applying inverse monotonicity gives
\[
\left(\frac n4\bar\Sigma_\rho+J(\rho)\right)^{-1}
\succeq
\left(\frac n2\bar\Sigma_\rho\right)^{-1}
=\frac2n\bar\Sigma_\rho^{-1}.
\]

Therefore, for all $n\ge n_{\rm prior}$,
\begin{equation}\label{eq:prior-absorbed}
\left(\frac n4\,\E_\rho[\Sigma_D(\theta)] + J(\rho)\right)^{-1}
\;\succeq\;
\frac{2}{n}\,\E_\rho[\Sigma_D(\theta)]^{-1}.
\end{equation}

Substituting \eqref{eq:prior-absorbed} into \eqref{eq:trace-vt}, we obtain
\begin{equation}\label{eq:trace-vt-2}
\E_\rho\E_{\mathcal D_n\mid\theta}\!\Big[
(\tilde\theta_n-\theta)^\top \bar I_\rho (\tilde\theta_n-\theta)
\Big]
\;\ge\;
\frac{2}{n}\,
\tr\!\Big(\bar I_\rho\,\E_\rho[\Sigma_D(\theta)]^{-1}\Big),
\qquad \forall n\ge n_{\rm prior}.
\end{equation}

Finally, we replace the averaged matrices by the actual matrices at $\theta$. By \cref{ass:design-coverage-b}, there exists $\mu_{\mathcal{R}}>0$ such that
\[
\Sigma_D(\theta)\succeq \mu_{\mathcal{R}} I,
\qquad \forall \theta\in\supp{(\rho)}.
\]
Moreover, by the uniform boundedness of the pair-feature map $g(e;\theta)$
(\cref{lem:u-properties}), there exists $G<\infty$ such that
\[
\|g(e;\theta)\|_2\le G
\qquad \forall e,\ \forall\theta\in\Theta.
\]
Hence
\[
\Sigma_D(\theta)=\E_{e\sim D}[g(e;\theta)g(e;\theta)^\top]\preceq G^2 I,
\qquad \forall \theta\in\Theta.
\]
Therefore, we have
\[
\mu_{\mathcal{R}} I \preceq \Sigma_D(\theta)\preceq G^2 I,
\qquad \forall \theta\in\supp{(\rho)}.
\]
Taking expectation gives
\[
\mu_{\mathcal{R}} I \preceq \bar\Sigma_D \preceq G^2 I.
\]
Moreover, for each $\theta$,
\[
\bar\Sigma_D \preceq G^2 I \preceq \frac{G^2}{\mu_{\mathcal{R}}}\,\Sigma_D(\theta),
\]
and thus, by inverse monotonicity,
\[
\bar\Sigma_D^{-1}
\;\succeq\;
\frac{\mu_{\mathcal{R}}}{G^2}\,\Sigma_D(\theta)^{-1}.
\]
Since $I(\theta)\succeq 0$, we obtain
\[
\tr\!\big(I(\theta)\bar\Sigma_D^{-1}\big)
\;\ge\;
\frac{\mu_{\mathcal{R}}}{G^2}\,
\tr\!\big(I(\theta)\Sigma_D(\theta)^{-1}\big).
\]
Taking expectation over $\theta\sim\rho$ and using linearity of trace,
\begin{equation}\label{eq:jensen-trace}
\tr\!\big(\bar I_\rho\,\bar\Sigma_D^{-1}\big)
=
\E_\rho\!\left[\tr\!\big(I(\theta)\bar\Sigma_D^{-1}\big)\right]
\;\ge\;
\frac{\mu_{\mathcal{R}}}{G^2}\,
\E_\rho\!\left[\tr\!\big(I(\theta)\Sigma_D(\theta)^{-1}\big)\right].
\end{equation}
Combining \eqref{eq:trace-vt-2} and \eqref{eq:jensen-trace}, we get
\[
\E_\rho\E_{\mathcal D_n\mid\theta}\!\Big[
(\tilde\theta_n-\theta)^\top \bar I_\rho (\tilde\theta_n-\theta)
\Big]
\;\ge\;
\frac{2\mu_{\mathcal{R}}}{nG^2}\,
\E_\rho\!\left[\tr\!\big(I(\theta)\Sigma_D(\theta)^\dagger\big)\right].
\]

 Therefore, by \cref{lem:fisher-uniform-bounds},
\[
\E_\rho\E_{\mathcal D_n\mid\theta}\!\Big[
(\tilde\theta_n-\theta)^\top
I(\theta)(\tilde\theta_n-\theta)
\Big]
\ge
\frac{m_I}{M_I}
\E_\rho\E_{\mathcal D_n\mid\theta}\!\Big[
(\tilde\theta_n-\theta)^\top
\bar I_\rho(\tilde\theta_n-\theta)
\Big]
\ge
\frac{2\mu_{\mathcal{R}}m_I}{nG^2 M_I}
\E_\rho\!\left[
\tr\!\big(I(\theta)\Sigma_D^\dagger(\theta)\big)
\right].
\]
This proves \eqref{eq:vt-lb-trace-main}.

\end{proof}

\begin{lemma}[Functional Van Trees lower bound]
 \label{lemma:functional-vt}
Suppose
\cref{ass:identifiable-space,ass:model-bounded,ass:feature-sep,ass:design-coverage-b,ass:prior-vantrees}
hold with $\mathcal R=\supp(\rho)$. Define
\[
J(\rho)
\doteq
\E_{\theta\sim\rho}\!\left[
\nabla\log\rho(\theta)\nabla\log\rho(\theta)^\top
\right],
\qquad
n_{\rm prior}
\doteq
 \lceil \frac{c_{\rm reg}}{\mu_\mathcal{R}}(1+\tr(J(\rho)))\rceil,
\]
where $c_{\rm reg}>0$ defined in \cref{lem:functional-van-tree-control} is a constant that only depends on the constants in the assumptions. For every policy estimator
$\tilde\pi_n$ measurable with respect to the $n$ i.i.d.\
pairwise comparisons drawn under $D$, and every
$n\ge n_{\rm prior}$,
\begin{equation}\label{eq:vt-lb-trace-main}
\E_{\theta\sim\rho}\E_{\mathcal D_n\mid\theta,D}
\left[
\frac1N\sum_{i=1}^N
\left\|
\sqrt{\tilde\pi_n(\cdot\mid x_i)}
-
\sqrt{\pi_\theta(\cdot\mid x_i)}
\right\|_2^2
\right]
\ge
\frac{1}{2n}
\E_{\theta\sim\rho}
\left[
\tr\!\left(I(\theta)\Sigma_D^\dagger(\theta)\right)
\right].
\end{equation}
\end{lemma}

\begin{proof}[Proof of \cref{lemma:functional-vt}]
   For any parameter $\theta\in\R^p$, define $\Phi(\theta):\Theta\mapsto\R^{m\times d}$ by 
        \[
        \Phi(\theta)\doteq \Big(m^{-1/2}\sqrt{\pi_\theta(y|x)}\Big)_{(x,y)}.
        \]
    Then $\|\Phi(\theta)-\Phi(\theta^\star)\|^2_2$ is the squared Hellinger distance between $\Phi(\theta)$ and $\Phi(\theta^\star)$, and $\nabla\Phi(\theta)^\top\nabla\Phi(\theta)=I(\theta)/4.$ Similar to the proof of \cref{lem:vt-lb-trace}, we define one-sample Fisher information matrix $\mathcal{I}_D$ in \cref{eq:one-sample-fisher} and equation \cref{eq:fisher-upper-by-SigmaD} holds. 

   For any estimator $\tilde{\theta}_n$, the functional Van Tree's inequality gives 
    \[
\E_{\rho}\E_{\mathcal D_n|\theta,D}
\|\Phi(\Tilde{\theta}_n)-\Phi(\theta)\|_2^2
\ge
\frac{\bigl(\E_\rho\tr(\nabla\Phi(\theta)\Sigma^\dagger_D(\theta)\nabla\Phi(\theta)^\top)\bigr)^2}
     {n\E_\rho\tr\Big(\nabla\Phi(\theta)\Sigma^\dagger_D(\theta)\mathcal{I}_D(\theta)\Sigma^\dagger_D(\theta)\nabla\Phi(\theta)^\top\Big) +K_D}
     \ge 
     \frac{\Big(\E_\rho\tr(I(\theta)\Sigma^\dagger_D(\theta))\Big)^2}
     {n\E_\rho\tr\Big(I(\theta)\Sigma^\dagger_D(\theta)\Big) +16K_D},
\]
where $K_D$ is defined by 
\[
K_D
\doteq\E_\rho\left\|
\sum_{j=1}^p\partial_{\theta_j}(\nabla\Phi(\theta)\Sigma^\dagger_D(\theta) e_j)
+\nabla\Phi(\theta)\Sigma^\dagger_D(\theta)\nabla\log\rho
\right\|_2^2,
\]
with $e_j$ is the j-th coordinate vector in $\R^p$. By \cref{lem:functional-van-tree-control}, we have 
\begin{equation*}
   16 \frac{K_D}{\E_{\theta\sim\rho}
\left[
\tr\!\left(I(\theta)\Sigma_D^\dagger(\theta)\right)
\right]}\leq c_{\rm reg}(1+\tr(J(\rho)))/\mu_{\mathcal{R}}
\end{equation*}
Therefore, when $n\ge n_{\prior}\doteq 
\lceil \frac{c_{\rm reg}}{\mu_\mathcal{R}}(1+\tr(J(\rho)))\rceil$, we have the claimed bound.
\end{proof}

As a result of \cref{prop:KL-sandwich-global,prop:KL-sandwich-nonlinear} and \cref{lem:vt-lb-trace,lemma:functional-vt}, we prove the lower bound as follows.

\begin{theorem}[Van Trees lower bound under a smooth prior]\label{prop:vt-lb-trace}
Suppose \cref{ass:identifiable-space,ass:model-bounded,ass:feature-sep,ass:design-coverage-b,ass:prior-vantrees} hold with $\mathcal{R}=\supp{(\rho)}$. Let $\tilde\pi_n=\pi_{\tilde\theta_n}$ be any induced policy by estimator $\tilde\theta_n$.

\begin{enumerate}
    \item Suppose $f_\theta(\phi)=\theta^\top\phi$. Then for all $n\ge n_{\rm prior}$ defined in \cref{lem:vt-lb-trace},
\begin{equation}\label{eq:vt-lb-gap-main}
\E_{\theta^\star\sim\rho}\E_{\mathcal D_n\mid \theta^\star}\!\Big[J(\pi_{\theta^\star})-J(\tilde\pi_n)\Big]
\;\ge\;
\frac{C_{\rm lb}}{n}\,
\E_{\theta^\star\sim\rho}\!\left[\tr\!\big(I(\theta^\star)\Sigma_D^\dagger(\theta^\star)\big)\right],
\end{equation}
for a constant $C_{\rm lb}>0$ depending only on constants appearing in assumptions, the prior $\rho$, and the sandwich constant.

\item Suppose $f_\theta$ is a general function. Then for all $n\ge n'_{\rm prior}$ defined in \cref{lemma:functional-vt},
\begin{equation}
\E_{\theta^\star\sim\rho}\E_{\mathcal D_n\mid \theta^\star}\!\Big[J(\pi_{\theta^\star})-J(\tilde\pi_n)\Big]
\;\ge\;
\frac{\beta}{2n}\,
\E_{\theta^\star\sim\rho}\!\left[\tr\!\big(I(\theta^\star)\Sigma_D^\dagger(\theta^\star)\big)\right].
\end{equation}
\end{enumerate}
\end{theorem}

\begin{proof}[Proof of \cref{prop:vt-lb-trace}]
By the lower side of \cref{prop:KL-sandwich-global}, there exists $c_->0$ such that for every
estimator output $\tilde\theta_n$,
\[
J(\pi_{\theta^*})-J(\pi_{\tilde\theta_n})
\;\ge\;
c_-\,(\tilde\theta_n-\theta^*)^\top I(\theta^*)(\tilde\theta_n-\theta^*).
\]
Taking $\E_\rho\E_{\mathcal D_n\mid\theta^*}$ and combining with \eqref{eq:vt-lb-trace-main}
gives \eqref{eq:vt-lb-gap-main}. Similarly, by \cref{prop:KL-sandwich-nonlinear}, we prove the second statement. 
\end{proof}

\section{Proof of \cref{thm:theta0}}\label{app:proof-theta0}

\begin{proof}[Proof of \cref{thm:theta0}]
We compare the reference-optimal design and the true oracle
at $\theta_0$ and $\theta^\star$. To control their inverse
covariances under this parameter change, we first establish
a positive lower bound on their smallest eigenvalues at
their respective optimization parameters. Write
\[
T_0(D)=\tr\!\left(I(\theta_0)\Sigma_D^\dagger(\theta_0)\right),
\qquad
T_\star(D)=
\tr\!\left(I(\theta^\star)\Sigma_D^\dagger(\theta^\star)\right),
\]
and recall that 
\[
D_{\theta_0}\in\arg\min_{D\in\Delta(E)}T_0(D),
\qquad
D^\star(\theta^\star)\in\arg\min_{D\in\Delta(E)}T_\star(D).
\]
In all design optimization problems above, the criterion $\operatorname{tr}(I(\theta)\Sigma_D^\dagger(\theta))$ is understood to be $+\infty$ whenever $\Sigma_D(\theta)$ is singular on $H$.

\noindent\textbf{(i)} We first show that $\Sigma_{D_{\theta_0}}(\theta_0)$ and $\Sigma_{D^\star(\theta^\star)}(\theta^\star)$ has minimal coverage, i.e., the smallest eigenvalue is bounded away from $0$. For any design with positive definite covariance on $H$,
by \cref{lem:Fisher-sandwich},
\[
\tr\!\left(I(\theta)\Sigma_D^\dagger(\theta)\right)
\ge
\frac{\underline\mu}
     {\lambda_{\min}^H(\Sigma_D(\theta))}.
\]
Thus, an upper bound on the optimal value of trace yields the required lower bound for the smallest eigenvalue. To obtain an explicit bound, motivated by the covariance form \cref{eq:fisher-softmax-cov} of policy fisher matrix $I(\theta)$, consider the following sampling design, for any pair $\{y,z\}\in \mathcal{E}_{x_i},$
\[
D_\theta^{\rm aux}(x_i,\{y,z\})
\propto
\pi_\theta(y\mid x_i)\pi_\theta(z\mid x_i),
\]
with normalization over $\mathcal{E}$. This is feasible because
$\mathcal{E}$ contains all within-prompt pairs. 

Recall for $e=(x_i,\{y,z\})$, we have
\[
g(e;\theta)
=
\beta\bigl(
\nabla_\theta f_\theta(\phi(x_i,y))
-
\nabla_\theta f_\theta(\phi(x_i,z))
\bigr).
\]
Since $I(\theta)$ is the prompt-averaged covariance of
the logit gradients (see \cref{eq:fisher-softmax-cov}), 
\[
\frac1N\sum_{i=1}^N
\sum_{\{y,z\}\in\mathcal E_{x_i}}
\pi_\theta(y\mid x_i)\pi_\theta(z\mid x_i)
g(e;\theta)g(e;\theta)^\top
=
\beta^2 I(\theta).
\]
The auxiliary design normalizes the weights
$\pi_\theta(y\mid x_i)\pi_\theta(z\mid x_i)$ over all
edges. Hence, by the definition of $\Sigma_D(\theta)$,
\[
\Sigma_{D_\theta^{\rm aux}}(\theta)
=
\frac{\beta^2 I(\theta)}
{\displaystyle
 \frac1N\sum_{i=1}^N
 \sum_{\{y,z\}\in\mathcal E_{x_i}}
 \pi_\theta(y\mid x_i)\pi_\theta(z\mid x_i)}.
\]
Because $\mathcal E_{x_i}$ contains all unordered pairs,
\[
\sum_{\{y,z\}\in\mathcal E_{x_i}}
\pi_\theta(y\mid x_i)\pi_\theta(z\mid x_i)
=
\frac12\left(
1-\sum_y\pi_\theta(y\mid x_i)^2
\right)
\le\frac12.
\]
Thus the denominator is positive and at most $1/2$, giving
\[
\Sigma_{D_\theta^{\rm aux}}(\theta)
\succeq 2\beta^2 I(\theta).
\]

Therefore,
\[
\inf_{D\in\Delta(\mathcal{E})}
\tr\!\left(I(\theta)\Sigma_D^\dagger(\theta)\right)
\le
\tr\!\left(
I(\theta)\Sigma_{D_\theta^{\rm aux}}^\dagger(\theta)
\right)
\le\frac{d_H}{2\beta^2}.
\]
The trace criterion is lower semicontinuous on the
compact simplex and finite at this feasible design, so
its minimum is attained. Combining the preceding bounds
at the two optimization parameters yields
\[
\Sigma_{D_{\theta_0}}(\theta_0)
\succeq\frac{2\beta^2\underline\mu}{d_H}I_H,
\qquad
\Sigma_{D^\star(\theta^\star)}(\theta^\star)
\succeq\frac{2\beta^2\underline\mu}{d_H}I_H.
\]

\noindent\textbf{(ii)} Next, we compare the sample covariance matrix $\Sigma_D$ as follows. By \cref{lem:u-properties}, uniformly over designs,
\[
\begin{aligned}
\|\Sigma_D(\theta^\star)-\Sigma_D(\theta_0)\|_{\rm op}
&\le
\E_{e\sim D}\Big[
\|g(e;\theta^\star)-g(e;\theta_0)\|_2
\bigl(\|g(e;\theta^\star)\|_2+\|g(e;\theta_0)\|_2\bigr)
\Big]\\
&\le 2G_uH_u r_0.
\end{aligned}
\]
For the reference-optimal design, Step 1 consequently gives
\[
\Sigma_{D_{\theta_0}}(\theta^\star)
\succeq
\Sigma_{D_{\theta_0}}(\theta_0)-2G_uH_u r_0I_H
\succeq
\left(
1-\frac{d_HG_uH_u r_0}{\beta^2\underline\mu}
\right)
\Sigma_{D_{\theta_0}}(\theta_0).
\]
Applying the same argument to the true oracle, with the
roles of the two parameters reversed, gives
\[
\Sigma_{D^\star(\theta^\star)}(\theta_0)
\succeq
\left(
1-\frac{d_HG_uH_u r_0}{\beta^2\underline\mu}
\right)
\Sigma_{D^\star(\theta^\star)}(\theta^\star).
\]
Hence, with
\[
C_\Sigma(r_0)
=
\left(
1-\frac{d_HG_uH_u r_0}{\beta^2\underline\mu}
\right)^{-1},
\]
inverse monotonicity yields
\[
\Sigma_{D_{\theta_0}}^\dagger(\theta^\star)
\preceq
C_\Sigma(r_0)\Sigma_{D_{\theta_0}}^\dagger(\theta_0),
\]
and
\[
\Sigma_{D^\star(\theta^\star)}^\dagger(\theta_0)
\preceq
C_\Sigma(r_0)
\Sigma_{D^\star(\theta^\star)}^\dagger(\theta^\star).
\]
We may take $r_{\rm ref}
\doteq
\frac{\beta^2\underline\mu}{d_HG_uH_u},$ where the constants are defined in
\cref{lem:u-properties,lem:Fisher-sandwich}
and $d_H=\dim(H)$. Hence, the denominator is positive when $r_0<r_{\rm ref}$.

\noindent\textbf{(iii)} Similarly, we show the following bounds for policy Fisher matrix. 
For $s_\theta(y\mid x)=\nabla_\theta\log\pi_\theta(y\mid x)$,
the bounds in the proof of \cref{lem:u-properties} give
\[
\|s_\theta\|_2\le2\alpha_1,
\qquad
\|\nabla_\theta s_\theta\|_{\rm op}
\le2\alpha_2+4\alpha_1^2.
\]
Differentiating
$I(\theta)=\E_x\E_{\pi_\theta}[s_\theta s_\theta^\top]$
therefore yields
\[
\begin{aligned}
\|\mathrm D_\theta I(\theta)[h]\|_{\rm op}
&\le
\left[
(2\alpha_1)^3+
4\alpha_1(2\alpha_2+4\alpha_1^2)
\right]\|h\|_2\\
&=
8\alpha_1(\alpha_2+3\alpha_1^2)\|h\|_2.
\end{aligned}
\]
Integrating along the parameter segment and using
$I(\theta_0),I(\theta^\star)\succeq\underline\mu I_H$
gives
\[
I(\theta^\star)\preceq C_I(r_0)I(\theta_0),
\qquad
I(\theta_0)\preceq C_I(r_0)I(\theta^\star),
\]
where
\[
C_I(r_0)
=
1+\frac{8\alpha_1(\alpha_2+3\alpha_1^2)}
        {\underline\mu}\,r_0.
\]

\noindent\textbf{(iv)} Combining Steps 2 and 3, and setting
$C_{\rm ref}(r_0)=C_I(r_0)C_\Sigma(r_0)>1$, gives
\[
T_\star(D_{\theta_0})
\le C_{\rm ref}(r_0)T_0(D_{\theta_0})
\le C_{\rm ref}(r_0)T_0(D^\star(\theta^\star))
\le C_{\rm ref}(r_0)^2
     T_\star(D^\star(\theta^\star)).
\]
The second inequality follows from the optimality of
$D_{\theta_0}$. This proves the claim.
\end{proof}

\section{Technical lemmas}
In this appendix, we list several useful technical lemmas, which are applied throughout the whole paper.

\begin{lemma}[Uniform properties of the DPO logit \(u_\theta\)]
\label{lem:u-properties}
Suppose \cref{ass:model-bounded} holds. Define, for \(e=(x,y^+,y^-)\),
\[
u_\theta(e)
:=
\beta\!\left(
\log\pi_\theta(y^+\mid x)-\log\pi_\theta(y^-\mid x)
-\log\pi_0(y^+\mid x)+\log\pi_0(y^-\mid x)
\right).
\]
Then the following hold uniformly over \(\theta\in\Theta\) and admissible \(e\):

\begin{enumerate}[label=(\roman*)]
\item Uniform boundedness of \(u_\theta\):
\[
|u_\theta(e)|\le U_{\max}.
\]

\item Uniform boundedness of gradient:
\[
\|\nabla_\theta u_\theta(e)\|_2\le G_u.
\]

\item Uniform boundedness of Hessian:
\[
\|\nabla_\theta^2 u_\theta(e)\|_{\op}\le H_u.
\]

\item Lipschitz continuity in \(\theta\):
for any \(\theta,\theta'\in\Theta\),
\[
|u_\theta(e)-u_{\theta'}(e)|\le G_u\|\theta-\theta'\|_2,\qquad
\|\nabla_\theta u_\theta(e)-\nabla_\theta u_{\theta'}(e)\|_2
\le H_u\|\theta-\theta'\|_2.
\]
\end{enumerate}

A valid choice of constants is
\[
G_u=4\beta\alpha_1,\qquad
H_u=2\beta(2\alpha_2+4\alpha_1^2),
\]
and, if there exists \(\underline p>0\) such that
\(\pi_\theta(y\mid x)\ge \underline p\) and \(\pi_0(y\mid x)\ge \underline p\) for all
\((x,y),\theta\), then
\[
U_{\max}=2\beta\log(1/\underline p^2).
\]
\end{lemma}

\begin{proof}[Proof of \cref{lem:u-properties}]
Write
\[
\log\pi_\theta(y\mid x)=
f_\theta(\phi(x,y))-\log\!\sum_{y'}\exp(f_\theta(\phi(x,y'))).
\]
Hence
\[
\nabla_\theta \log\pi_\theta(y\mid x)
=
\nabla_\theta f_\theta(\phi(x,y))
-
\E_{y'\sim \pi_\theta(\cdot\mid x)}\!\big[\nabla_\theta f_\theta(\phi(x,y'))\big].
\]
So, by Assumption~1 (\(\|\nabla_\theta f_\theta\|_2\le\alpha_1\)),
\[
\|\nabla_\theta \log\pi_\theta(y\mid x)\|_2\le 2\alpha_1.
\]
Therefore
\[
\nabla_\theta u_\theta(e)
=
\beta\!\left(
\nabla_\theta\log\pi_\theta(y^+\mid x)-\nabla_\theta\log\pi_\theta(y^-\mid x)
\right),
\]
and
\[
\|\nabla_\theta u_\theta(e)\|_2
\le \beta(2\alpha_1+2\alpha_1)=4\beta\alpha_1=:G_u.
\]
This proves (ii).

For the Hessian, use the standard softmax identity:
\[
\nabla_\theta^2\log\pi_\theta(y\mid x)
=
\nabla_\theta^2 f_\theta(\phi(x,y))
-\E_{y'}[\nabla_\theta^2 f_\theta(\phi(x,y'))]
-\Cov_{y'\sim\pi_\theta(\cdot\mid x)}\!\big(\nabla_\theta f_\theta(\phi(x,y'))\big).
\]
Hence
\[
\|\nabla_\theta^2\log\pi_\theta(y\mid x)\|_{\op}
\le
\alpha_2+\alpha_2+\underbrace{\|\Cov(\nabla_\theta f)\|_{\op}}_{\le 4\alpha_1^2}
\le 2\alpha_2+4\alpha_1^2.
\]
Thus
\[
\|\nabla_\theta^2u_\theta(e)\|_{\op}
\le
\beta\Big(
\|\nabla_\theta^2\log\pi_\theta(y^+\mid x)\|_{\op}
+\|\nabla_\theta^2\log\pi_\theta(y^-\mid x)\|_{\op}
\Big)
\le 2\beta(2\alpha_2+4\alpha_1^2)=:H_u,
\]
proving (iii).

For Lipschitzness in (iv), apply mean value theorem in \(\theta\):
\[
|u_\theta(e)-u_{\theta'}(e)|
\le
\sup_{\tilde\theta}\|\nabla_\theta u_{\tilde\theta}(e)\|_2\,\|\theta-\theta'\|_2
\le G_u\|\theta-\theta'\|_2,
\]
\[
\|\nabla_\theta u_\theta(e)-\nabla_\theta u_{\theta'}(e)\|_2
\le
\sup_{\tilde\theta}\|\nabla_\theta^2u_{\tilde\theta}(e)\|_{\op}\,\|\theta-\theta'\|_2
\le H_u\|\theta-\theta'\|_2.
\]

Finally, for (i), if \(\pi_\theta,\pi_0\ge \underline p>0\), then
\[
\left|\log\pi_\theta(y^+\mid x)-\log\pi_\theta(y^-\mid x)\right|
\le 2\log(1/\underline p),
\]
and similarly for \(\pi_0\), hence
\[
|u_\theta(e)|\le 2\beta\log(1/\underline p^2)=U_{\max}.
\]
\end{proof}

\begin{lemma}[Global lower bound on logistic curvature from uniform logit bound]
\label{lem:global-kappa0}
Suppose \cref{ass:model-bounded} holds.
Define
\[
\kappa_0:=\sigma(U_{\max})\bigl(1-\sigma(U_{\max})\bigr).
\]
Then
\[
\sigma'\!\bigl(u_\theta(e)\bigr)\ge \kappa_0>0,\qquad
\forall\,\theta\in\Theta,\ \forall\,e.
\]
\end{lemma}

\begin{proof}[Proof of \cref{lem:global-kappa0}]
For logistic sigmoid \(\sigma(t)=1/(1+e^{-t})\),
\[
\sigma'(t)=\sigma(t)\bigl(1-\sigma(t)\bigr)
=\frac{e^{-t}}{(1+e^{-t})^2}
=\frac{1}{2+e^t+e^{-t}}.
\]
Hence \(\sigma'(t)\) is an even function of \(t\), and decreases as \(|t|\) increases.
Therefore, on the interval \([-U_{\max},U_{\max}]\),
\[
\inf_{|t|\le U_{\max}}\sigma'(t)=\sigma'(U_{\max})
=\sigma(U_{\max})\bigl(1-\sigma(U_{\max})\bigr)=\kappa_0.
\]
Since \(|u_\theta(e)|\le U_{\max}\) uniformly over \((\theta,e)\), we have
\[
\sigma'\!\bigl(u_\theta(e)\bigr)\ge
\inf_{|t|\le U_{\max}}\sigma'(t)=\kappa_0,
\]
for all \(\theta\in\Theta\) and all admissible \(e\). Also \(\sigma(U_{\max})\in(0,1)\), so \(\kappa_0>0\).
\end{proof}

\begin{lemma}[Uniform boundedness and Lipschitzness of the DPO loss class]
\label{lem:loss-lip-bdd}
Suppose \cref{ass:model-bounded} holds, and
Lemma~\ref{lem:u-properties} gives
\[
|u_\theta(e)|\le U_{\max},\qquad
\|\nabla_\theta u_\theta(e)\|_2\le G_u,
\quad \forall \theta\in\Theta,\ \forall e.
\]
For \(z=(e,a)\), \(a\in\{0,1\}\), define
\[
\ell_\theta(z):=\ell(a,u_\theta(e)),
\qquad
\ell(a,u):=-a\log\sigma(u)-(1-a)\log(1-\sigma(u)).
\]
Then:

\begin{enumerate}
\item[(i)] (Uniform envelope) for all \(\theta\in\Theta\), \(z\),
\[
0\le \ell_\theta(z)\le B_\ell,
\qquad
B_\ell:=\log(1+e^{U_{\max}}).
\]

\item[(ii)] (Uniform Lipschitz in \(\theta\)) for all \(\theta,\theta'\in\Theta\), \(z\),
\[
|\ell_\theta(z)-\ell_{\theta'}(z)|
\le L_\ell\|\theta-\theta'\|_2,
\qquad
L_\ell:=G_u.
\]
\end{enumerate}
\end{lemma}

\begin{proof}
Use the equivalent form
\[
\ell(a,u)=\log(1+e^u)-au.
\]
Hence, if \(|u|\le U_{\max}\),
\[
0\le \ell(a,u)\le \log(1+e^{U_{\max}})=B_\ell,
\]
which proves (i) since \(|u_\theta(e)|\le U_{\max}\) uniformly.

Next,
\[
\partial_u \ell(a,u)=\sigma(u)-a,\qquad |\partial_u \ell(a,u)|\le 1,
\]
so for any \(u,u'\in\R\),
\[
|\ell(a,u)-\ell(a,u')|\le |u-u'|.
\]
Therefore
\[
|\ell_\theta(z)-\ell_{\theta'}(z)|
\le |u_\theta(e)-u_{\theta'}(e)|.
\]
By mean-value theorem in parameter space,
\[
|u_\theta(e)-u_{\theta'}(e)|
\le
\sup_{\tilde\theta\in\Theta}\|\nabla_\theta u_{\tilde\theta}(e)\|_2\,
\|\theta-\theta'\|_2
\le G_u\|\theta-\theta'\|_2.
\]
So (ii) holds with \(L_\ell=G_u\).
\end{proof}

\begin{lemma}[Hessian regularity of the DPO loss]\label{lem:loss-hessian-properties}
Suppose \cref{ass:identifiable-space,ass:model-bounded} hold.
Let \(U_{\max},G_u,H_u\) be the constants from \cref{lem:u-properties} such that for all
\(\theta\in\Theta\) and all admissible \(e\),
\[
|u_\theta(e)|\le U_{\max},\qquad
\|\nabla_\theta u_\theta(e)\|_2\le G_u,\qquad
\|\nabla_\theta^2 u_\theta(e)\|_{\op}\le H_u.
\]
For \(z=(e,a)\) with \(a\in\{0,1\}\), define
\[
\ell_\theta(z):=\ell(a,u_\theta(e)),
\qquad
\ell(a,u)=-a\log\sigma(u)-(1-a)\log(1-\sigma(u)).
\]
Let \(P_H\) denote the orthogonal projector onto \(H\). Then the following hold uniformly over \(\Theta\):

\begin{enumerate}
\item[(i)] (Per-sample projected Hessian envelope).
\[
\sup_{\theta\in\Theta,\ z}\ \|P_H\nabla_\theta^2\ell_\theta(z)P_H\|_{\op}
\le M_H,
\qquad
M_H:=G_u^2+H_u.
\]

\item[(ii)] (Per-sample projected Hessian Lipschitzness).
There exists a constant \(L_H<\infty\), depending only on the fixed constants in the standing assumptions,
such that for all \(\theta,\theta'\in\Theta\),
\[
\sup_{z}\ \|P_H(\nabla_\theta^2\ell_\theta(z)-\nabla_{\theta'}^2\ell_{\theta'}(z))P_H\|_{\op}
\le L_H\|\theta-\theta'\|_2.
\]
In particular, one may take
\[
L_H:=\frac14G_u^3+\frac34G_uH_u+\Gamma_u,
\]
where
\[
\Gamma_u
\doteq
\sup_{\theta\in\Theta}\ \sup_{e}\ \|\nabla_\theta^3 u_\theta(e)\|_{\op},
\]
which is finite under \cref{ass:model-bounded}.

\item[(iii)] (Uniform variance proxy).
\[
\sup_{\theta\in\Theta}
\left\|
\mathbb E\!\left[
\left(
P_H\nabla_\theta^2\ell_\theta(z)P_H
-\mathbb E[P_H\nabla_\theta^2\ell_\theta(z)P_H]
\right)^2
\right]
\right\|_{\op}
\le V_H,
\qquad
V_H:=4M_H^2.
\]
\end{enumerate}
\end{lemma}

\begin{proof}[Proof of \cref{lem:loss-hessian-properties}]
Fix \(z=(e,a)\) with \(a\in\{0,1\}\) and \(\theta\in\Theta\). Write
\[
u:=u_\theta(e),\qquad g:=\nabla_\theta u_\theta(e),\qquad
U:=\nabla_\theta^2u_\theta(e).
\]
For logistic cross-entropy,
\begin{equation}\label{eq:dpo-loss-hess-decomp}
\nabla_\theta^2\ell_\theta(z)
=
\sigma'(u)\,gg^\top+(\sigma(u)-a)\,U.
\end{equation}

\noindent\textbf{(i)}
Since \(0\le \sigma'(u)\le 1/4\le 1\) and \(|\sigma(u)-a|\le 1\),
\[
\|\nabla_\theta^2\ell_\theta(z)\|_{\op}
\le
\|gg^\top\|_{\op}+\|U\|_{\op}.
\]
Moreover \(\|gg^\top\|_{\op}=\|g\|_2^2\le G_u^2\) and \(\|U\|_{\op}\le H_u\), hence
\[
\|\nabla_\theta^2\ell_\theta(z)\|_{\op}\le G_u^2+H_u=M_H.
\]
Finally, \(\|P_H A P_H\|_{\op}\le \|A\|_{\op}\) for any matrix \(A\), so
\(\|P_H\nabla_\theta^2\ell_\theta(z)P_H\|_{\op}\le M_H\).

\noindent\textbf{(ii)}
Let \((u,g,U)\) and \((u',g',U')\) correspond to \(\theta\) and \(\theta'\), respectively.
Subtracting \eqref{eq:dpo-loss-hess-decomp} at \(\theta\) and \(\theta'\) and regrouping yields
\[
\nabla_\theta^2\ell_\theta(z)-\nabla_{\theta'}^2\ell_{\theta'}(z)
=
(\sigma'(u)-\sigma'(u'))gg^\top
+\sigma'(u')(gg^\top-g'g'^\top)
+(\sigma(u)-\sigma(u'))U
+(\sigma(u')-a)(U-U').
\]
Using
\(
|\sigma'(u)-\sigma'(u')|\le \tfrac14|u-u'|
\),
\(
|\sigma(u)-\sigma(u')|\le \tfrac14|u-u'|
\),
and
\(
|u-u'|\le G_u\|\theta-\theta'\|_2
\),
we obtain
\[
\|(\sigma'(u)-\sigma'(u'))gg^\top\|_{\op}
\le \tfrac14 G_u^3\|\theta-\theta'\|_2,
\qquad
\|(\sigma(u)-\sigma(u'))U\|_{\op}
\le \tfrac14 G_uH_u\|\theta-\theta'\|_2.
\]
Next, since \(gg^\top-g'g'^\top=(g-g')g^\top+g'(g-g')^\top\),
\[
\|gg^\top-g'g'^\top\|_{\op}
\le
\|g-g'\|_2\,\|g\|_2+\|g'\|_2\,\|g-g'\|_2
\le
2G_u\,\|g-g'\|_2.
\]
Because \(\|g-g'\|_2\le H_u\|\theta-\theta'\|_2\), it follows that
\[
\|\sigma'(u')(gg^\top-g'g'^\top)\|_{\op}
\le
\tfrac14\cdot 2G_uH_u\|\theta-\theta'\|_2
=
\tfrac12 G_uH_u\|\theta-\theta'\|_2.
\]
Finally, \(|\sigma(u')-a|\le 1\) and the mean-value theorem gives
\[
\|U-U'\|_{\op}
=
\|\nabla_\theta^2u_\theta(e)-\nabla_{\theta'}^2u_{\theta'}(e)\|_{\op}
\le
\Gamma_u\|\theta-\theta'\|_2.
\]
Therefore
\(
\|(\sigma(u')-a)(U-U')\|_{\op}\le \Gamma_u\|\theta-\theta'\|_2
\).
Summing the four bounds yields
\[
\|\nabla_\theta^2\ell_\theta(z)-\nabla_{\theta'}^2\ell_{\theta'}(z)\|_{\op}
\le
\Big(\tfrac14G_u^3+\tfrac34G_uH_u+\Gamma_u\Big)\|\theta-\theta'\|_2,
\]
and projecting by \(P_H\) does not increase the operator norm, proving (ii).

\noindent\textbf{(iii)}
Define the centered matrix
\[
X_\theta:=P_H\nabla_\theta^2\ell_\theta(z)P_H
-\mathbb E[P_H\nabla_\theta^2\ell_\theta(z)P_H].
\]
By (i), \(\|P_H\nabla_\theta^2\ell_\theta(z)P_H\|_{\op}\le M_H\), hence
\(\|X_\theta\|_{\op}\le 2M_H\). Therefore,
\[
\left\|\mathbb E[X_\theta^2]\right\|_{\op}
\le
\mathbb E\|X_\theta^2\|_{\op}
\le
\mathbb E\|X_\theta\|_{\op}^2
\le 4M_H^2
=
V_H,
\]
uniformly in \(\theta\in\Theta\).
\end{proof}

\begin{lemma}[Gram perturbation bound for $\Sigma_D(\theta^\star)$]
\label{lem:gram-perturb}
Let $g(e;\theta)=\nabla_\theta u_\theta(e)$ and
\[
\Sigma_D(\theta)\doteq \E_{e\sim D}\big[g(e;\theta)g(e;\theta)^\top\big].
\]
Suppose that for all $e$ and all $\theta,\theta'\in\Theta$,
\[
\|g(e;\theta)\|_2 \le G,
\qquad
\|g(e;\theta)-g(e;\theta')\|_2 \le H_u\|\theta-\theta'\|_2
\]
for some constants $G,H_u>0$. Then for every $\theta\in\Theta$,
\begin{equation}
\label{eq:gram-perturb}
\bigl\|\Sigma_D(\theta)-\Sigma_D(\theta^\star)\bigr\|_{\op}
\;\le\;
2G H_u\|\theta-\theta^\star\|_2
\;+\;
H_u^2\|\theta-\theta^\star\|_2^2.
\end{equation}
\end{lemma}

\begin{proof}[Proof of \cref{lem:gram-perturb}]
Write \(\Delta\theta=\theta-\theta^\star\) and \(\Delta g(e)=g(e;\theta)-g(e;\theta^\star)\).
Using \(g(e;\theta)=g(e;\theta^\star)+\Delta g(e)\), we have the identity
\[
g(e;\theta)g(e;\theta)^\top-g(e;\theta^\star)g(e;\theta^\star)^\top
=
\Delta g(e)\,g(e;\theta^\star)^\top+g(e;\theta^\star)\Delta g(e)^\top+\Delta g(e)\,\Delta g(e)^\top .
\]
Taking expectation and applying \(\|\mathbb E[M]\|_{\op}\le \mathbb E\|M\|_{\op}\) yields
\[
\bigl\|\Sigma_D(\theta)-\Sigma_D(\theta^\star)\bigr\|_{\op}
\le
2\,\mathbb E\!\big[\|\Delta g(e)\|_2\,\|g(e;\theta^\star)\|_2\big]
+
\mathbb E\|\Delta g(e)\|_2^2.
\]
By \(\|g(e;\theta^\star)\|_2\le G\) and \(\|\Delta g(e)\|_2\le H_u\|\Delta\theta\|_2\) for all \(e\),
\[
\mathbb E\!\big[\|\Delta g(e)\|_2\,\|g(e;\theta^\star)\|_2\big]
\le
G\,\mathbb E\|\Delta g(e)\|_2
\le
G H_u\|\Delta\theta\|_2,
\qquad
\mathbb E\|\Delta g(e)\|_2^2
\le
H_u^2\|\Delta\theta\|_2^2.
\]
Combining the above inequalities gives \eqref{eq:gram-perturb}.
\end{proof}

\begin{lemma}[Fisher non-degeneracy]\label{lem:sep-implies-fisher}
Suppose \cref{ass:identifiable-space,ass:model-bounded,ass:feature-sep} hold. In particular, assume
\[
|f_\theta(\phi(x,y))|\le \alpha_0,\qquad \forall \theta\in\Theta,\ \forall x,\ \forall y\in\mathcal A(x),
\]
and that $|\mathcal A(x)|\le d_{\max}$ for all $x$.
Then the Fisher information is uniformly
non-degenerate on the identifiable subspace $H$:
\[
v^\top I(\theta) v \;\ge\; \underline\mu\,\|v\|_2^2,
\qquad \forall v\in H,
\]
where $\underline\mu \;=\; \Big(\frac{e^{-2\alpha_0}}{d_{\max}}\Big)^2\,\Delta_g^2.$
\end{lemma}

\begin{proof}[Proof of \cref{lem:sep-implies-fisher}]
Fix $\theta\in\Theta$ and a prompt $x$. Let
\[
\psi_\theta(x,y)\doteq \nabla_\theta f_\theta(\phi(x,y)),\ \
Z \doteq v^\top \psi_\theta(x,Y),\ \ Y\sim\pi_\theta(\cdot\mid x),
\]
for an arbitrary unit vector $v\in H$. Under the softmax model,
\[
v^\top I(\theta)v
=
\E_x\!\left[
\Var_{y\sim\pi_\theta(\cdot\mid x)}\!\big(v^\top\psi_\theta(x,y)\big)
\right].
\]

First, we show that the softmax probabilities are bounded away from $0$. Since $|f_\theta(\phi(x,y))|\le \alpha_0$, we have $\exp(f_\theta(\phi(x,y)))\in[e^{-\alpha_0},e^{\alpha_0}]$. Hence
\[
\sum_{y'\in\mathcal A(x)}\exp(f_\theta(\phi(x,y')))\le |\mathcal A(x)|\,e^{\alpha_0}\le d_{\max}e^{\alpha_0},
\]
and therefore for every $y\in\mathcal A(x)$,
\[
\pi_\theta(y\mid x)
=
\frac{\exp(f_\theta(\phi(x,y)))}{\sum_{y'}\exp(f_\theta(\phi(x,y')))}
\;\ge\;
\frac{e^{-\alpha_0}}{d_{\max}e^{\alpha_0}}
=
\frac{e^{-2\alpha_0}}{d_{\max}}
\doteq p_{\min}.
\]

Second, we show the variance lower bound from two separated atoms. For any random variable $Z$ supported on discrete values $\{z_y\}$ with
probabilities $\{p_y\}$, we have 
\[
\Var(Z)= \frac12\sum_{i,j}p_ip_j(z_i-z_j)^2\;\ge\; p_{y_1}p_{y_2}(z_{y_1}-z_{y_2})^2,
\]
for any two indices $y_1,y_2$. By \cref{ass:feature-sep}, there exist $y_1,y_2\in\mathcal A(x)$ such that
\[
|z_{y_1}-z_{y_2}|=\Big|v^\top(\psi_\theta(x,y_1)-\psi_\theta(x,y_2))\Big|\ge \Delta_g.
\]
Combining with $p_{y_1},p_{y_2}\ge p_{\min}$ yields
\[
\Var_{Y\sim\pi_\theta(\cdot\mid x)}\!\big(v^\top\psi_\theta(x,Y)\big)
\;\ge\;
p_{\min}^2\,\Delta_g^2.
\]

Using the bound above for each $x$ gives
\[
v^\top I(\theta)v \;\ge\; p_{\min}^2\,\Delta_g^2
=
\Big(\frac{e^{-2\alpha_0}}{d_{\max}}\Big)^2\,\Delta_g^2,
\]
for all $\theta\in\Theta$ and all unit $v\in H$. 
\end{proof}

\begin{lemma}[Covariance bound for $\nabla L_n(\theta^\star)$]\label{lem:score-cov}
Suppose \cref{ass:identifiable-space,ass:model-bounded} hold, and the pairwise labels are well-specified:
\[
a_i \mid e_i \sim \mathrm{Bernoulli}\!\big(\sigma(u_i(\theta^\star))\big),
\qquad i=1,\dots,n,
\]
with conditional independence across $i$ given $\{e_i\}_{i=1}^n$.
Then
\begin{equation}\label{eq:grad-mean-zero-cond}
\E\!\left[\nabla L_n(\theta^\star)\,\middle|\,\{e_i\}_{i=1}^n\right]=0,
\end{equation}
and
\begin{equation}\label{eq:grad-cov-cond}
\E\!\left[\nabla L_n(\theta^\star)\nabla L_n(\theta^\star)^\top \,\middle|\, \{e_i\}_{i=1}^n\right]
\preceq
\frac{1}{4n}\,\widehat\Sigma_n(\theta^\star).
\end{equation}
Moreover, we have
\begin{equation}\label{eq:grad-second-moment-pinv}
\E\!\left[\big\|\Sigma_D(\theta^\star)^{\dagger/2}\nabla L_n(\theta^\star)\big\|_2^2\right]
\le \frac{\dim(H)}{4n}.
\end{equation}
\end{lemma}

\begin{proof}[Proof of \cref{lem:score-cov}]
For logistic loss, \(\partial \ell(a,u)/\partial u=\sigma(u)-a\). Hence for any \(\theta\in\Theta\),
\[
\nabla L_n(\theta)
=
\frac1n\sum_{i=1}^n\big(\sigma(u_\theta(e_i))-a_i\big)\,g_\theta(e_i).
\]
In particular, at \(\theta^\star\),
\[
\nabla L_n(\theta^\star)
=
\frac1n\sum_{i=1}^n\big(\sigma(u_{\theta^\star}(e_i))-a_i\big)\,g_{\theta^\star}(e_i).
\]

Condition on \(\{e_i\}_{i=1}^n\). By realizability,
\[
\E[a_i\mid e_i]=\sigma\!\big(u_{\theta^\star}(e_i)\big),
\]
so each summand has conditional mean zero, which implies \eqref{eq:grad-mean-zero-cond}. Moreover,
by conditional independence across \(i\), for \(i\neq j\),
\[
\E\!\left[
\big(\sigma(u_{\theta^\star}(e_i))-a_i\big)\big(\sigma(u_{\theta^\star}(e_j))-a_j\big)
\,\middle|\,\{e_k\}_{k=1}^n
\right]=0.
\]
Therefore,
\[
\E\!\left[\nabla L_n(\theta^\star)\nabla L_n(\theta^\star)^\top \,\middle|\, \{e_i\}_{i=1}^n\right]
=
\frac{1}{n^2}\sum_{i=1}^n
\Var(a_i\mid e_i)\,g_{\theta^\star}(e_i)\,g_{\theta^\star}(e_i)^\top.
\]
For a Bernoulli random variable with success probability \(\sigma(u)\),
\[
\Var(a\mid e)=\sigma(u)(1-\sigma(u))=\sigma'(u)\le \frac14.
\]
Hence,
\[
\E\!\left[\nabla L_n(\theta^\star)\nabla L_n(\theta^\star)^\top \,\middle|\, \{e_i\}_{i=1}^n\right]
\preceq
\frac{1}{4n^2}\sum_{i=1}^n g_{\theta^\star}(e_i)\,g_{\theta^\star}(e_i)^\top
=
\frac{1}{4n}\,\widehat\Sigma_n(\theta^\star),
\]
which proves \eqref{eq:grad-cov-cond}.

Let \(\Sigma_\star\doteq \Sigma_D(\theta^\star)=\E_{e\sim D}[g_{\theta^\star}(e)\,g_{\theta^\star}(e)^\top]\).
Note that
\[
\big\|\Sigma_\star^{\dagger/2}\nabla L_n(\theta^\star)\big\|_2^2
=
\tr\!\Big(
\Sigma_\star^{\dagger/2}\,
\nabla L_n(\theta^\star)\nabla L_n(\theta^\star)^\top\,
\Sigma_\star^{\dagger/2}
\Big).
\]
Taking expectation and using \eqref{eq:grad-cov-cond},
\[
\E\Big[\big\|\Sigma_\star^{\dagger/2}\nabla L_n(\theta^\star)\big\|_2^2\Big]
\le
\frac{1}{4n}\,
\E\!\left[\tr\!\big(\Sigma_\star^{\dagger/2}\widehat\Sigma_n(\theta^\star)\Sigma_\star^{\dagger/2}\big)\right].
\]
Since \(\E[\widehat\Sigma_n(\theta^\star)]=\Sigma_\star\), we obtain
\[
\E\Big[\big\|\Sigma_\star^{\dagger/2}\nabla L_n(\theta^\star)\big\|_2^2\Big]
\le
\frac{1}{4n}\,
\tr\!\big(\Sigma_\star^{\dagger/2}\Sigma_\star\Sigma_\star^{\dagger/2}\big).
\]
The matrix \(\Sigma_\star^{\dagger/2}\Sigma_\star\Sigma_\star^{\dagger/2}\) is the orthogonal projector onto
\(\mathrm{Range}(\Sigma_\star)\), hence its trace equals \(\mathrm{rank}(\Sigma_\star)\).
Under \cref{ass:identifiable-space}, \(\mathrm{Range}(\Sigma_\star)=H\), so
\(\mathrm{rank}(\Sigma_\star)=\dim(H)\). Therefore,
\[
\E\Big[\big\|\Sigma_\star^{\dagger/2}\nabla L_n(\theta^\star)\big\|_2^2\Big]
\le \frac{\dim(H)}{4n}.
\]
\end{proof}

To prove \cref{lem:vt-lb-trace}, we first show the following two useful technical lemmas.

\begin{lemma}[Uniform bounds for the Fisher matrix on $\Theta$]\label{lem:fisher-uniform-bounds}
Suppose \cref{ass:identifiable-space,ass:model-bounded,ass:feature-sep} hold.
Then there exist constants
\[
0<m_I\le M_I<\infty
\]
depending only on primitive constants such that for all $\theta\in\Theta$,
\begin{equation}\label{eq:fisher-uniform-two-sided}
m_I I_H \;\preceq\; I(\theta) \;\preceq\; M_I I_H
\qquad\text{on }H,
\end{equation}
where $I_H$ denotes the identity operator on the identifiable space $H$.
Equivalently, for every $v\in H$,
\begin{equation}\label{eq:fisher-uniform-two-sided-qf}
m_I\|v\|_2^2
\;\le\;
v^\top I(\theta)v
\;\le\;
M_I\|v\|_2^2.
\end{equation}

In particular, for the prior-averaged Fisher matrix
\[
\bar I_\rho \doteq \E_{\theta\sim\rho}[I(\theta)],
\]
we have
\begin{equation}\label{eq:fisher-bar-uniform-two-sided}
m_I I_H \;\preceq\; \bar I_\rho \;\preceq\; M_I I_H.
\end{equation}
Consequently, for all $\theta\in\Theta$ and all $v\in H$,
\begin{equation}\label{eq:fisher-bar-pointwise-comparable}
\frac{m_I}{M_I}\,v^\top I(\theta)v
\;\le\;
v^\top \bar I_\rho v
\;\le\;
\frac{M_I}{m_I}\,v^\top I(\theta)v .
\end{equation}
\end{lemma}

\begin{proof}[Proof of \cref{lem:fisher-uniform-bounds}]
The lower bound in \eqref{eq:fisher-uniform-two-sided} is exactly the uniform nondegeneracy on $H$ proved in \cref{lem:sep-implies-fisher}, so we may take $m_I=\underline\mu$.

For the upper bound, recall
\[
I(\theta)
=
\E_x\E_{y\sim\pi_\theta(\cdot\mid x)}
\big[s_\theta(x,y)s_\theta(x,y)^\top\big],
\qquad
s_\theta(x,y)=\nabla_\theta\log\pi_\theta(y\mid x).
\]
By the softmax form and \cref{ass:model-bounded}, the score is uniformly bounded on $H$:
for all $(x,y)$ and all $\theta\in\Theta$, $
\|s_\theta(x,y)\|_2 \le 2\alpha_1$, (here $\alpha_1$ is the uniform bound on $\|\nabla_\theta f_\theta(\phi(x,y))\|_2$).
Therefore, for any $v\in H$,
\[
v^\top I(\theta)v
=
\E_x\E_{y\sim\pi_\theta(\cdot\mid x)}
\big[(v^\top s_\theta(x,y))^2\big]
\le
\E_x\E_{y\sim\pi_\theta(\cdot\mid x)}
\big[\|v\|_2^2\|s_\theta(x,y)\|_2^2\big]
\le
4\alpha_1^2\|v\|_2^2.
\]
Hence \eqref{eq:fisher-uniform-two-sided} holds with $ M_I=4\alpha_1^2$.

Now average \eqref{eq:fisher-uniform-two-sided} with respect to $\rho$.
Since Loewner order is preserved under expectation,
\[
m_I I_H
\preceq
\E_\rho[I(\theta)]
=
\bar I_\rho
\preceq
M_I I_H,
\]
which proves \eqref{eq:fisher-bar-uniform-two-sided}.

Finally, for any $v\in H$ and any $\theta\in\Theta$,
\[
v^\top \bar I_\rho v \le M_I\|v\|_2^2
\le \frac{M_I}{m_I} v^\top I(\theta)v,
\]
and similarly
\[
v^\top \bar I_\rho v \ge m_I\|v\|_2^2
\ge \frac{m_I}{M_I} v^\top I(\theta)v.
\]
This proves \eqref{eq:fisher-bar-pointwise-comparable}.
\end{proof}

\begin{lemma}\label{lem:prior-absorb}
Suppose \cref{ass:design-coverage-b,ass:prior-vantrees} holds with $\mathcal{R}=\supp{(\rho)}$, and define
\[
J(\rho)
\;\doteq\;
\int_\Theta
\nabla \log \rho(\theta)\,\nabla \log \rho(\theta)^\top\,
\rho(\theta)\,d\theta,
\qquad
\bar\Sigma_\rho \;\doteq\; \E_{\theta\sim\rho}\!\big[\Sigma_D(\theta)\big].
\]

Let
\[
n_{\rm prior}
\doteq
\left\lceil
\frac{4\lambda_{\max}(J(\rho))}{\mu_{\mathcal R}}
\right\rceil.
\]
Then, for every $n\ge n_{\rm prior}$,
\[
J(\rho)\preceq \frac n4\bar\Sigma_\rho.
\]

\end{lemma}

\begin{proof}[Proof of \cref{lem:prior-absorb}]
By \cref{ass:design-coverage-b},
\[
\bar\Sigma_\rho
=\E_{\theta\sim\rho}[\Sigma_D(\theta)]
\succeq \mu_{\mathcal R}I_H.
\]
For every $n\ge n_{\rm prior}$, the definition of
$n_{\rm prior}$ gives
\[
\lambda_{\max}(J(\rho))
\le \frac{n\mu_{\mathcal R}}4.
\]
Consequently,
\[
J(\rho)
\preceq \lambda_{\max}(J(\rho))I_H
\preceq \frac{n\mu_{\mathcal R}}4 I_H
\preceq \frac n4\bar\Sigma_\rho.
\]
\end{proof}

\begin{lemma}\label{lem:functional-van-tree-control}
Suppose
\cref{ass:identifiable-space,ass:model-bounded,ass:feature-sep,ass:design-coverage-b,ass:prior-vantrees}
hold with $\mathcal R=\supp(\rho)$.
Let $H$ be the identifiable space with $d_H=\dim(H)$. Define
\[
T_D
=
\E_{\theta\sim\rho}
\tr\!\Big(I(\theta)\Sigma_D^\dagger(\theta)\Big),
\qquad
V_\theta=m^{-1/2}\nabla_\theta\sqrt{\pi_\theta},
\qquad
C(\theta)=V_\theta\Sigma_D^\dagger(\theta),
\]
where $V_\theta$ is the Jacobian of the normalized
square-root policy vector, and define
\[
K_D
=
\E_{\theta\sim\rho}
\left\|
\sum_{j=1}^{d_H}
\partial_{\theta_j}\!\Big(C(\theta)e_j\Big)
+
C(\theta)\nabla\log\rho(\theta)
\right\|_2^2.
\]
Then $
16\frac{K_D}{T_D}
\le
\frac{c_{\rm reg}}{\mu_{\mathcal R}}
\left(1+\tr\big(J(\rho)\big)\right)
,$ where
\[
c_{\rm reg}
\doteq
\frac{16G_u^2}{\underline\mu\mu_\mathcal{R}}
\left[
\Big(\alpha_2+3\alpha_1^2
+
\frac{2\alpha_1G_uH_u}{\mu_{\mathcal R}}\Big)^2
+
\frac{\alpha_1^2}
     {d_H^2}
\right],
\]
and $\underline\mu>0$ defined in \cref{lem:sep-implies-fisher} satisfies
$I(\theta)\succeq\underline\mu I_H$ on $\mathcal R$.
\end{lemma}

\begin{proof}[Proof of \cref{lem:functional-van-tree-control}]
By differentiation and \cref{ass:model-bounded}, we have 
\[
\|V_\theta\|_{\rm op}\le\alpha_1,
\qquad
\|\partial_{\theta_j}V_\theta\|_{\rm op}
\le\alpha_2+3\alpha_1^2.
\]
Differentiating
$\Sigma_D(\theta)=\E_{e\sim D}
[g(e;\theta)g(e;\theta)^\top]$
and using \cref{lem:u-properties}, we obtain
\[
\partial_{\theta_j}\Sigma_D(\theta)
=
\E_{e\sim D}\!\left[
(\partial_{\theta_j}g(e;\theta))g(e;\theta)^\top
+
g(e;\theta)(\partial_{\theta_j}g(e;\theta))^\top
\right],
\]
and hence
\[
\|\partial_{\theta_j}\Sigma_D(\theta)\|_{\rm op}
\le 2G_uH_u.
\]

By \cref{ass:design-coverage-b},
$\|\Sigma_D^\dagger(\theta)\|_{\rm op}
\le\mu_{\mathcal R}^{-1}$, so
\[
\|C(\theta)\|_{\rm op}
\le\frac{\alpha_1}{\mu_{\mathcal R}}.
\]
The inverse differentiation rule on $H$ gives
\[
\partial_{\theta_j}C
=
(\partial_{\theta_j}V_\theta)\Sigma_D^\dagger
-
V_\theta\Sigma_D^\dagger
(\partial_{\theta_j}\Sigma_D)
\Sigma_D^\dagger.
\]
Consequently,
\[
\|\partial_{\theta_j}C(\theta)\|_{\rm op}
\le
\frac{\alpha_2+3\alpha_1^2}{\mu_{\mathcal R}}
+
\frac{2\alpha_1G_uH_u}{\mu_{\mathcal R}^2}.
\]

Since
$\partial_{\theta_j}(Ce_j)
=(\partial_{\theta_j}C)e_j$,
the triangle inequality in $L^2(\rho)$ yields
\begin{equation*}
        \sqrt{K_D}
\le
\left(
\E_\rho\left\|
\sum_{j=1}^{d_H}\partial_{\theta_j}(Ce_j)
\right\|_2^2
\right)^{1/2}
+
\left(
\E_\rho\|C\nabla\log\rho\|_2^2
\right)^{1/2}
\le
d_H\left[
\frac{\alpha_2+3\alpha_1^2}{\mu_{\mathcal R}}
+
\frac{2\alpha_1G_uH_u}{\mu_{\mathcal R}^2}
\right]
+
\frac{\alpha_1}{\mu_{\mathcal R}}
\sqrt{\tr \big(J(\rho)\big)}.
\end{equation*}

Finally, \cref{lem:u-properties} implies
\[
\tr\big(\Sigma_D(\theta)\big)
=\E_{e\sim D}\|g(e;\theta)\|_2^2
\le G_u^2.
\]
By the Cauchy-Schwarz inequality for the positive
eigenvalues of $\Sigma_D(\theta)$ on $H$,
\[
\tr\big(\Sigma_D^\dagger(\theta)\big)
\ge
\frac{d_H^2}{\tr\big(\Sigma_D(\theta)\big)}
\ge\frac{d_H^2}{G_u^2}.
\]
Therefore,
\[
T_D \ge \underline\mu\,
\E_\rho\tr\big(\Sigma_D^\dagger(\theta)\big)
\ge
\frac{d_H^2\underline\mu}{G_u^2}.
\]
Combining this with the bound on $\sqrt{K_D}$
proves the claim.
\end{proof}

\section{Experiment implementation details and hyperparameters}

\subsection{Implementation Details for the IMDb Experiment}
\label{app:imdb-implementation}

The IMDb experiments were run on GPU workers equipped with H100-class GPUs. Depending on the curation setting and dataset size, each DPO fine-tuning run took from several minutes to roughly one hour, while the \(D^\ast\) construction and evaluation steps were substantially lighter.

\paragraph{SFT model and reference policy.}
We follow the IMDb setup in \citet{rafailov2023direct}. We first fine-tune GPT-2-large on reviews from the training split of the IMDb dataset using supervised fine-tuning (SFT). The resulting SFT model is used as the reference policy \(\pi_0\) for all subsequent DPO experiments. In all DPO runs, the policy is initialized from this SFT model, and the reference policy is fixed to the same SFT model.

\paragraph{DPO training configuration.}
For each curated preference dataset, we run full-parameter DPO training. We use RMSprop as the optimizer, with learning rate \(10^{-5}\). The per-device training batch size is \(16\), and we use gradient accumulation of \(2\), giving an effective batch size of \(32\). We use a cosine learning-rate scheduler with warmup ratio \(0.1\). The maximum sequence length is set to \(256\). For the reward--KL frontier experiments, each point in the frontier corresponds to the final checkpoint of one DPO training run under a specific DPO regularization parameter \(\beta\). We sweep
\[
\beta \in \{0.05, 0.1, 0.2, 0.5, 1, 2, 5\}.
\]
For evaluation, we generate responses from the trained policy and compute both the reward and the KL divergence relative to the SFT reference policy. Following \citet{rafailov2023direct}, we use the sentiment classifier \texttt{siebert/sentiment-roberta-large-english} as the reward model. The KL is estimated on generated responses as the sequence-level log-probability difference between the trained policy and the reference policy,
\[
\log \hat\pi(y\mid x)-\log \pi_0(y\mid x)
=
\sum_{t=1}^{|y|}
\left[
\log \hat\pi(y_t\mid x,y_{<t})
-
\log \pi_0(y_t\mid x,y_{<t})
\right].
\]
This sequence-level empirical KL estimate follows the evaluation protocol used in \citet{rafailov2023direct}.

\paragraph{Candidate generation.}
For prompt selection, we sample a candidate pool of \(1{,}000\) prompts and generate two responses for each prompt from the SFT reference model. This gives one candidate comparison per prompt. We then select \(N=175\) comparisons for DPO training. For response selection, we generate \(D=8\) candidate responses for each prompt. This gives \(\binom{8}{2}\) candidate response pairs within each prompt. We select one response pair per prompt for DPO training.

\paragraph{Feature representation.}
To compute the \(D^\ast\)-based design, we represent each prompt--response pair \((x,y)\) using the last-token hidden representation from the SFT model. Let
\[
\phi(x,y) \in \mathbb R^d
\]
denote this last-token representation. For a candidate comparison
\[
e=(x,y_a,y_b),
\]
we define the feature difference
\[
g_e
=
\beta_D\left(\phi(x,y_a)-\phi(x,y_b)\right),
\]
where \(\beta_D\) is the DPO scaling parameter used in the design computation. This feature difference is used as a low-dimensional proxy for the pairwise sensitivity vector of the DPO logit.

\paragraph{Estimating the design weight matrix \(I(\theta_0)\).}
The design objective uses a feature-space approximation to \(I(\theta_0)\). We first fit a linear approximation to the SFT log-probabilities. Specifically, for each prompt \(x_i\) and candidate response \(y_{ik}\), we compute the SFT log-probability
\[
\ell_{ik} = \log \pi_0(y_{ik}\mid x_i).
\]
Since only within-prompt comparisons are relevant, we remove a prompt-specific mean and fit a ridge regression:
\[
\widehat \theta_0
\in
\arg\min_{\theta}
\sum_{i,k}
\left(
\phi(x_i,y_{ik})^\top \theta
-
\left(\ell_{ik}
-
\frac{1}{K_i}\sum_{\ell=1}^{K_i}\ell_{i\ell}
\right)
\right)^2
+
\alpha \|\theta\|_2^2 .
\]
This gives a feature-space approximation to the SFT policy. We then define a softmax distribution over the candidate responses for each prompt:
\[
\widehat \pi_{\theta_0}(y_{ik}\mid x_i)
=
\frac{
\exp\left(\phi(x_i,y_{ik})^\top \widehat \theta_0\right)
}{
\sum_{\ell=1}^{K_i}
\exp\left(\phi(x_i,y_{i\ell})^\top \widehat \theta_0\right)
}.
\]
The matrix \(I(\theta_0)\) is estimated by the average within-prompt feature covariance:
\[
\widehat I(\theta_0)
=
\frac{1}{M}
\sum_{i=1}^M
\sum_{k=1}^{K_i}
\widehat \pi_{\theta_0}(y_{ik}\mid x_i)
\left(
\phi(x_i,y_{ik})-\widehat \mu_i
\right)
\left(
\phi(x_i,y_{ik})-\widehat \mu_i
\right)^\top,
\]
where
\[
\widehat \mu_i
=
\sum_{k=1}^{K_i}
\widehat \pi_{\theta_0}(y_{ik}\mid x_i)
\phi(x_i,y_{ik}).
\]

\paragraph{\(D^\ast\) design computation.}
Given candidate comparisons \(e\in\mathcal E\), we compute design weights over the corresponding feature differences \(g_e\). For a nonnegative weight vector \(w\) over candidate comparisons, define the regularized information matrix
\[
A(w)
=
\lambda I
+
\sum_{e\in\mathcal E}
w_e g_e g_e^\top .
\]
The implemented \(D^\ast\) design minimizes the feature-space \(I\)-optimal objective
\[
\Phi(w)
=
\operatorname{tr}\left(
\widehat I(\theta_0) A(w)^{-1}
\right).
\]
We solve this continuous design problem approximately using a Frank-Wolfe procedure.

\paragraph{Sampling from the design.}
In the prompt selection experiment, we first compute \(D^\ast\) weights over all \(1{,}000\) candidate prompt-level comparisons. We then sample \(N=175\) comparisons without replacement according to the normalized \(D^\ast\) weights. The benchmark selects the first \(175\) prompt-level comparisons.

In the response selection experiment, we compute \(D^\ast\) weights over all within-prompt candidate response pairs. For each prompt, we normalize the weights over its \(\binom{8}{2}\) candidate response pairs and sample one pair without replacement from this within-prompt distribution. The benchmark always compares the first two generated responses for each prompt.

\paragraph{Monte Carlo replications and error bars.}
For each dataset-generation seed, we independently resample the candidate pool, regenerate candidate responses, recompute the \(D^\ast\) design weights, reconstruct the curated preference training set, and retrain the DPO policy from the same SFT reference model. Error bars in the reward--KL frontier plots report Monte Carlo standard errors across these independent runs. In the prompt selection experiment, error bars are computed over \(30\) independent runs; in the response selection experiment, error bars are computed over \(80\) independent runs.

\subsection{Implementation Details for the Anthropic-HH Experiment}
\label{app:hh-implementation}
The Anthropic-HH experiments were run on H100-class GPU workers. The trace-design construction and response generation used one GPU, while each Pythia-2.8B DPO fine-tuning run used four GPUs. Across the Anthropic-HH budgets considered, each DPO fine-tuning run took about one hour. GPT-based evaluation was performed through API calls and did not require GPU computation.

\paragraph{Dataset and reference policy.}
We further evaluate our dataset curation method on the Anthropic Helpful--Harmless (HH) preference dataset. We use the default train/test splits of the Anthropic HH-RLHF dataset, which contain preference comparisons from both helpfulness and harmlessness data. Each example consists of a prompt, a chosen response, and a rejected response. Following the DPO pipeline in \citet{rafailov2023direct}, we first perform supervised fine-tuning (SFT) on Pythia-2.8B using the chosen responses in the HH training split. The resulting SFT model is used as both the initialization of the DPO policy and the fixed reference policy \(\pi_0\). Each HH training example consists of a prompt \(x\), a chosen response \(y^+\), and a rejected response \(y^-\), which directly defines one candidate preference comparison.

\paragraph{Candidate pool and benchmark.}
We construct a candidate pool from the HH training split. In the full-pool experiment, the candidate pool contains \(M=160{,}800\) preference comparisons. Given a training budget \(N\), the benchmark dataset consists of the first \(N\) comparisons in this candidate pool. The \(D^\ast\)-curated dataset uses the same budget \(N\), but selects comparisons according to the optimized \(D^\ast\) design weights. Both methods therefore use the same number of preference comparisons for DPO training.

\paragraph{Feature construction.}
For each candidate comparison \(e_i=(x_i,y_i^+,y_i^-)\), we use the SFT model to extract hidden representations. Specifically, we feed the full sequences \(x_i \circ y_i^+\) and \(x_i \circ y_i^-\) into the SFT model and take the last-token hidden states. We denote the resulting representations by \(\phi_i^+\) and \(\phi_i^-\). We then use the difference
\[
    g_i = \beta(\phi_i^+ - \phi_i^-)
\]
as the design vector for this comparison, where \(\beta=0.1\) matches the DPO regularization parameter used in training. To make the design computation numerically tractable, we standardize the hidden representations and apply PCA to obtain a \(128\)-dimensional feature representation.

\paragraph{\(D^\ast\) design computation.}
The \(D^\ast\)-based design is computed using the same feature-space approximation procedure as in the IMDb experiment. We first fit a local parameter \(\theta_0\) by ridge regression from the feature representations to the centered SFT log-likelihoods of the chosen and rejected responses. Using this fitted \(\theta_0\), we compute a Fisher-type matrix \(I(\theta_0)\) over the candidate pool. We then solve the trace-design problem
\[
    \min_{q\in\Delta_M}
    \operatorname{tr}\left[
    \left(
        \lambda I + \sum_{i=1}^M q_i g_i g_i^\top
    \right)^{-1}
    I(\theta_0)
    \right],
\]
where \(q\) is a distribution over candidate comparisons and \(\lambda=10^{-3}\) is a ridge parameter. We solve this optimization problem using Frank--Wolfe. The resulting optimized design weights are used as sampling probabilities. For a fixed budget \(N\), we sample \(N\) comparisons \textbf{without replacement} according to these weights.

\paragraph{DPO training configuration.}
For each curated HH preference dataset, we perform DPO starting from the SFT Pythia-2.8B model, with the same SFT model fixed as the reference policy. We use the same DPO training configuration for the \(D^\ast\)-curated dataset and the benchmark dataset. In the reported runs, we use \(\beta=0.1\), batch size \(64\), RMSprop optimizer with learning rate \(10^{-6}\), and a linear learning-rate warmup over the first \(150\) optimization steps. Each model is trained for one epoch. This ensures that differences in downstream performance are attributable to the curated preference data rather than to training hyperparameters.

\paragraph{Evaluation protocol.}
To evaluate the trained policies, we generate responses on prompts from the HH test split. For each test prompt, we compare the generated response with the corresponding HH chosen response. We use sampling temperatures \(0.25\), \(0.7\), and \(1.0\), and evaluate \(500\) test prompts for each method-temperature pair. GPT-4.1 is used as an automatic judge. The judge compares the model-generated response against the HH chosen response and returns whether the model response is better, worse, or tied. We report the win rate of the trained policy's generated response against the HH chosen response, counting a tie as \(0.5\). For all methods, we use the same test prompts and the same judging protocol.

\end{document}